\documentclass{article}

\PassOptionsToPackage{numbers, compress}{natbib}

\usepackage[preprint]{neurips_2019}

\usepackage[utf8]{inputenc} 
\usepackage[T1]{fontenc}    
\usepackage{hyperref}       
\usepackage{url}            
\usepackage{booktabs}       
\usepackage{amsfonts}       
\usepackage{nicefrac}       
\usepackage{microtype}      

\usepackage{amsmath}
\usepackage{amsfonts}
\usepackage{amsthm}
\usepackage{bm}
\usepackage{mathtools}

\newtheorem{definition}{Definition}
\newtheorem{lemma}{Lemma}

\theoremstyle{definition}
\usepackage{algorithm}
\usepackage{algorithmic}
\usepackage{todonotes}
\usepackage{xhfill}

\newcommand{\rewlat}{{g}}
\newcommand{\dynlat}{{z}}
\newcommand{\rewqvariat}{{q_{\omega}}}
\newcommand{\dynqvariat}{{q_{\delta}}}

\usepackage{subcaption}
\setcitestyle{square}
\usepackage{wrapfig}

\title{Disentangled Skill Embeddings \\ for Reinforcement Learning}

\author{
  Janith C. Petangoda{$\ ^\dagger$}\thanks{Research conducted while doing a research placement at PROWLER.io. } \\
    Imperial College London\\
  \texttt{jcp17@ic.ac.uk} 
   \And
   Sergio Pascual-D\'iaz\thanks{Equal contribution.}\\
   PROWLER.io\\
   \texttt{sergio.diaz@prowler.io} 
   \AND
   Vincent Adam\\
   PROWLER.io\\
   \texttt{vincent.adam@prowler.io} 
  \And
   Peter Vrancx\\
   PROWLER.io\\
   \texttt{peter@prowler.io} 
   \And
   Jordi Grau-Moya\\
   PROWLER.io\\
   \texttt{jordi@prowler.io} 
}

\begin{document}

\maketitle
\begin{abstract}
We propose a novel framework for multi-task reinforcement learning (MTRL). Using a variational inference formulation, we learn policies that generalize across both changing dynamics and goals.  The resulting policies are parametrized by shared parameters that allow for transfer between different dynamics and goal conditions, and by task-specific latent-space embeddings that allow for specialization to particular tasks. We show how the latent-spaces enable generalization to unseen dynamics and goals conditions. Additionally, policies equipped with such embeddings serve as a space of skills (or options)  for hierarchical reinforcement learning.  Since we can change task dynamics and goals \emph{independently}, we name our framework Disentangled Skill Embeddings (DSE). 
\end{abstract}

\section{Introduction}
In recent years, Reinforcement Learning (RL)~\citep{sutton1998reinforcement}  techniques have been successfully applied  to solve complex sequential decision-making problems under uncertainty~\citep{mnih2015human,hessel2017rainbow, haarnoja2017reinforcement,haarnoja2018soft}. However, agents trained on a single-task typically exhibit poor performance when faced with a modified task with different dynamics or reward functions~\citep{taylor2009transfer}. A key remaining challenge for advancing the field towards general purpose applications is to train agents that can generalize over tasks~\citep{plappert2018multi}.  
Generalization over tasks is also important in hierarchical reinforcement learning (HRL) settings~\citep{dayan1993feudal,sutton1999between,dietterich1998maxq} where there is a need to be data-efficient by using reusable skills across different unseen situations. However, a long-standing problem in HRL is of obtaining a general skill-set and how to properly reuse that set in different situations. 

In this paper, we focus on the problem of learning policies that generalize well under changes in both the dynamics and reward functions. We do so by formulating a novel Multi-Task RL (MTRL) problem from a variational inference (VI) perspective. Our formulation relies on two latent skill embeddings which hold information about the dynamics of the system and the goal. The skill embeddings are disentangled in that one can \emph{independently} specify a change in the dynamics or the goal of the system and still obtain a well-performing policy. We call this method Disentangled Skill Embeddings (DSE). Having trained such policies using DSE, we can tackle HRL problems by allowing the agent to move in the previously learned space of skills.

We contribute: 1) a novel MTRL formulation for learning disentangled skill embeddings using a VI objective; 2) two MTRL algorithms, DSE-REINFORCE and DSE-SAC (Soft Actor Critic), that can learn multi-dynamics and multi-goal policies and generalize to new tasks after fast retraining; and 3) demonstrate that one can learn higher level policies over latent skills in a HRL scenario.

\section{Related work}
Most approaches in MTRL ~\citep{taylor2009transfer,oh2017zero, henderson2017benchmark,nachum2018data,steindor, deisenroth2014multi} consider changes in dynamics or reward functions to result in different tasks. That is, two tasks with the same dynamics but different reward functions are considered to be different tasks. Here, we focus on exploiting  known changes either in dynamics or goals, or both, for better generalization.

An approach that decouples dynamics and rewards is \citep{devin2017learning}. In this work modular neural networks that capture robot dynamics are combined with modules that capture the task goals. This allows robots to  solve novel tasks by recombining task and robot modules. In \citep{zhang2018decoupling}  decoupling reward and dynamics is done in a model-based framework based on successor features. 
 
Closest to our work is~\citep{hausman2018learning} where  policies trained in a MTRL scenario are equipped with a latent variable embedding that describes a particular task. However, the latent variable embedding contains entangled information about both the transition and the reward function. This can impede generalization as the latent space might not be able to represent a policy for a task for unseen reward and dynamics combination. In contrast, our approach disentangles the latent spaces to overcome this issue.
Another similar approach is \citep{gupta2018meta}, which is a meta-learning algorithm that optimizes the latent space for the policies to enable structured exploration across multiple time steps.

We derive our algorithm using a variational infernence formulation for RL. Several previous works have described RL as an inference problem~\citep{kappen2005path,todorov2008general,levine2013variational} including its relation to entropy regularization~\cite{haarnoja2017reinforcement,haarnoja2018soft,grau2016planning,peters2010relative,ziebart2008maximum}. This formalism has recently attracted attention~\citep{haarnoja2017reinforcement,levine2018reinforcement,abdolmaleki2018maximum}, because it provides a powerful and intuitive way to describe more complex agent architectures using the tools from graphical models.

The multi-task algorithm, Distral \citep{teh2017distral}, utilizes, not only entropy regularization, but also adds a relative-entropy penalty that encourages the policies to be close to a shared compressed policy for transfer between tasks. The trained policies of this approach are not parameterized by any variables that identify the task at hand. This is important since policies trained with Distral cannot generalize beyond their training tasks.

\section{Background and Notation}
We consider a set $\mathcal M = \{m_{i,j} | \forall i \in I, j \in J \}$ of Markov decision processes (MDPs) $m_{i,j}$  defined as the tuple  $\left\langle \mathcal{S}, \mathcal{A}, \gamma, P_i, r_j \right\rangle$ where $\mathcal{S}$ is the state space, $\mathcal{A}$ the action space, $\gamma$  is the discount factor,  $P_i: \mathcal{S} \times \mathcal{A} \times \mathcal{S} \rightarrow [0,1]$ denotes the $i$'th state transition function that fully specifies the dynamics of the system, and   $r_j:\mathcal{S} \times \mathcal{A} \rightarrow \mathbb{R}$ denotes the reward function that quantifies the agent's performance and fully specifies the goal of the system.
 
In the multi-task problem, the agent must provide an optimal policy $\pi^\star_{i,j}$ for each of the $I\times J$ possible tasks. Obtaining all solutions independently maximizes performance on each individual task, but does not transfer information between the tasks. Solving all tasks with a single policy $\pi$ maximizes transfer since all  parameters of the policy are shared but, importantly,  the final solution only maximizes the average reward across tasks. Thus, a mixture of shared and task specific parameters is ideal. 

In the following we derive a variational multi-task reinforcement learning formulation where the policy has some parameters that are shared across all tasks, and two latent embeddings---dynamics-specific and reward/goal-specific---that serve as task-specific parameters.

\section{Disentangled Skill Embeddings}
We learn flexible skills that are reusable across different dynamics and goals by learning two latent spaces, $\dynlat$ and $\rewlat$. We achieve this by gathering data from the set of MDPs $\{m_{i,j} \}$, indexed by the dynamics-condition $i$ and goal-condition $j$ for all $i \in \{1,\dots I\}$ and $j \in \{1,\dots J\}$ and then learn the conditional distributions $\dynqvariat(\dynlat | i)$ and $\rewqvariat(\rewlat |j)$ for each $i$ and $j$. The latent variables are inputs to the policy $\pi(a|s,\dynlat,\rewlat)$ serving as behaviour modulators. Importantly, once the latent spaces are fully learnt, one can directly use the policy equipped with the skill embeddings without knowledge of the task indices. 

We now  derive a variational inference (VI) formulation for multi-task RL that allows learning both latent spaces and the policy. As in ~\citep{haarnoja2017reinforcement,kappen2005path,todorov2008general,levine2013variational,levine2018reinforcement,abdolmaleki2018maximum}, we start by introducing a random variable $R$ that denotes whether the trajectory  $\tau:= (i,j, s_0, a_0, \dynlat_0,\rewlat_0,  \dots s_T, a_T, \dynlat_T, \rewlat_T)$ is optimal ($R = 1$) or not ($R = 0$). 
Note that this includes the  dynamics index $i$ and the goal index $j$ as well as $\dynlat_t$ and $\rewlat_t$ for all time-steps $t$. The likelihood of an optimal trajectory is defined as $p(R=1 | \tau) \propto \exp \left(\sum_t^T r_j(a_t, s_t) \right)$.  We denote the posterior trajectory probability assuming optimality as $p(\tau | R=1)$. Treating $\tau$ as a latent variable with prior probability $p(\tau)$, we specify the log-evidence as $\log p(R=1) = \log \int p(R=1| \tau ) p(\tau) d\tau$. 

We  now introduce a  variational distribution on trajectories $q(\tau)$ which combined with Jensen's inequality provides the Evidence Lower Bound (ELBO) $\mathbb E_{\tau \sim q(\tau)} \left[ \log{p(R=1| \tau ) p(\tau)} - \log{q(\tau)}   \right]$.
In practice, we maximize the ELBO and use $q$ as an approximate posterior. 
The generative model is $p(\tau) = p(i)p(j)p(s_0) \prod^{T-1}_{t=0} p(\dynlat_t|i) p(\rewlat_t|j) p(a_t| s_t,\dynlat_t,\rewlat_t)$ $P_i(s_{t+1}  | s_t, a_t)$  and the variational distribution is $q(\tau) = p(i)p(j)p(s_0) \prod^{T-1}_{t=0} \dynqvariat(\dynlat_t|i)\rewqvariat(\rewlat_t|j) \pi(a_t| s_t,\dynlat_t,\rewlat_t)$ $P_i(s_{t+1}  | s_t, a_t)$. We stress that the only difference between these are the conditional factors involving the latent variables $z$ and $g$. The MTRL problem can now be stated as a maximization of the ELBO w.r.t. $\pi, \dynqvariat$ and $\rewqvariat$:
\begin{align} \label{eq:objectiveDSE}
    \max_{\pi, \rewqvariat, \dynqvariat} \mathop{\mathbb{E}}_{q(\tau)}  \bigg[  \sum_{t=0}^{\infty} \gamma^t \bigg( r_j(s_t, a_t) - \frac{1}{\alpha_d} \log \frac{\dynqvariat(\dynlat_t |i)}{p(\dynlat_t | i)} 
     - \frac{1}{\alpha_r} \log  \frac{\rewqvariat(\rewlat_t|j)}{p(\rewlat_t | j)} -  \frac{1}{\alpha_\pi} \log \frac{\pi(a_t| s_t,\dynlat_t,\rewlat_t)}{p(a_t| s_t,\dynlat_t,\rewlat_t)} \bigg) \bigg] 
\end{align}
where we added the scalars $\frac{1}{\alpha_d}, \frac{1}{\alpha_r}$ and $ \frac{1}{\alpha_\pi}$ to weight each information term  (see Appendix A.1 for a mathematical justification) and we set the problem to have infinite horizon ($T \rightarrow \infty$) with discount factor $\gamma$. 
The first two information terms measure how far the  variational distributions $\dynqvariat$ and $\rewqvariat$ are from the specified priors that we assume fixed and equal for all conditions.  Similarly, the last information term measures how far the variational policy is from the prior policy. Note that by setting $\alpha_\pi \rightarrow \infty$ we can eliminate this restriction. Furthermore, by setting $p(a_t| s_t,\dynlat_t,\rewlat_t)$ to be an improper uniform prior  we recover a formulation with entropy regularization in the policy.

This trajectory-based formulation of the problem is sufficient to derive a novel REINFORCE-type algorithm  equipped with DSE.  However, our work also provides a derivation of a multi-task SAC algorithm with DSE (DSE-SAC) that requires a full specification of the recursive properties of value functions and the optimal solutions for the policy and embeddings. We describe those properties in the next section. 

\subsection{Recursions, Optimal Policies and Optimal Embeddings}
Crucial to the construction of DSE-SAC is a recursive property that we can exploit for value bootstraping. A task-indexed value function $V^\pi_{i,j}(s)$ can be defined by taking the expectation in Equation~\eqref{eq:objectiveDSE} over all random variables except $s, i$ and $j$. The Q-function is then defined as $Q^\pi_{i,j}(s,a) := r_j(s,a) + \gamma \sum_{s'} P_i(s' |s, a) V^\pi_{i,j}(s')$. We provide a lemma for the value recursion.
\begin{lemma}[Index- and state-dependent Value Function Recursion] 
The index-dependent Value function satisfies the following recursive property.  
\begin{align}\label{eq:value_recursion_ij}
 &V^\pi_{i,j}(s) = \sum_{ \dynlat,\rewlat, a} \dynqvariat(\dynlat|i)  \rewqvariat(\rewlat|j) \pi(a|s,\dynlat,\rewlat) \bigg[ r_j(s,a) - \frac{1}{\alpha_d} \log\frac{\dynqvariat(\dynlat|i)}{p(\dynlat|i)} \nonumber \\
 &- \frac{1}{\alpha_r}\log\frac{\rewqvariat(\rewlat|j)}{p(\rewlat|j)}
-\frac{1}{\alpha_\pi}\log\frac{\pi(a|s, \dynlat,\rewlat) }{p(a|s,\dynlat,\rewlat)} + \gamma \sum_{s'} P_i(s'|s,a) V^\pi_{i,j}(s')
 \bigg] 
\end{align}
\end{lemma}
The proof of the previous and subsequent lemmas can be found in Appendix A.2 and A.4. 

DSE-SAC also requires analytic solutions for the policy and embeddings. The optimal policy can be obtained by computing the functional derivative of a Lagrangian (see Appendix) of the variational problem w.r.t. the policy and equating the result to zero. 
\begin{lemma}[Optimal policy with DSE]\label{eq:optimal_policy} Let the variational distributions $\dynqvariat$ and $\rewqvariat$ be fixed. Then, the optimal policy is 
\begin{equation*}
    \pi^\star(a|s,\dynlat,\rewlat) = \frac{p(a|s, \dynlat, \rewlat)\exp\left( \alpha_\pi \bar Q^{\pi^\star}(s,a, \dynlat, \rewlat) \right)}{Z(s,\dynlat,\rewlat)}
\end{equation*}
where $Z(s,\dynlat,\rewlat)$ is the normalizing function and $\bar Q^{\pi^\star}(s,a,\rewlat,\dynlat) := \sum_{i,j} q(i|\dynlat)$ $q(j|\rewlat) Q^{\pi^\star}_{i,j}(s,a)$ with $q(i|\dynlat)=\frac{p(i)\dynqvariat(\dynlat|i) }{\sum_i p(i)\dynqvariat(\dynlat|i)}$ and $q(j|\rewlat)=\frac{p(j)\rewqvariat(\rewlat|j)}{\sum_j p(j)\rewqvariat(\rewlat|j)}$ are the Bayesian posteriors over $i$ and $j$.
\end{lemma}
Note that the Q-values $\bar Q^{\pi^\star}$ are computed by using both Bayesian posterior distributions over the task indices, i.e., $q(i|z)$ and $q(j|g)$.  Intuitively, in the extreme case where a $z$ and $g$ can completely specify the task at hand with certainty (i.e., the Bayesian posterior is peaked), the optimal policy selects the correct Q-function for this task; whereas for non-extreme cases a mixture is computed. 

Employing the same procedure as before, we write the optimal variational distributions as follows: 
\begin{lemma}[Optimal Embeddings] \label{eq:optimal_variationals}
Assuming fixed $\pi$, the optimal variational distributions are
\begin{equation*}
   \dynqvariat^\star(\dynlat|i) = \frac{p(\dynlat|i)e^{ \alpha_d D^\star_i(\dynlat)  }}{Z(i)} \quad \text{where} \quad D^\star_i(z)  :=  \mathop{\mathbb E}_{\substack{j\\ g\sim q_{\omega_j}^\star}} \mathop{\mathbb{E}}_{\substack{a\sim \pi \\ s \sim p}} \bigg( Q^{\pi}_{i,j}(s,a) - \frac{1}{\alpha_\pi} \log \frac{\pi(a|s,\dynlat,\rewlat)}{p(a|s,\dynlat,\rewlat)}\bigg) 
\end{equation*}
\begin{equation}
\rewqvariat^\star(\rewlat|j) = \frac{p(\rewlat|i)e^{ \alpha_r G^\star_j(\rewlat)  }}{Z(j)} \quad \text{where} \quad G^\star_j(g)   :=  \mathop{\mathbb E}_{\substack{i \\ z\sim q_{\delta_i}^\star}} \mathop{\mathbb{E}}_{\substack{a\sim \pi \\ s \sim p}} \bigg( Q^{\pi}_{i,j}(s,a) - \frac{1}{\alpha_\pi} \log \frac{\pi(a|s,\dynlat,\rewlat)}{p(a|s,\dynlat,\rewlat)}\bigg). 
\end{equation}
where $D$ and $G$ are conceptually similar to Value functions but depend on $i$, $z$ and $g$, $j$, respectively.
\end{lemma}

\section{DSE Algorithms}
This section focuses on describing two practical algorithms using disentangled embeddings. DSE-REINFORCE is updated on-policy and requires full trajectories from the different tasks. Although it is easier to implement, REINFORCE-type algorithms  are known to  suffer from high variance in the gradient estimates which slows down training. The second algorithm, DSE-SAC, is inspired by the SAC algorithm~\citep{haarnoja2018soft}, and is an more data-efficient off-policy algorithm that directly uses the transitions of all the tasks sampled from a replay memory. It achieves this by estimating the Value functions and Q-functions. 

Common to both algorithms is the parametrization of the policy $\pi_\theta(a |s,\dynlat, \rewlat)$ with parameters $\theta$ representing those of a neural network. Additionally, we consider pure entropy regularization by setting  the prior policy $p(a|s,\dynlat, \rewlat)$ to an improper uniform distribution by which we can ignore. Similarly, we consider the embeddings $\dynqvariat( \dynlat| i)$ and  $\rewqvariat(\rewlat | j)$ to have parameters (abusing notation) $\delta$ and $\omega$, respectively. With this notation in place we proceed to describe DSE-REINFORCE.

\subsection{DSE-REINFORCE}
 DSE-REINFORCE first samples $M$ trajectories $\{\tau^{i,j}_m \}_{m=1}^{M}$ from each combination of dynamics and goal contexts $i,j$, where $\tau^{i,j}_m:=(s_{0,m}^{i,j}, a_{0,m}^{i,j}, z_{0,m}^{i,j}, g_{0,m}^{i,j}, \dots$ $s_{t,m}^{i,j}, a_{t,m}^{i,j}, z_{t,m}^{i,j}, g_{t,m}^{i,j},\dots ) $. Then, these are used to update the shared parameters $\theta$ and the parameters $\delta$ and $\omega$ of the variational distributions. For the updates of the variational distributions we use the reparametrization trick~\citep{kingma2013auto} and make the assumption that the variational parameters $\delta = (\delta_1 \dots \delta_i \dots \delta_I) $ and $\omega = (\omega_1 \dots \omega_j \dots \omega_J) $  contain a set of specific parameters $\delta_i$ and $\omega_j$ for each dynamics and goal context. We further assume that the latent variables are multivariate Gaussians with diagonal covariance matrix (this assumption can easily be relaxed). 
Therefore, the latent variables are expressed as $\dynlat_{\delta_i}(\epsilon):= \mu_{\delta_i} + \sigma_{\delta_i}\epsilon  $ and $\rewlat_{\omega_j}(\varepsilon):=  \mu_{\omega_j} + \sigma_{\omega_j} \varepsilon$ where $\epsilon \sim \mathcal N(\bm 0,\bm 1)$ and $\varepsilon \sim \mathcal N(\bm 0, \bm1)$ are the noise terms; $\mu_{\delta_i}$ and $\mu_{\omega_j}$ are the mean vectors and; $\sigma_{\delta_i}^2$ and $\sigma_{\omega_j}^2$  are the diagonal vectors of the covariance matrix. 

The maximization in~\eqref{eq:objectiveDSE} can be written  using Monte Carlo estimators as $\max_{\theta, \delta, \omega}  \mathcal{L}(\theta, \delta, \omega)$ with
\begin{align*}
\mathcal L(\theta, \delta, \omega) := 
  \frac{1}{MIJ}\sum_{m,i,j=1}^{M,I,J} R_0(\tau_{m}^{i,j}) - \frac{C}{\alpha_d} \mathbb{E}_{i} [\textrm{KL}(\dynqvariat(\cdot |i)||p(\cdot | i)) ] - \frac{C}{\alpha_r} \mathbb{E}_{j} \left[\textrm{KL} (\rewqvariat(\cdot |j) || p(\cdot | j) ) \right]  
\end{align*}
where we have used the following definition of the regularized discounted future returns 
\begin{align*}
    \tilde R_t(\tau_{m}^{i,j}) := \sum_{h=0}^{T-1} \gamma^{h-t} \bigg( r_j(s_{h,m}^{i,j}, a_{h,m}^{i,j}) -  \frac{1}{\alpha_\pi} \log \pi_\theta \left(a_{h,m}^{i,j} |s_{h,m}^{i,j},  \dynlat_{\delta_i}(\epsilon_{h,m}^{i,j}), \rewlat_{\omega_j} (\varepsilon_{h,m}^{i,j}) \right) \bigg) .
\end{align*}
Moreover, we have separated the KL terms of the variational distributions out of the summation over $t$, as they are independent of $t$ and can be computed in closed form due to the Gaussian assumption. Consequently, we added the corresponding sum of discounts by computing the geometric sum $C:=\frac{1-\gamma^{T-1}}{1-\gamma}$.  Note that we implicitly redefined the trajectories so that they contain the noise realizations instead of the latent variables. Algorithmic details can be found in the Appendix B.

\paragraph{Adaptive normalization using Pop-Art:}
In our preliminary experiments we observed that DSE-REINFORCE was selectively solving some tasks but not others. For this reason we use the adaptive rescaling method Pop-Art~\citep{van2016learning,hessel2018multi} to normalize the discounted rewards $\tilde R_t (\tau_m^{i,j})$ to have zero mean and unit variance before each training iteration. Thus all tasks affect the gradient equally.

\subsection{DSE-SAC}
DSE-SAC collects $M$ transitions $\{(s_{t,m}^{i,j}, a_{t,m}^{i,j}, r_j( s_{t,m}^{i,j}, a_{t,m}^{i,j}), s_{t+1,m}^{i,j})\}_{m=1}^M$ from each dynamics and reward context $i,j$ and stores them in separate replay memories $\mathcal D_{ij}$.  Then, samples from the replay memories are used to estimate the Q-functions  $Q_{\phi_{i,j}}$, the value functions $V_{\psi_{i,j}}$, the variational distributions $q_{\delta_i}$ and $q_{\omega_j}$ and the policy $\pi_\theta$.  Subscripts denote the symbols for the parameters of the neural networks used as function approximators. 

The Q-functions are learned by optimizing the loss
$  \mathcal L_{Q_{i,j}} := \mathbb E_{s,a \sim \mathcal D_{i,j}} \Big[\left( \hat Q (s,a) -  Q_{\phi_{i,j}} (s,a) \right)^2 \Big]$, 
where the target $\hat Q (s,a) := r_j(s,a) + \gamma V_{\bar \psi_{i,j}}(s') $ is a one-sample estimate obtained with real experience and $\bar \psi_{i,j}$ denotes the parameters of a target network which is  updated at every training iteration as $\bar \psi_{i,j} \leftarrow \tau \psi_{i,j} + (1-\tau) \bar\psi_{i,j}$ for some $\tau \in (0, 1]$ . 

The value functions are learned by minimizing
$ \mathcal L_{V_{i,j}} := \mathop{\mathbb E}_{s\sim \mathcal D_{i,j}  }  \Big[\Big(    \hat V_{i,j}(s)  - V_{\psi_{i,j}} (s)     \Big)^2 \Big]$,
where the target exploits the value recursion in Equation~\eqref{eq:value_recursion_ij}: 
\begin{align*}
    \hat V_{i,j}(s) :=
    \mathop{\mathbb E}_{z \sim q_{\delta_i},  g\sim q_{\omega_j}, a\sim \pi_\theta}
    \bigg[ Q_{\phi_{i,j}}^\textrm{min} (s,a)- \log \pi_\theta (a |s, \dynlat, \rewlat) - \frac{1}{\alpha_d} \log \frac{q_{\delta_i}(z)}{p(z)}  -  \frac{1}{\alpha_r}\log  \frac{q_{\omega_j}(g)}{p(g)} \bigg].
\end{align*}
In order to reduce overestimation of Q-values, we follow ~\citep{haarnoja2018soft} where the minimum of two Q-function approximators is used i.e., $Q_{\phi_{i,j}}^\textrm{min}(s,a) := \min_{x \in {1,2}} Q_{\phi_{i,j}}^{(x)}(s,a)$, which have different sets of parameters and initialization but are trained using the same loss $\mathcal L_{Q_{i,j}}$.

The parameters $\theta$ can be learned by minimizing the following expected KL-divergence between the parametric policy $\pi_\theta$ and the optimal policy from Lemma~\eqref{eq:optimal_policy} (with an improper prior $p(a|s,z,g)$)
\begin{equation*}
    \mathop{\mathbb E}_{i,j}
    \mathop{\mathbb E}_{ s \sim \mathcal D_{i,j}, z \sim q_{\delta_i},  g\sim q_{\omega_j}}
     \Big[ \textrm{KL} \Big(\pi_\theta (\cdot | s, g, z) \Big\| \frac{e^{ \alpha_\pi \bar Q^{\pi_\theta}(\cdot, s, \dynlat,\rewlat)}}{Z(s,\dynlat,\rewlat)}  \Big)\Big] 
\end{equation*}
The policy loss is written as
$
    \mathcal L_\pi := \mathop{\mathbb E}_{i,j}
    \mathop{\mathbb E}_{s \sim \mathcal D_{i,j}, z \sim q_{\delta_i}, g\sim q_{\omega_j}, a \sim \pi}
    \left[\log \pi_\theta (a | s, \dynlat, \rewlat)  - Q^\textrm{min}_{\phi_{i,j}}(s, a)  \right]
$,
where the Q-function is estimated with a single sample i.e., $\bar Q^{\pi^\star}(s,a,\dynlat,\rewlat) \approx Q^\textrm{min}_{\phi_{i,j}}(s,a)$ given that $z \sim q_{\delta_i}(\cdot)$ and $g \sim q_{\omega_j}(\cdot)$. The normalizing function $Z(s,g,z)$ can be safely ignored as in~\cite{haarnoja2018soft}.

Following a similar rationale as before, the variational parameters for each context ($i,j$) can be learned  by minimizing the following KL-divergences  
\begin{align*}
& \textrm{KL} \Big(q_{\delta_i}(\cdot) \Big\| \frac{p(\cdot|i)e^{ \alpha_d \tilde D_i(\cdot )  }}{Z(i)}  \Big), \, \quad \textrm{KL} \Big(q_{\omega_j}(\cdot) \Big\| \frac{p(\cdot|i)e^{ \alpha_r \tilde G_j(\cdot )  }}{Z(j)}  \Big),
\end{align*}
which translates into 
\begin{align}
    \mathcal L_{q_{\delta_i}} :=\mathbb E_{z \sim q_{\delta_i}} \left[ \log \frac{q_{\delta_i} (z)}{p(z)} -\alpha_d \tilde D_{i}(z)\right] \, \quad \mathcal L_{q_{\omega_j}} :=\mathbb E_{g \sim q_{\omega_j}} \left[ \log \frac{q_{\omega_j} (g)}{p(g)} -\alpha_r \tilde G_{j}(g)\right] 
\end{align}
where, for clarity, we define $
\tilde D_{i}(z) := \mathop{\mathbb E}_{j, g\sim q_{\omega_j}}
[\kappa_{i, j}(z, g)], \quad
 \tilde G_{j}(g)  :=  \mathop{\mathbb E}_{i, z\sim q_{\delta_i}}[\kappa_{i, j}(z, g)]
$
 with $\kappa _{i, j}(z, g) = \mathop{\mathbb E}_{a\sim \pi_\theta, s \sim p}\bigg[ Q_{\phi_{i,j}}(s,a) - \frac{1}{\alpha_\pi} \log \pi_\theta(a|s,\dynlat,\rewlat)\bigg]$.   All remaining algorithmic details are in the Appendix.

\section{Experiments}
\newcommand\widthB{90}
\begin{figure}[t]
    \centering
    \includegraphics[width=\widthB pt]{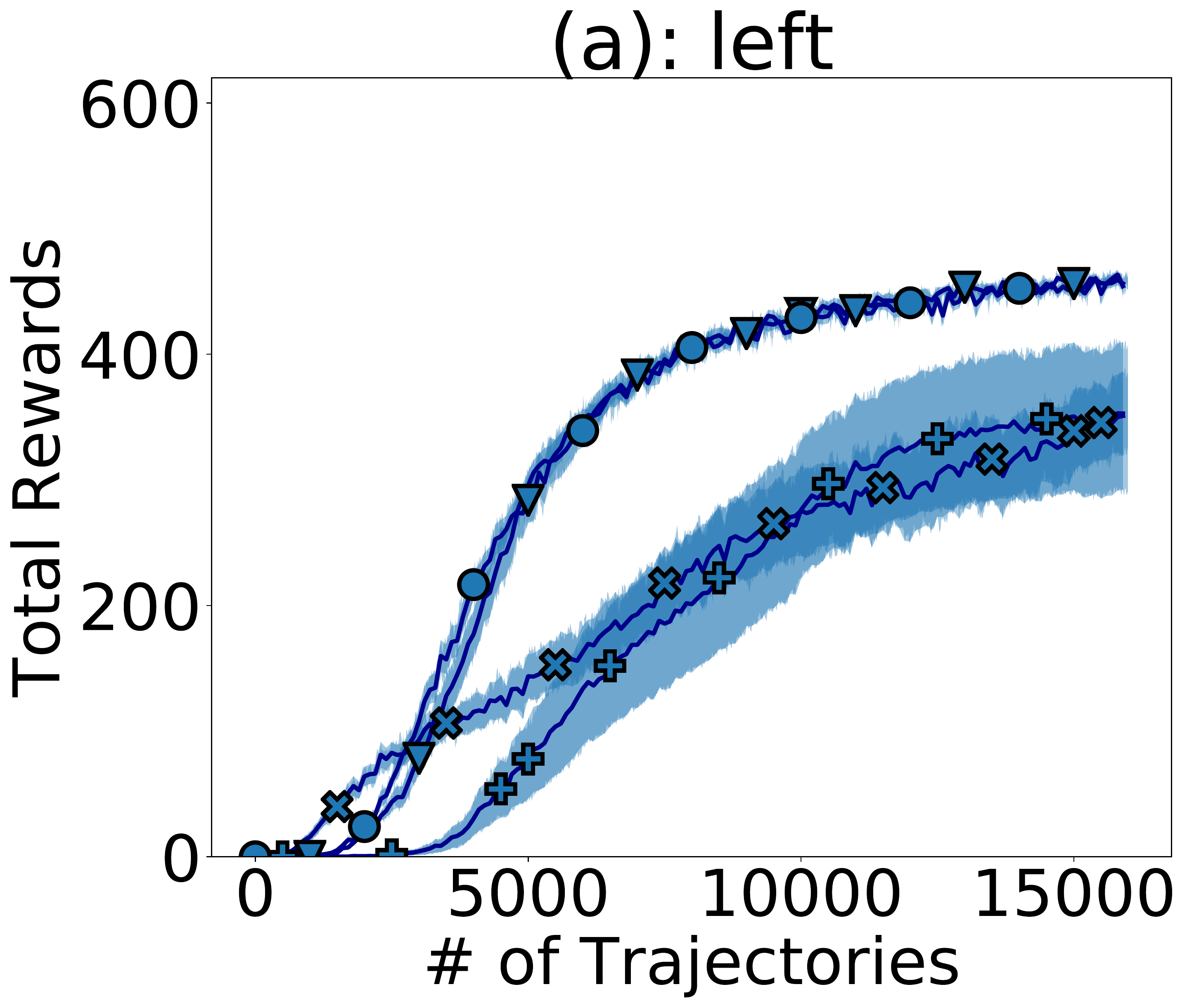}
    \includegraphics[width=\widthB pt]{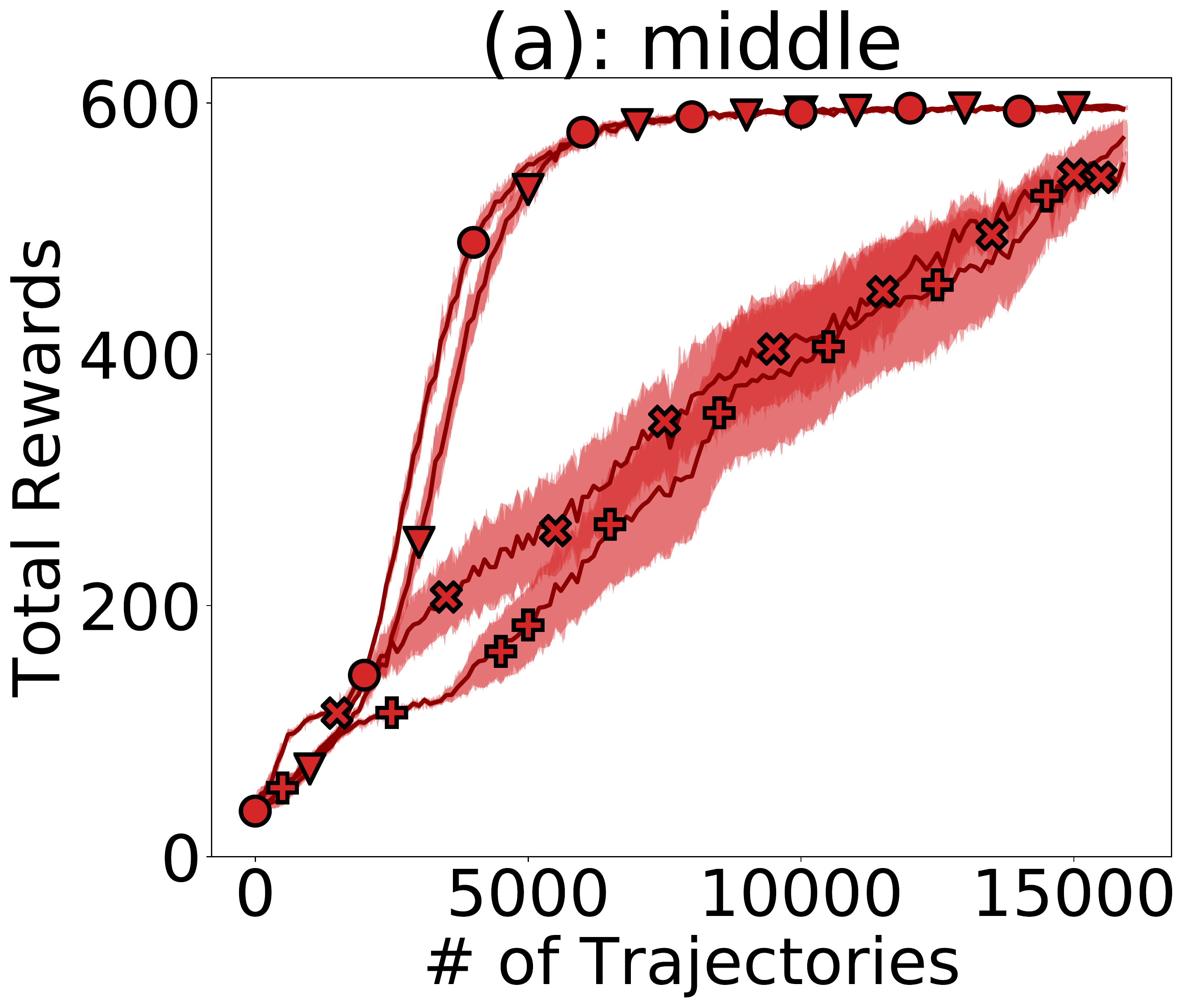}
    \includegraphics[width=\widthB pt]{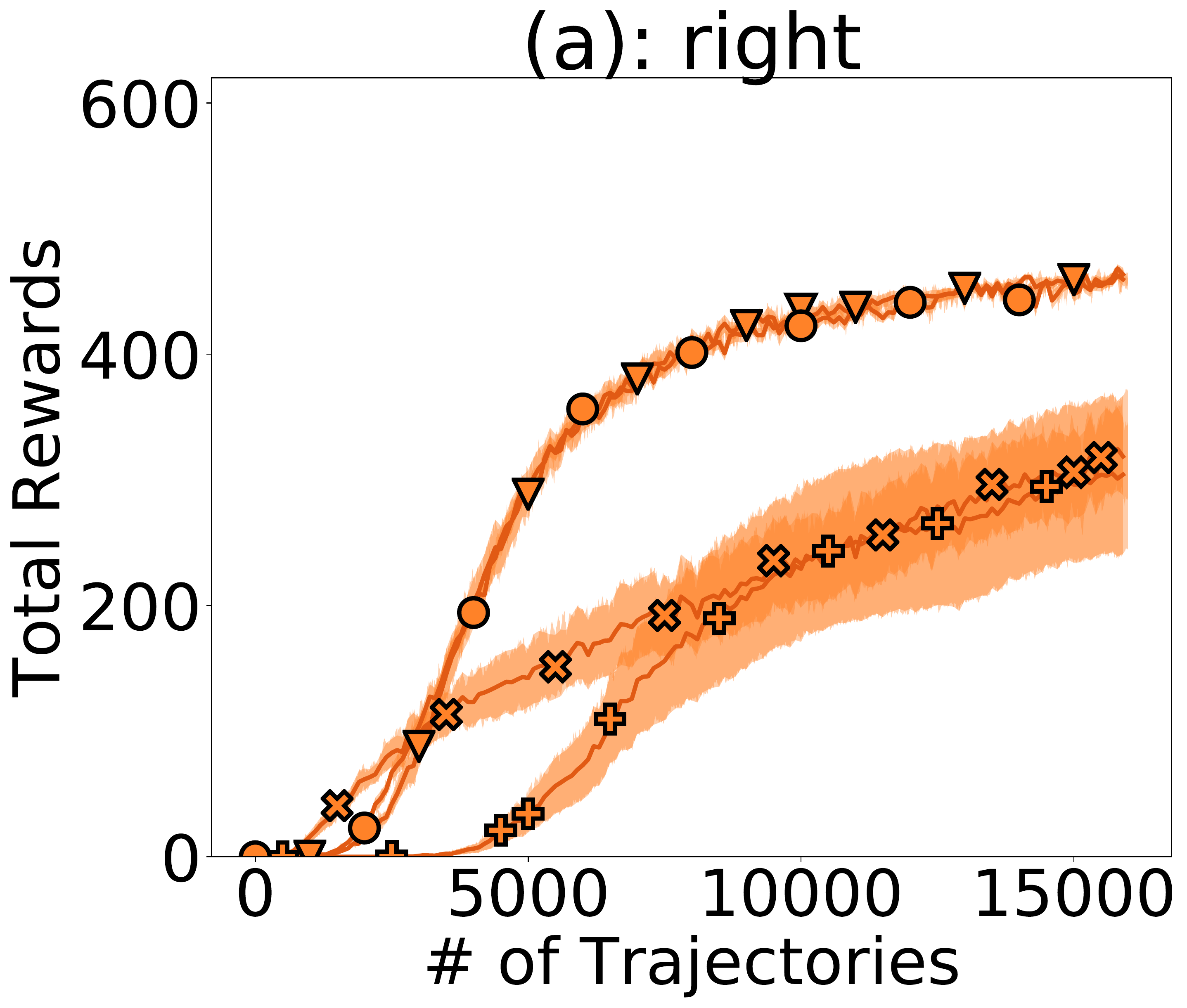}
    \includegraphics[width=\widthB pt]{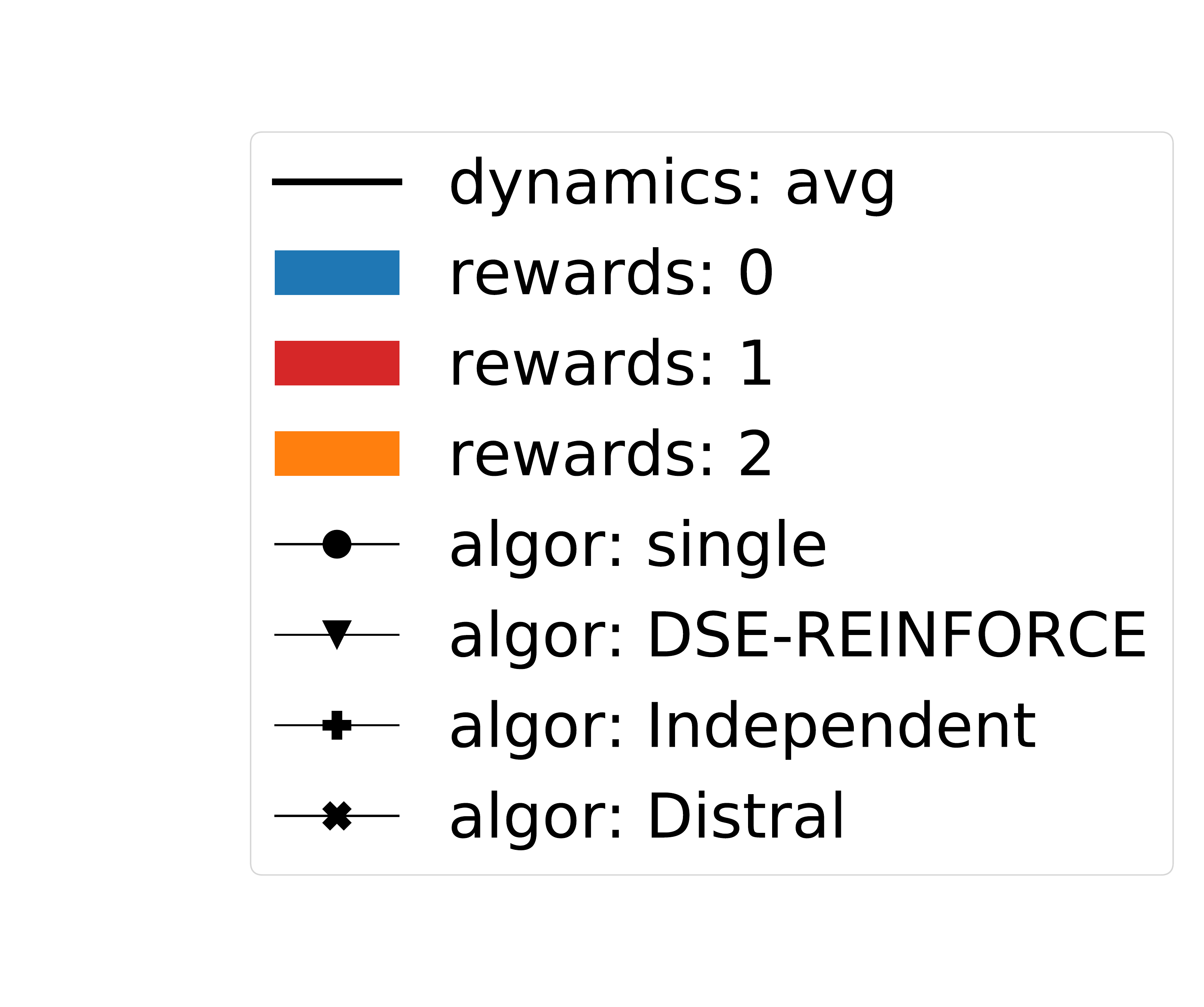}\\
    \includegraphics[width=\widthB pt]{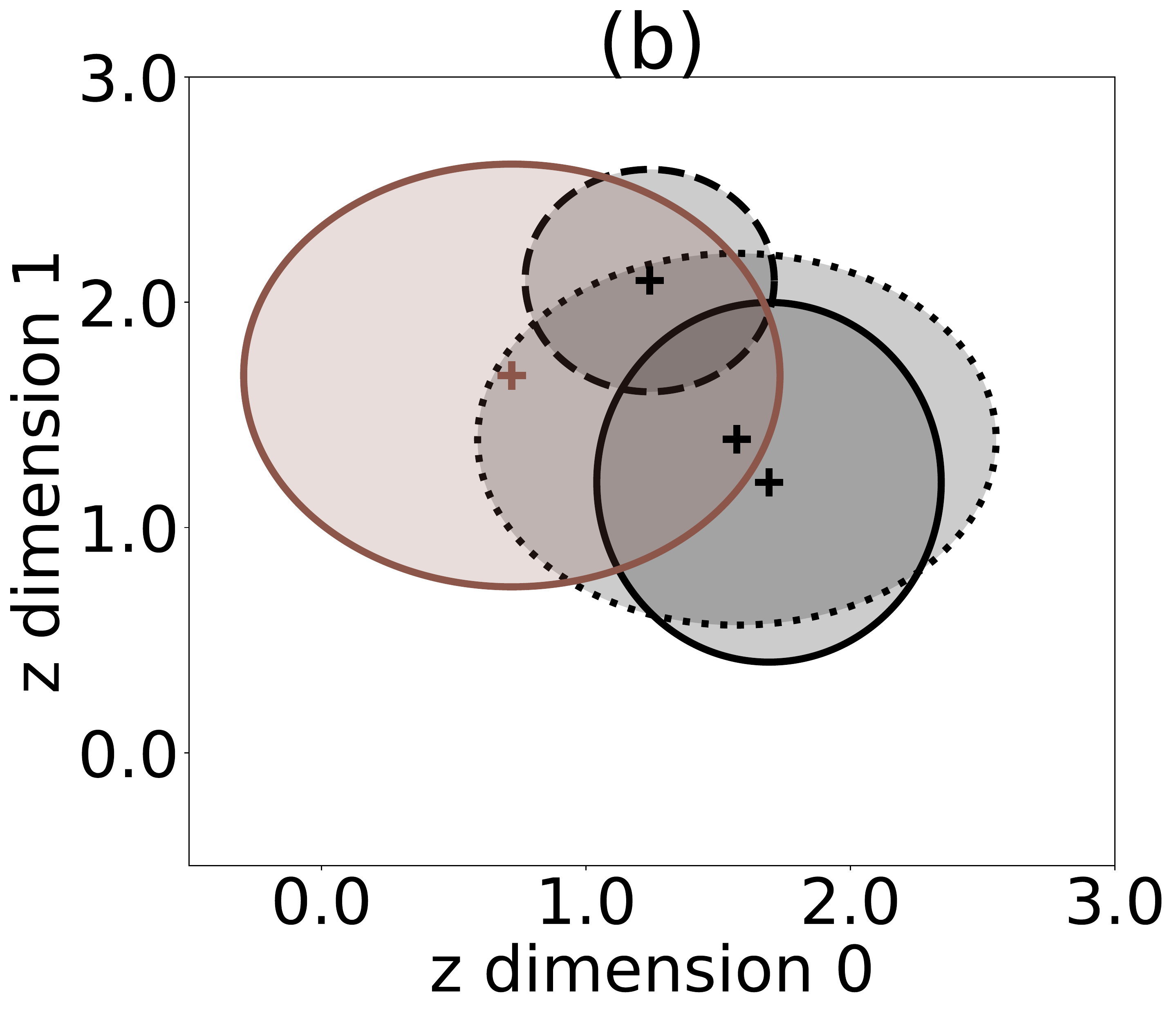}
    \includegraphics[width=\widthB pt]{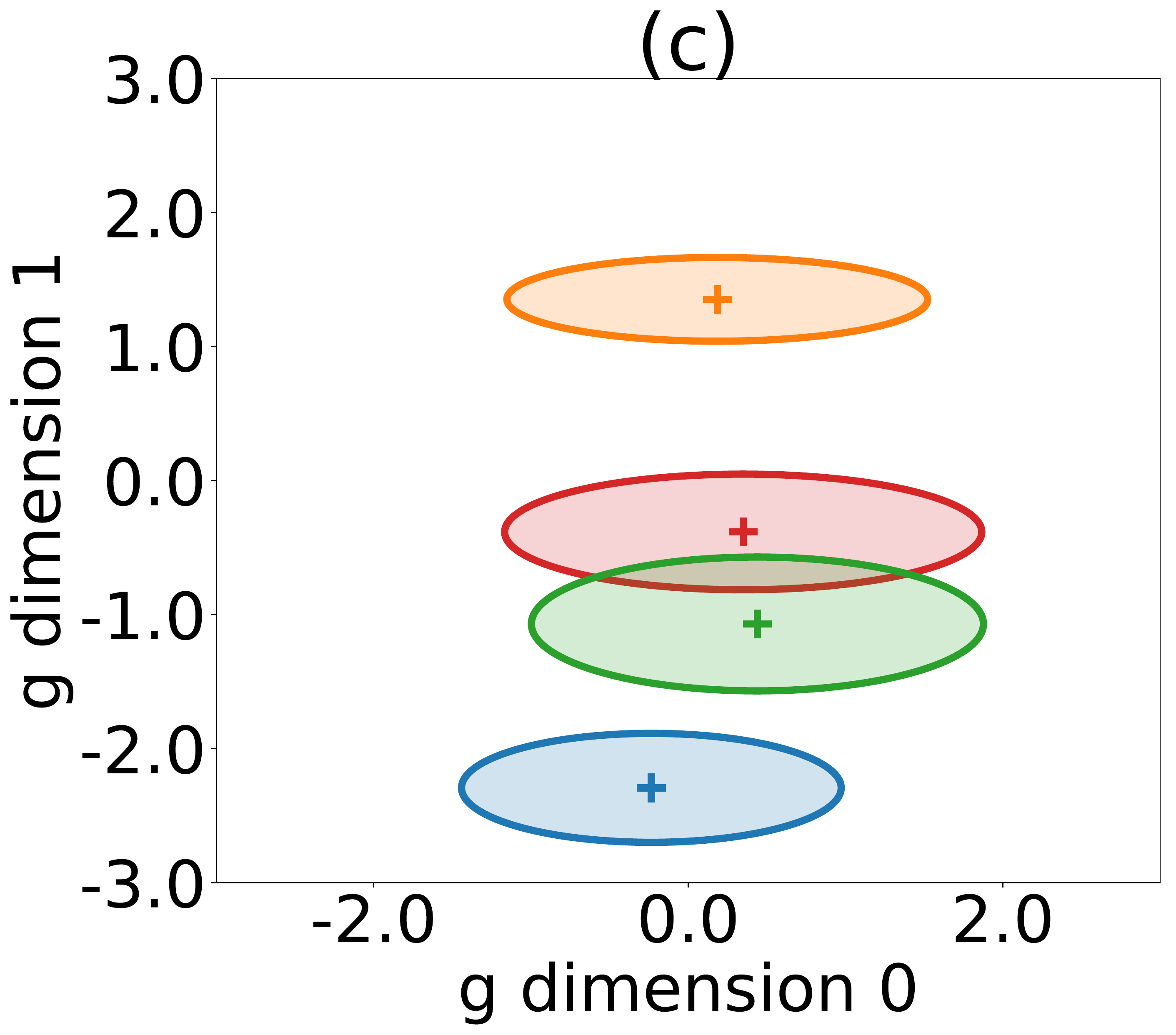} 
    \includegraphics[width=\widthB pt]{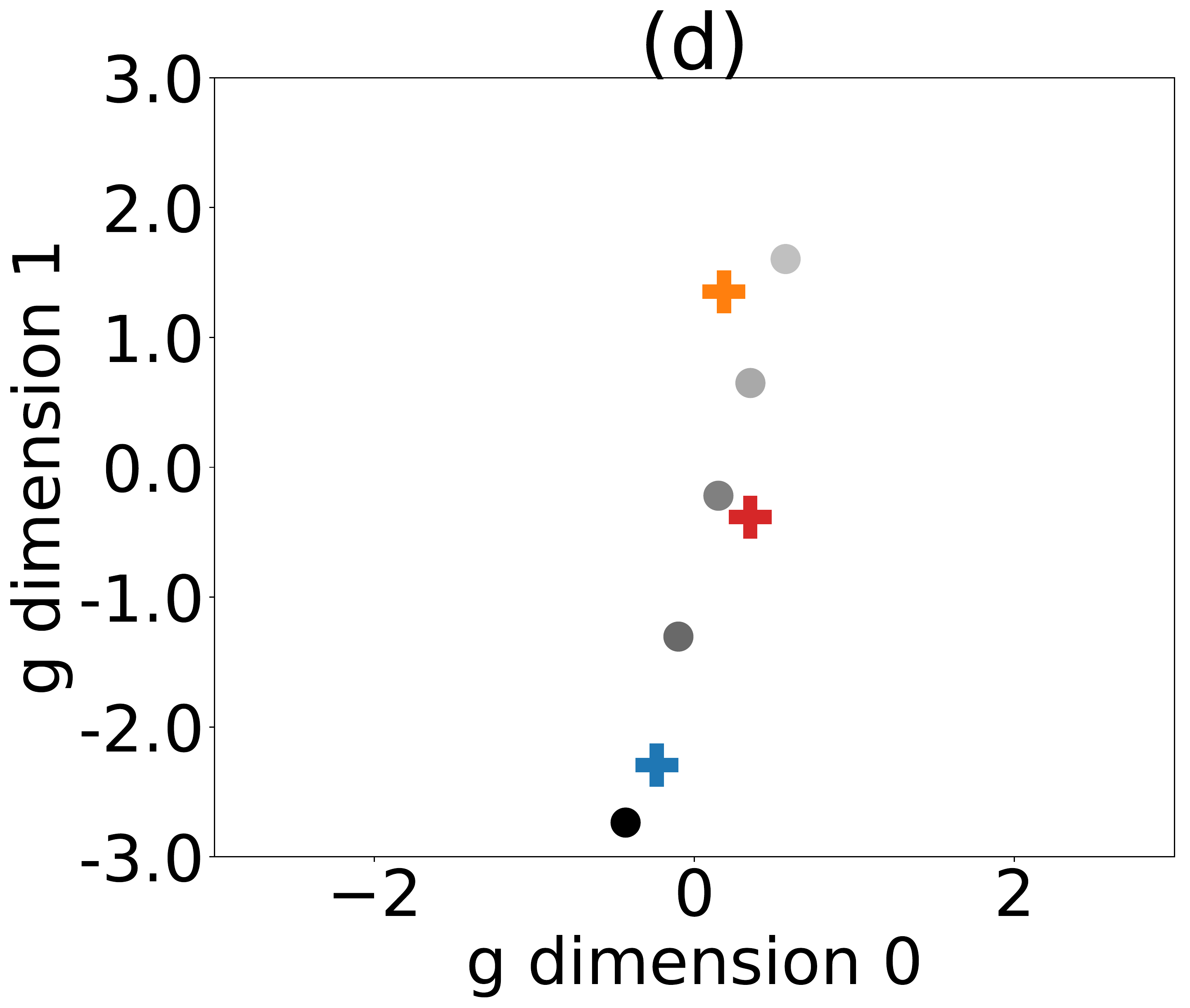}
    \includegraphics[width=\widthB pt]{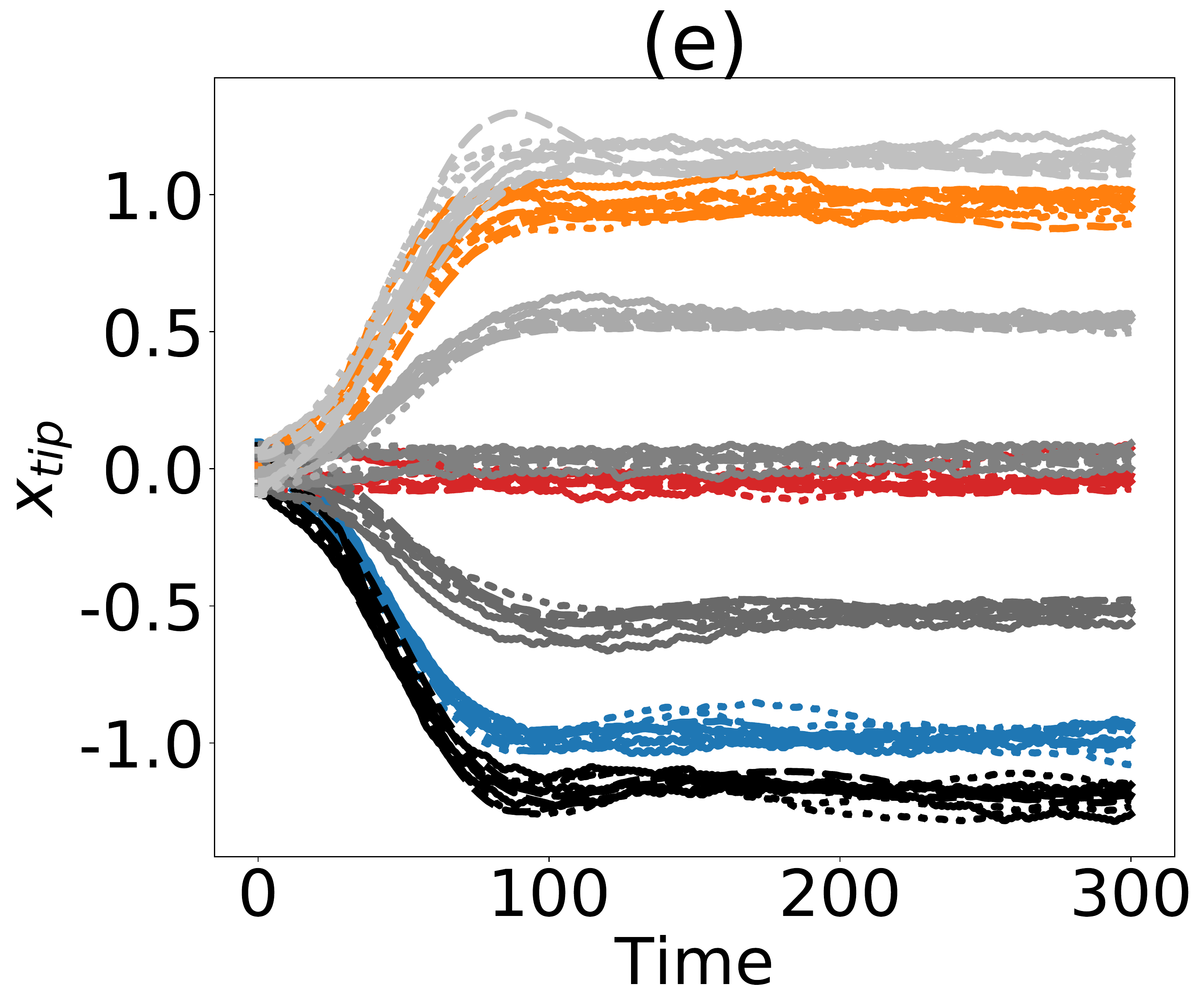}
    \caption{Multi-task Cartpole. The colors correspond to different goal contexts. (a) shows the average evolution of the total rewards across 4 random seeds; shaded regions are the standard error, separated by the goals. For clarity, we averaged over dynamics (full plots in Appendix C.1). (b) and (c) show the latent spaces $g$ and $z$ respectively after the training. The distributions shown are $1.5\times$ standard deviation. (d) and (e) are results from a generalization experiment carried out on $g$ (see Sec~6.1 for description).}
    \label{fig:cartpole_learning_results}
\end{figure}

Here we empirically validate our algorithms and show the applicability of the trained  policies equipped with DSE on both multi-task and hierarchical RL problems. DSE-REINFORCE is tested on a discrete action-space problem (Cartpole) and DSE-SAC on a continous action-space problem (Reacher\footnote{From the Mujoco dynamics simulation software}). On multi-task problems we show the benefit of disentanglement  when compared to three baselines: single-embedding algorithm similar to~\cite{hausman2018learning}, Distral~\cite{teh2017distral} and independent learners.  Hyperparameter values are shown in the Appendix B.3.

\subsection{DSE-REINFORCE on Cartpole}
\label{sec:cartpole_main}
We extended the Cartpole environment provided in the Open AI gym library\footnote{\url{https://gym.openai.com/envs/CartPole-v1/}} by modifying the reward function to reflect the need to balance at different locations: left ($x=-1$), middle ($x=0$) and right ($x=1$).  Additionally, we allowed for three different dynamics conditions by changing the mass of the cart $m=\{0.2, 1.0, 2.0\}$. Simulations were run for $300$ time steps. 

Figure~\ref{fig:cartpole_learning_results}(a) shows that DSE-REINFORCE solves all nine tasks simultaneously at approximately the same rate exceeding the performance of the baselines: Distral \citep{teh2017distral} and independently trained (no multi-task; trained with REINFORCE) algorithms; and performing similarly to the single embedding case. Importantly, we find that DSE-REINFORCE produces a policy that generalizes better than the baselines as we show in the next section.
In Figure~\ref{fig:cartpole_learning_results}(b) we observe the variational distributions learned for embeddings of the different dynamics contexts (in grey) and in (c) for the different reward contexts (in red, orange and blue). These have separated to represent the different tasks in the latent space. The variational distributions shown in dark-red color in (b) and green color in (c) are the result of learning (with identical priors on $z, g$ and conditioned on the trained shared parameters) in a new unseen condition ($ m=1.75, x=-0.5$) successfully solving the task. This shows that the latent spaces are able to interpolate well. 
In Figure~\ref{fig:cartpole_learning_results}(d) we show the mean of the variational distributions for the goal contexts in color; and in grey, latent vectors that we used to test whether the learned policy is able to generalize to unseen goals. We show in panel (e) the x-location of the tip of the pole. As can be seen through the grey conditions, the policy is able to generalize to new locations in an ordered (along the x-axis) fashion. 

\paragraph{Retraining and generalization of DSE-REINFORCE on Cartpole} 
In this section we test generalization when there are missing dynamics or goal conditions on a $3\times3$ task matrix. We consider the case of training on off-diagonal tasks ($6-3$) and testing on the diagonal; and the case of training on only $4$ tasks ($4-5$) (See Figure~\ref{fig:retraining_cartpole}). The testing phase is executed on each test-task by initializing the variational distributions with matching indices  and retraining both the variational and shared parameters.

\newcommand\widthA{75}
\begin{figure}[t]
    \centering
    \includegraphics[width=\widthA pt]{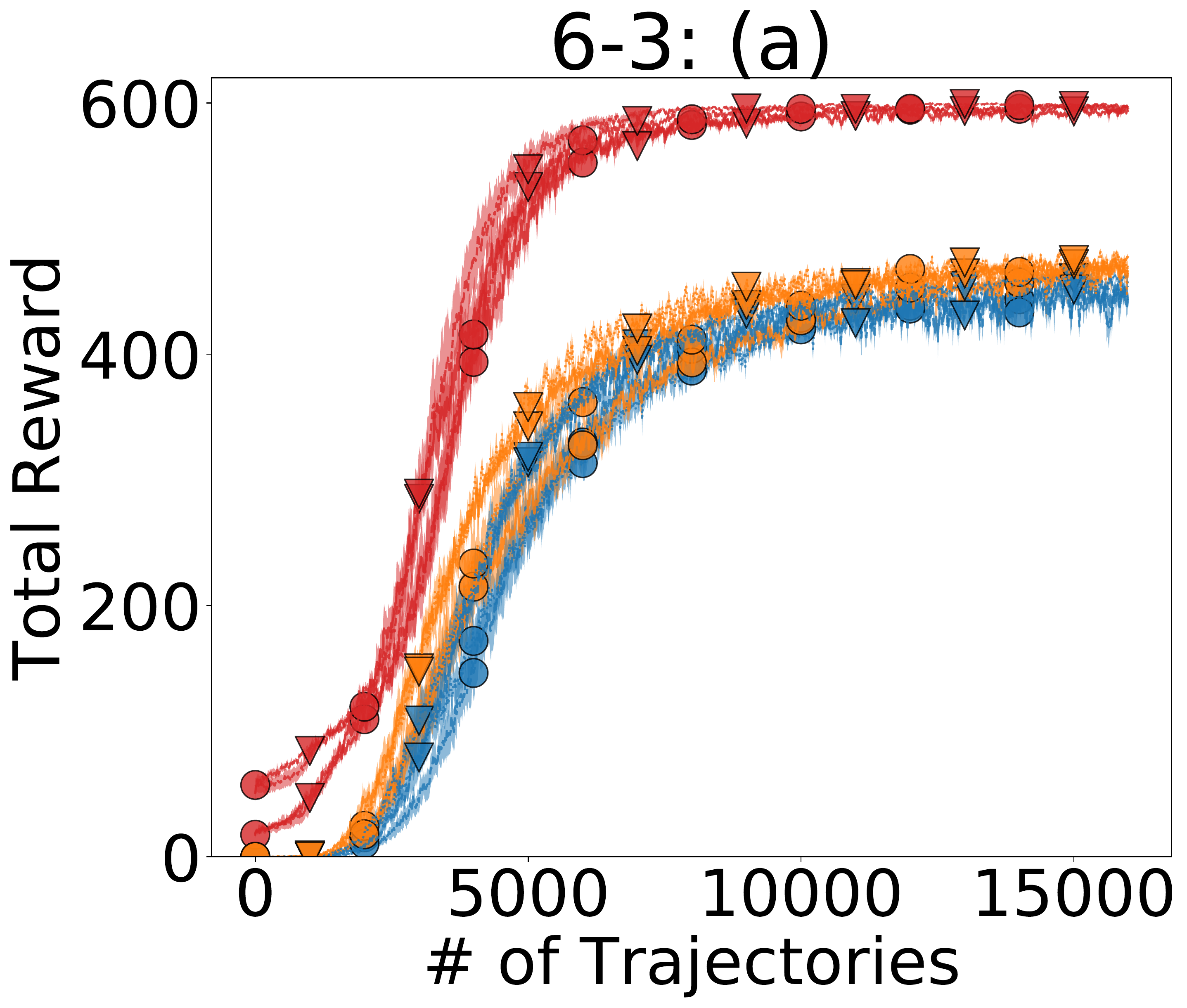}
    \includegraphics[width=\widthA pt]{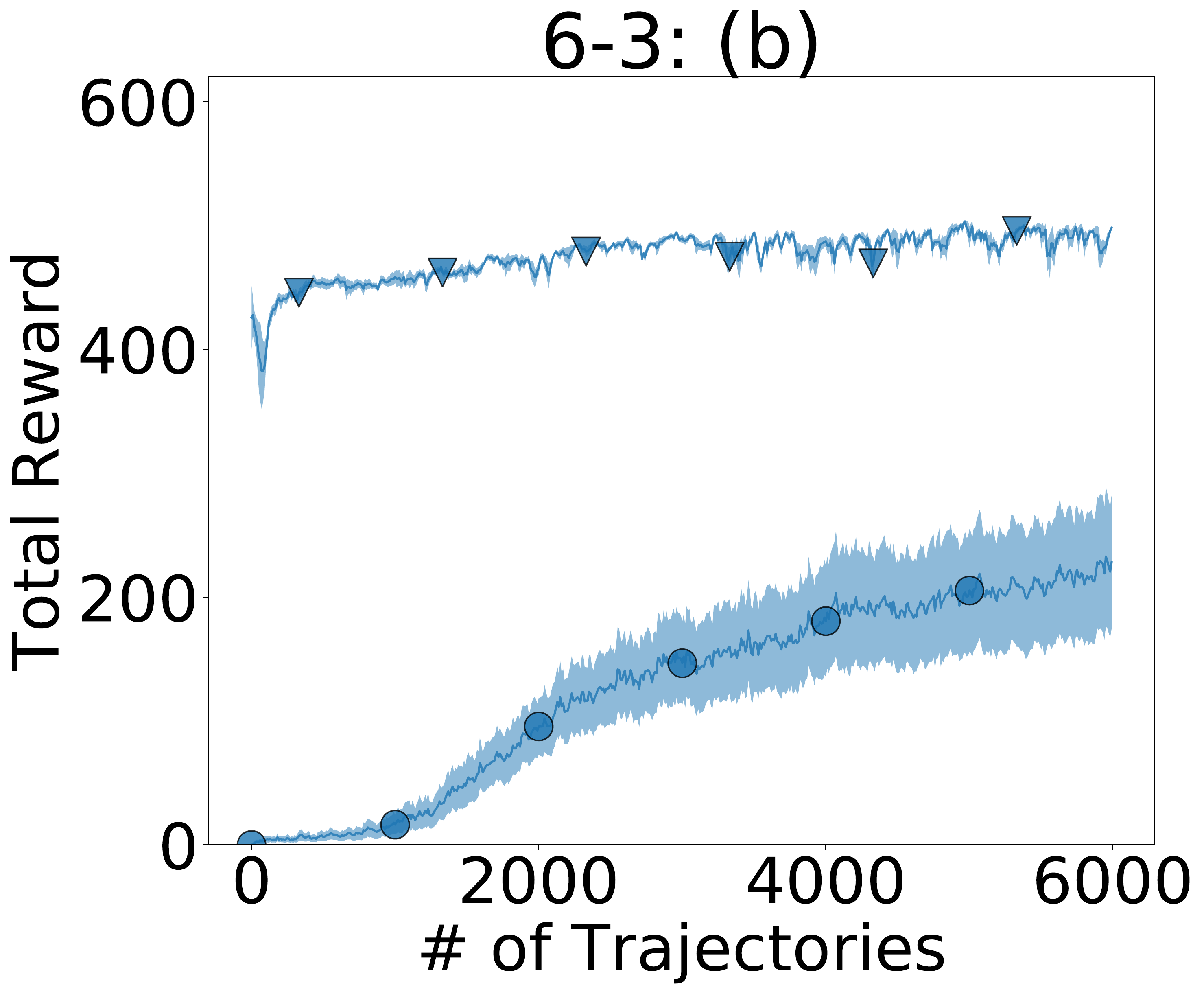}
    \includegraphics[width=\widthA pt]{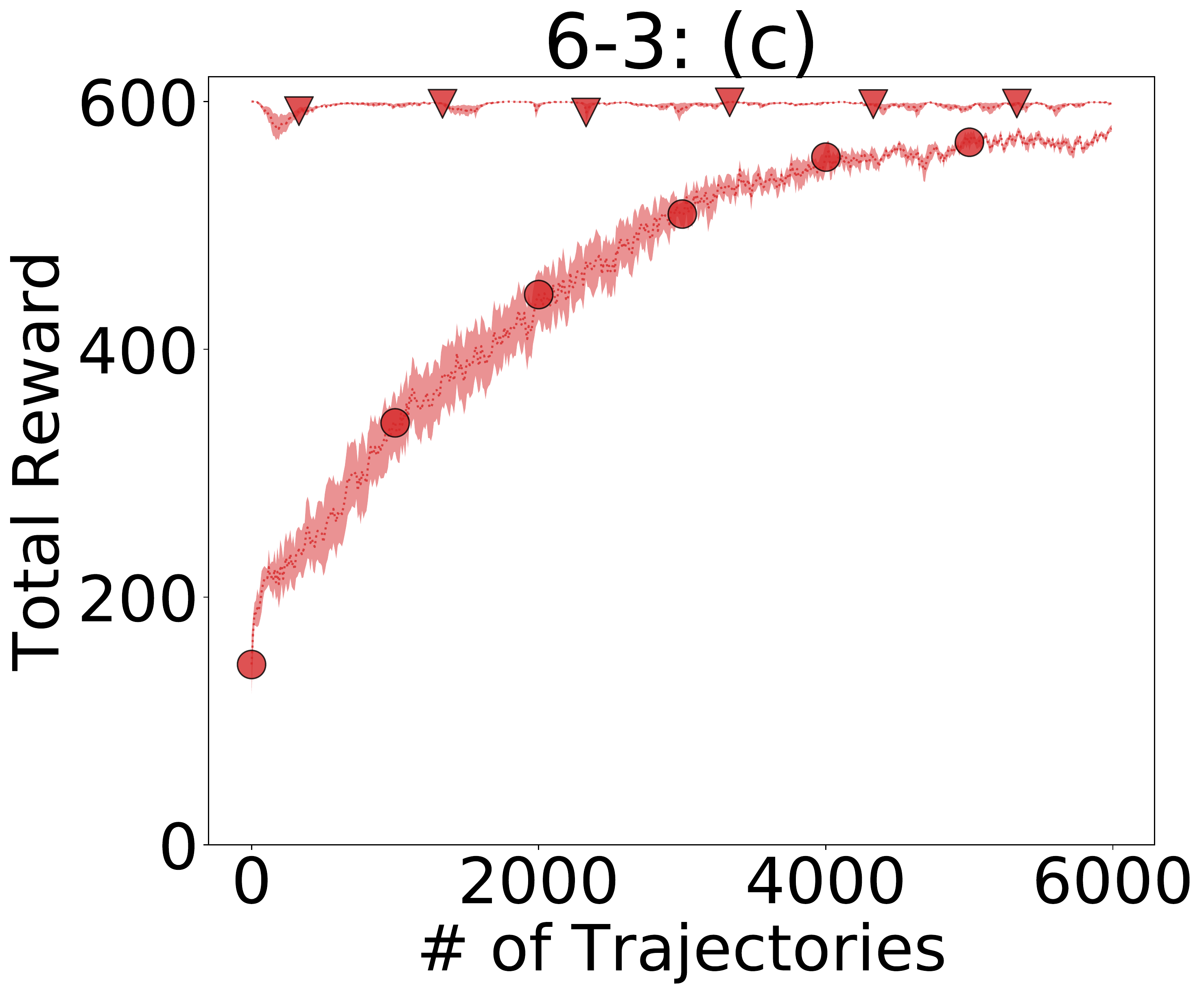}
    \includegraphics[width=\widthA pt]{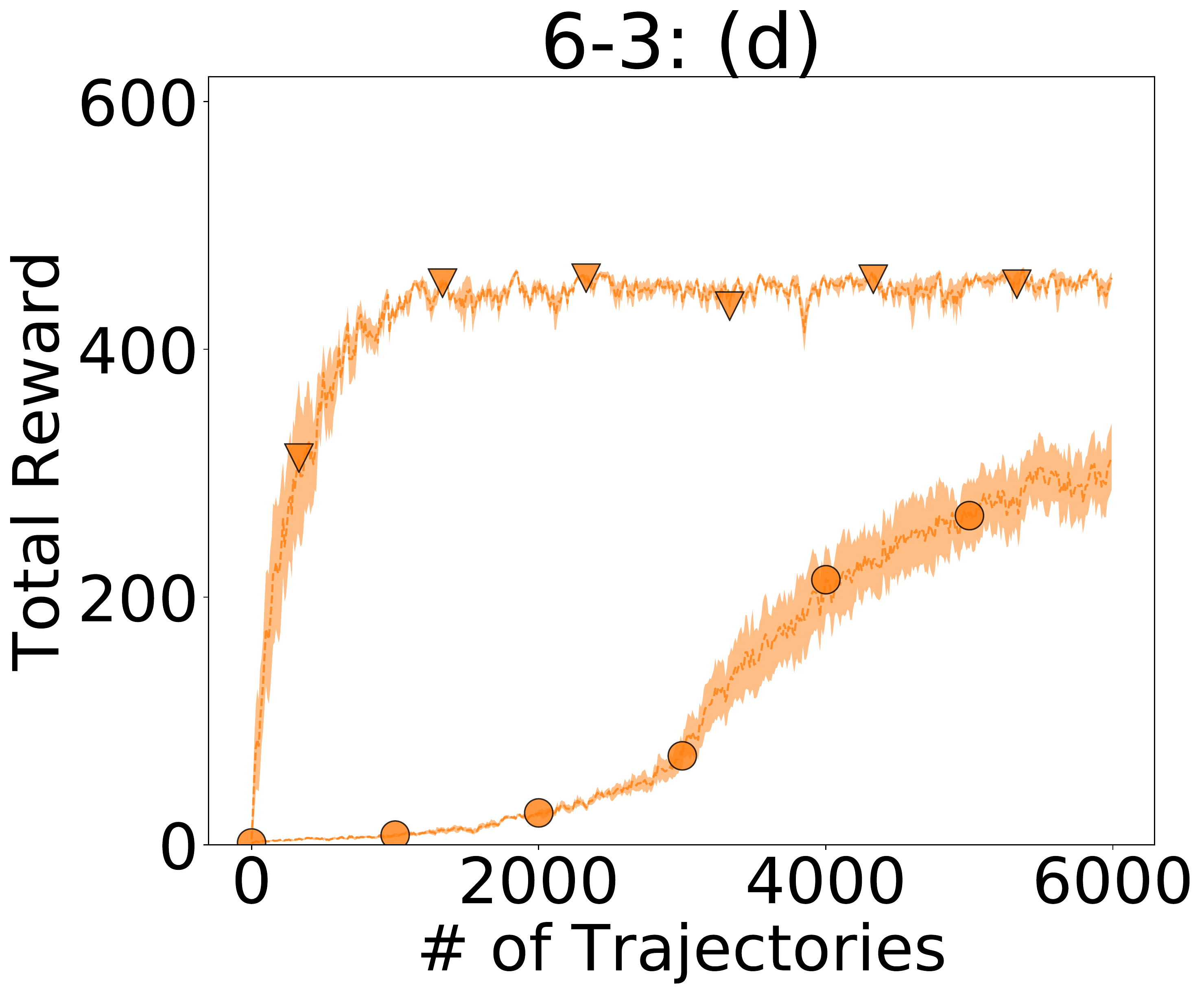}
    \includegraphics[width=75 pt]{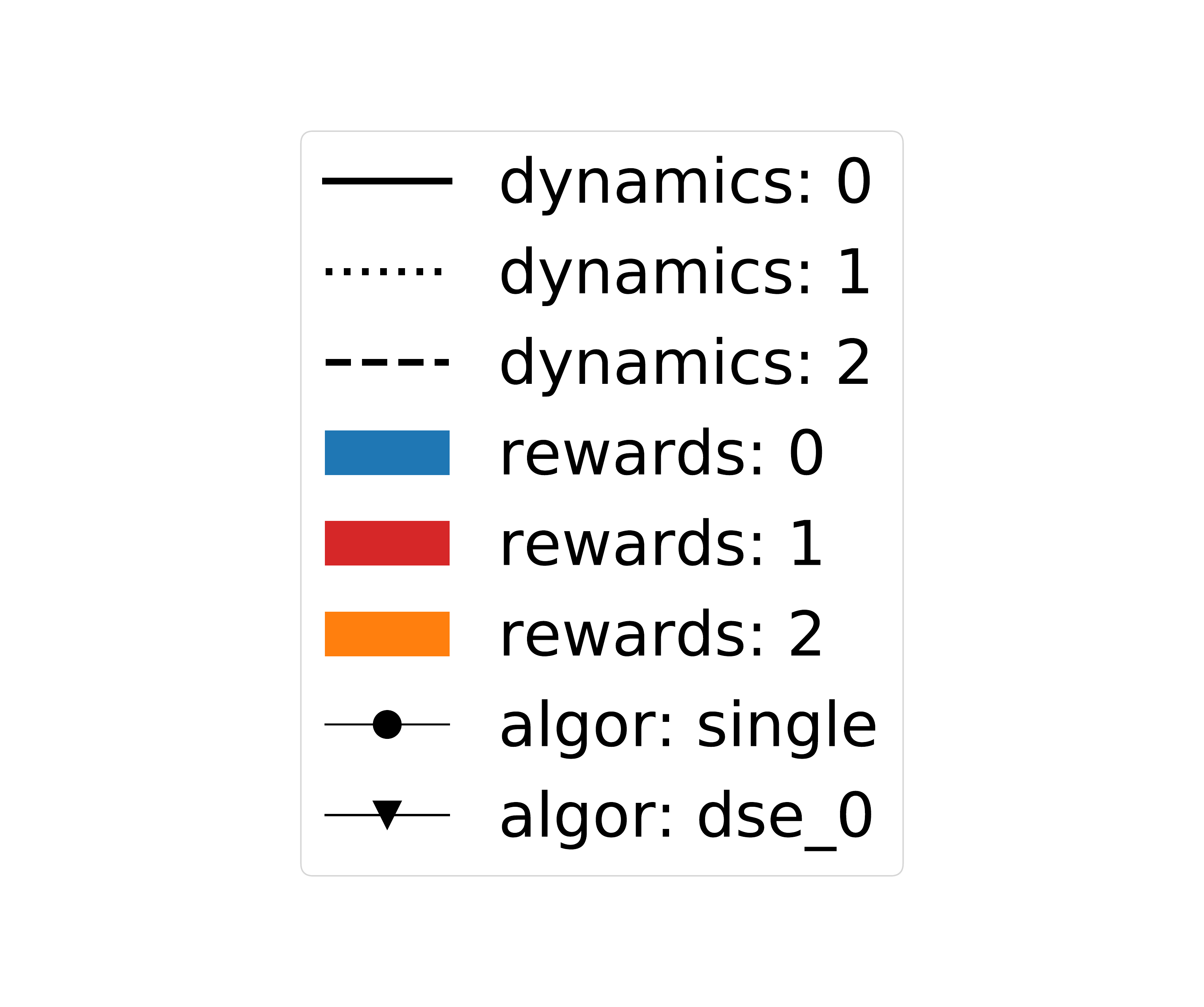}\\
    \hspace{-27pt}
    \includegraphics[width=\widthA pt]{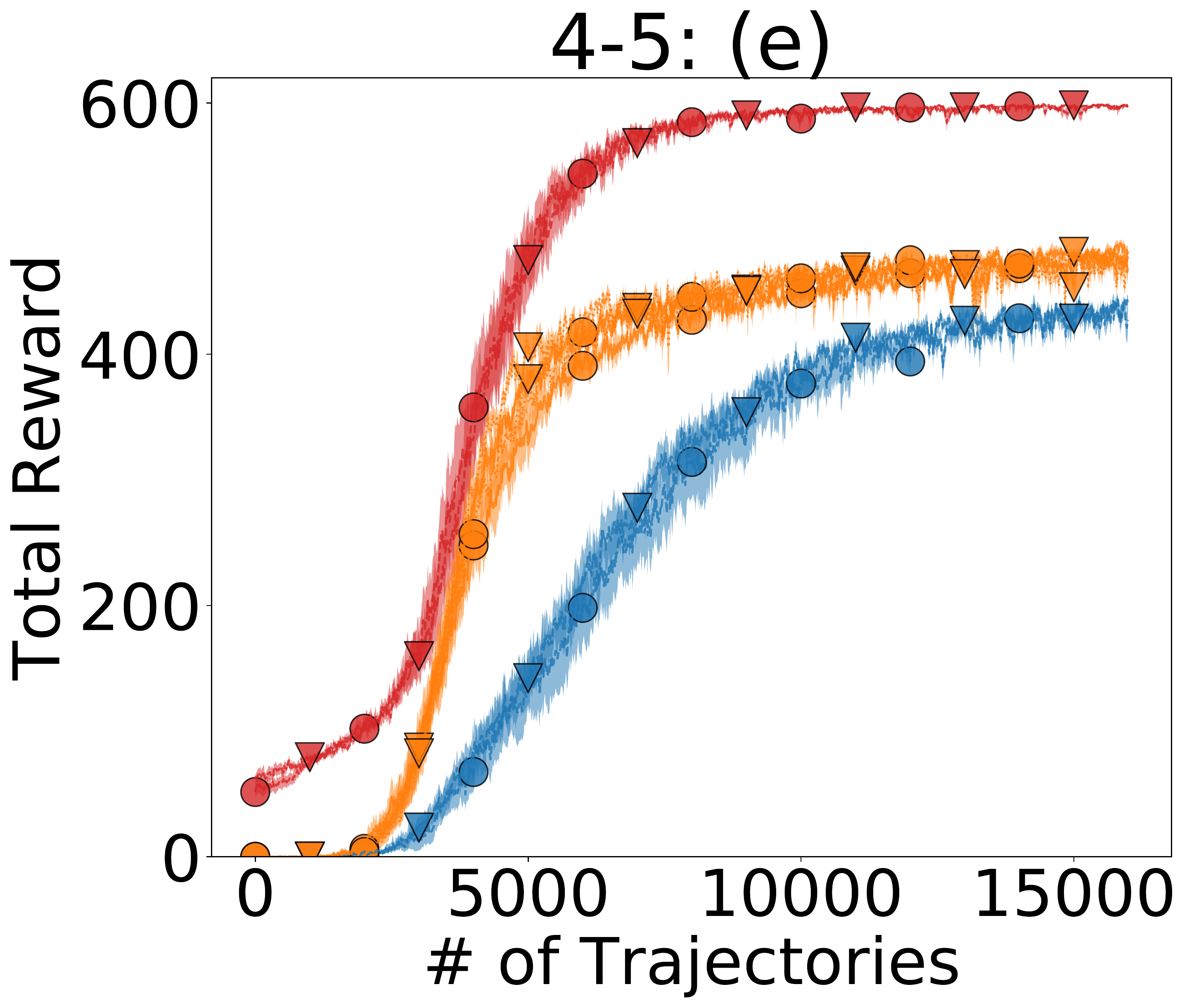}
    \includegraphics[width=\widthA pt]{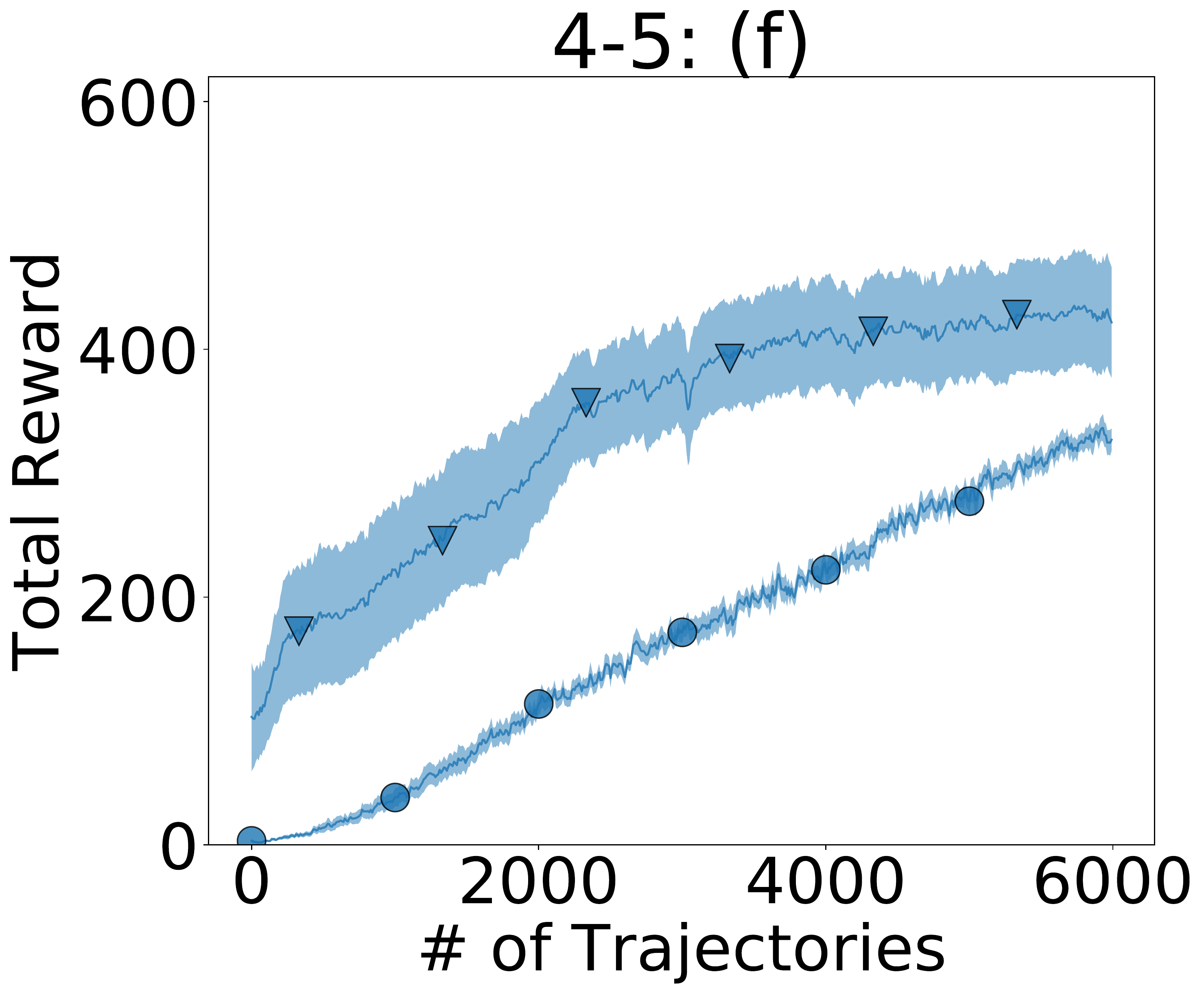}
    \includegraphics[width=\widthA pt]{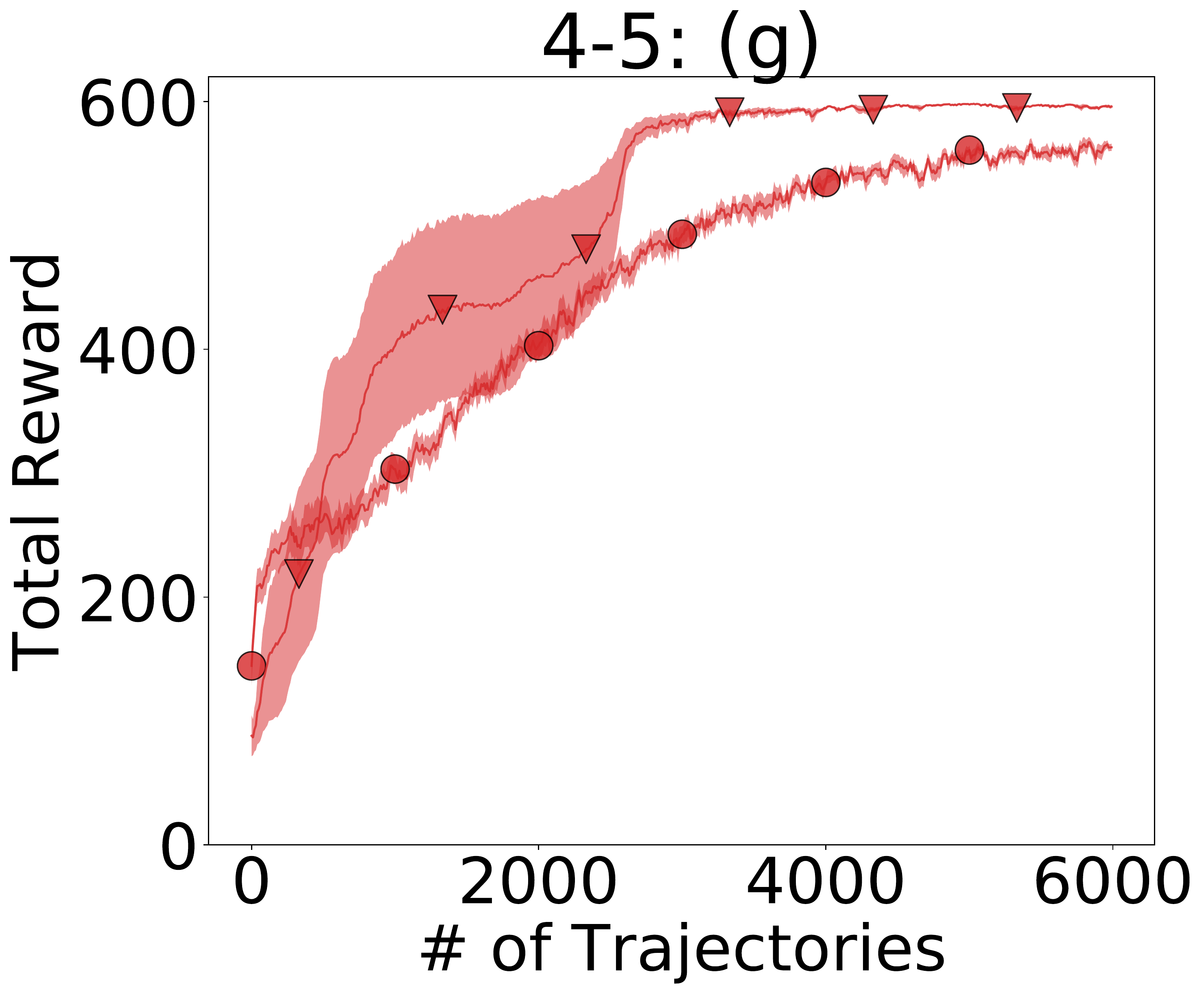}
    \includegraphics[width=\widthA pt]{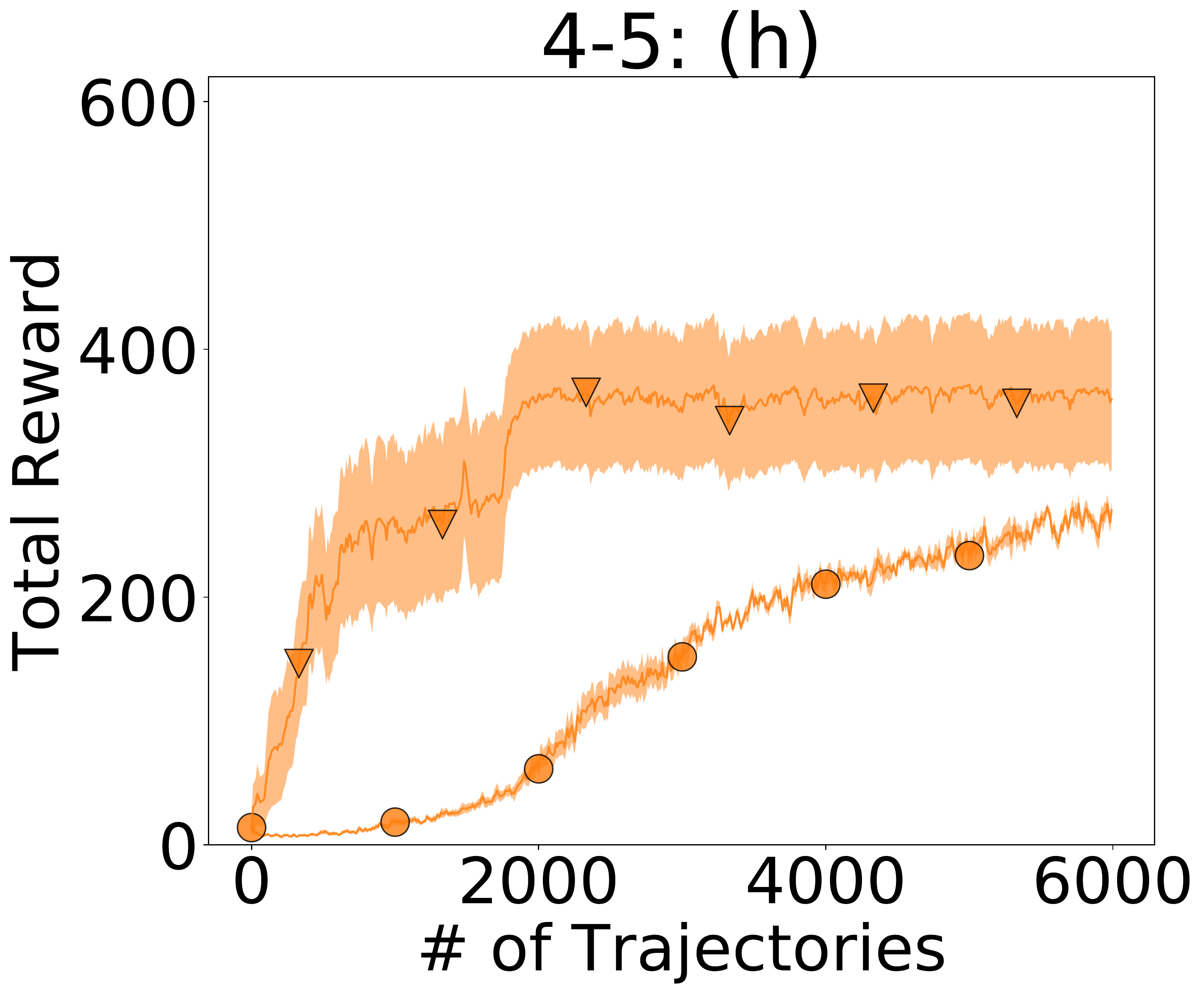}
    \hspace{17pt}
    \includegraphics[width=30 pt]{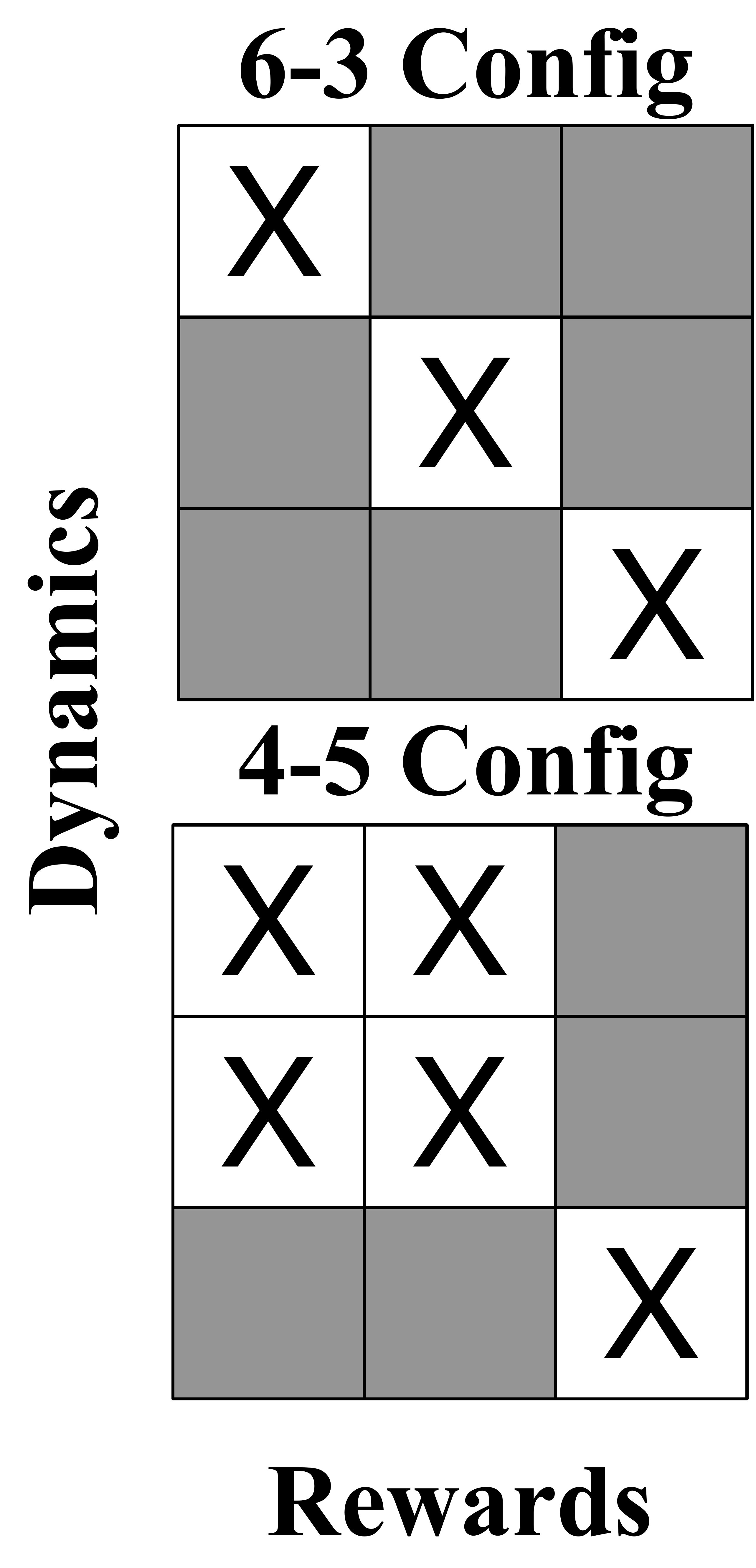}
    
    \caption{Retraining experiments on 6-3 and 4-5 configurations for 2 algorithms. The configurations are shown in the bottom right panel; the setting marked by \textbf{X} were omitted during initial training ((a) and (e)). (a-h) show the average reward curves and their standard errors over 5 random seeds. The columns are organized according to: first column the multitask training and the rest of columns correspond to  retraining of $\mathrm{reward}=0$, $\mathrm{reward}=1$ and $\mathrm{reward}=2$ respectively. (a-d) were for the 6-3 configuration; the rest were for the 4-5 configuration. For brevity, the reward curves for 4-5 in (e-f) were also averaged across similar goals. The hyperparameters were the same as for the complete MTRL case in Section \ref{sec:cartpole_main}. }
    \label{fig:retraining_cartpole}
\end{figure}

In Figure \ref{fig:retraining_cartpole}, we show on the left-most panels the multi-task training for both (6-3) and (4-5) settings and on the remaining panels the performance of the testing phase. We compare our DSE-REINFORCE (dse0) against the single-embedding algorithm from previous section. We can clearly see the benefits of disentangling the dynamics from the reward; the DSE algorithm provides strong initializations for tasks never seen before that are not mere ``interpolation"-tasks as tested in the previous section. Note that independent single-task training would need about 10000 trajectories to train whereas DSE sometimes solves the test-task instantaneously (without accounting for the multi-task trainning).

\paragraph{HRL on Cartpole}
\begin{wrapfigure}{R}{0.50\textwidth}
    \centering
    \vspace{-5mm}
    \includegraphics[width=0.23\textwidth]{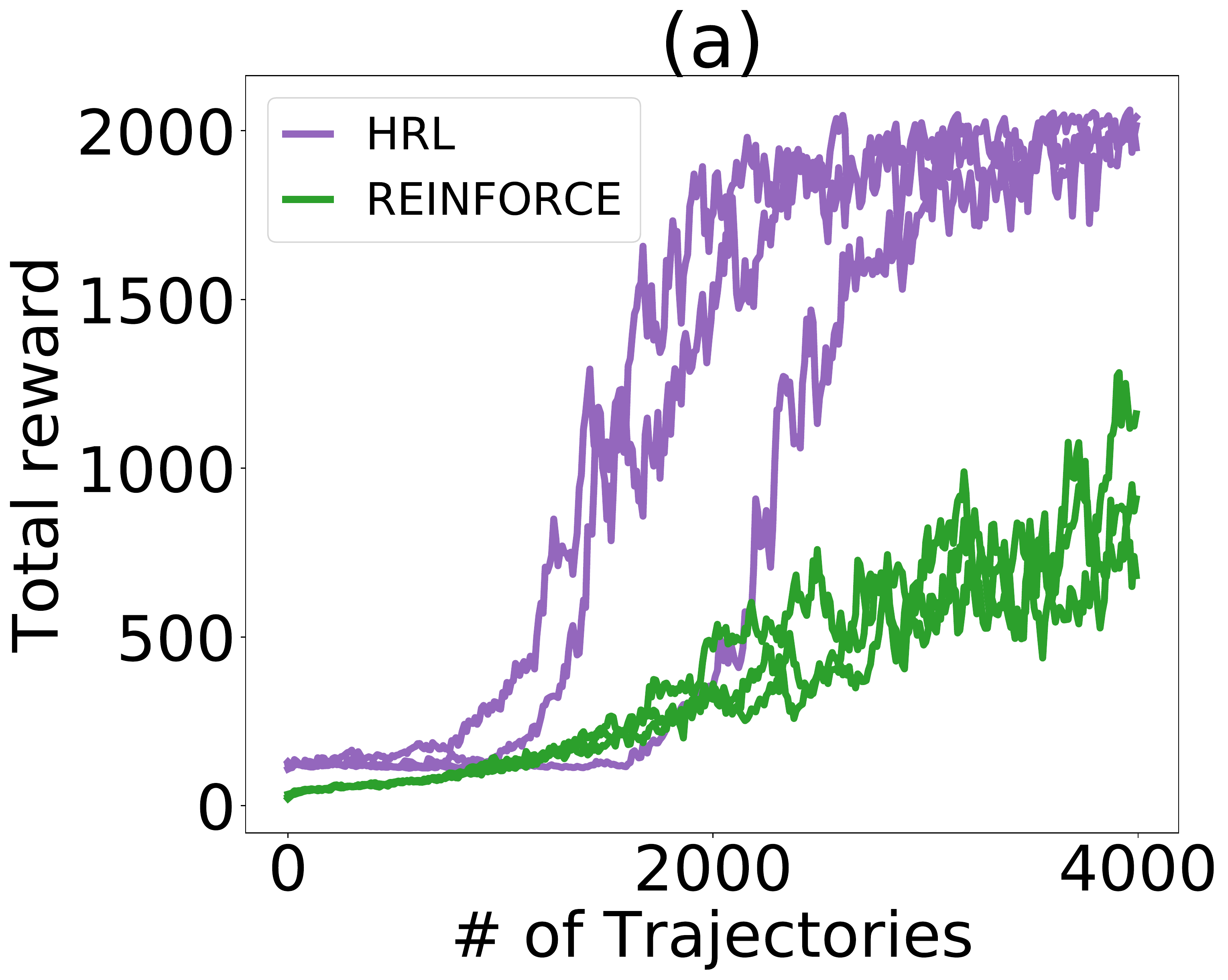}
    \includegraphics[width=0.225\textwidth]{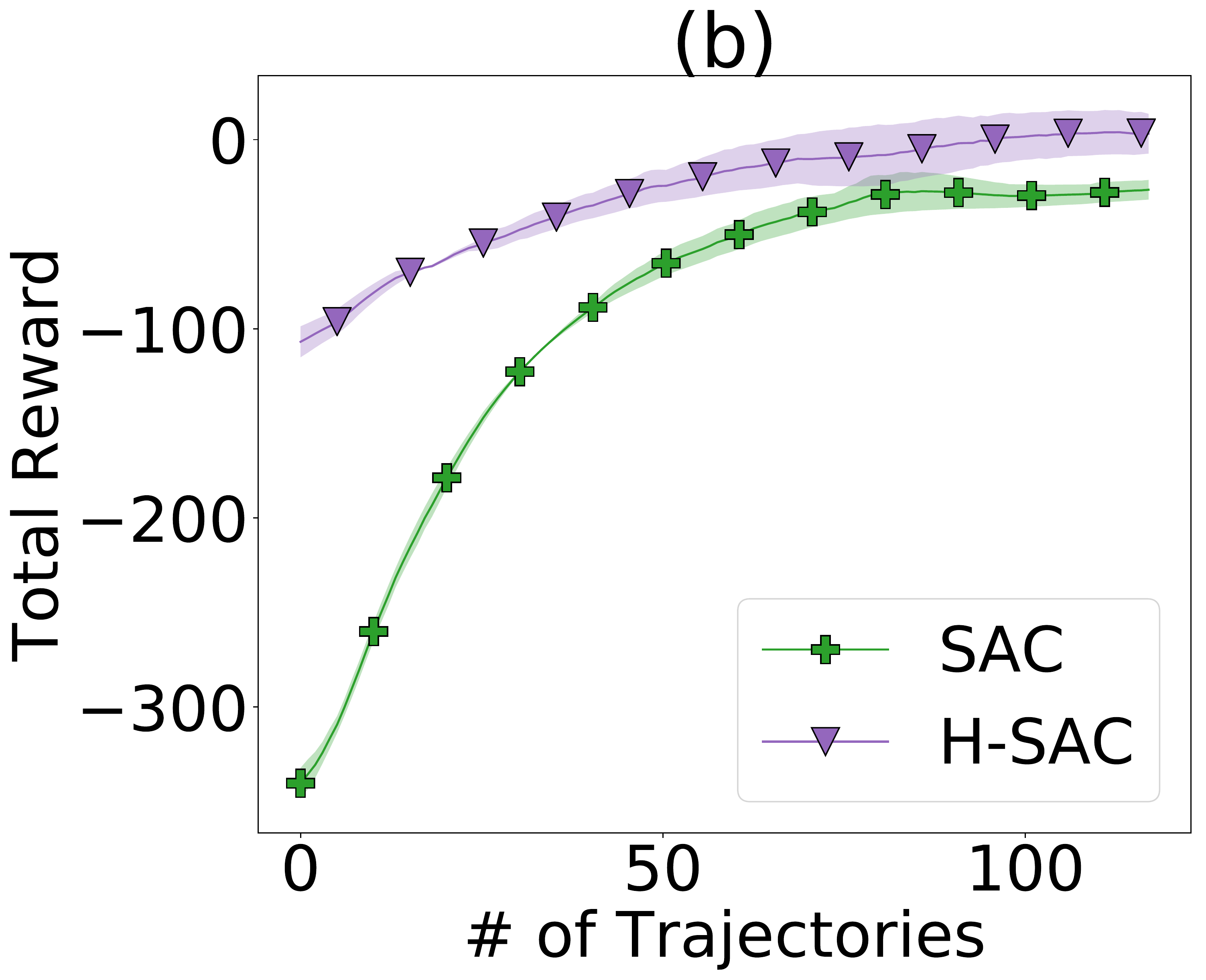}
    \caption{Evolution of the total reward for HRL problems. (a) is the 1-Asteroid AsteroidCartpole problem. (b) is the Hierarchical Reacher problem.}
    \label{fig:hrl_asteroid_reward}
\end{wrapfigure}

We test the validity of the trained policies equiped with DSE in an HRL scenario by training a  high-level policy that acts on the rewards latent space. For this, we developed a novel cartpole problem (AsteroidCartpole), where a balanced cartpole must avoid falling asteroids; this is detailed in the Appendix. For this, we fixed the mass to $m=1$; the latent variable for $z$ was fixed to the mean of $q_{\delta_{i}}(\cdot|i; m=1)$. The high-level policy acted on a discrete action space consisting of 5 selected points of rewards latent space; three were the means of the learned variational distributions for the goal-contexts and the remaining two were interpolations $\mu = \{[-0.1, -1.30], [0.35, 0.65]\}$. Figure~\ref{fig:hrl_asteroid_reward}(a) shows the evolution of the episodic rewards while training the high-level policy (HRL) with standard REINFORCE equipped with a baseline and with Pop-Art. As a comparison, we also trained the same REINFORCE algorithm but acting directly on the low-level actions.  As seen, the hierarchical policy outperforms the baseline and attains maximum reward.

\subsection{DSE-SAC on Reacher}
\newcommand{\widthC}{90}
\begin{figure}[t]
    \centering
    \includegraphics[width=\widthC pt]{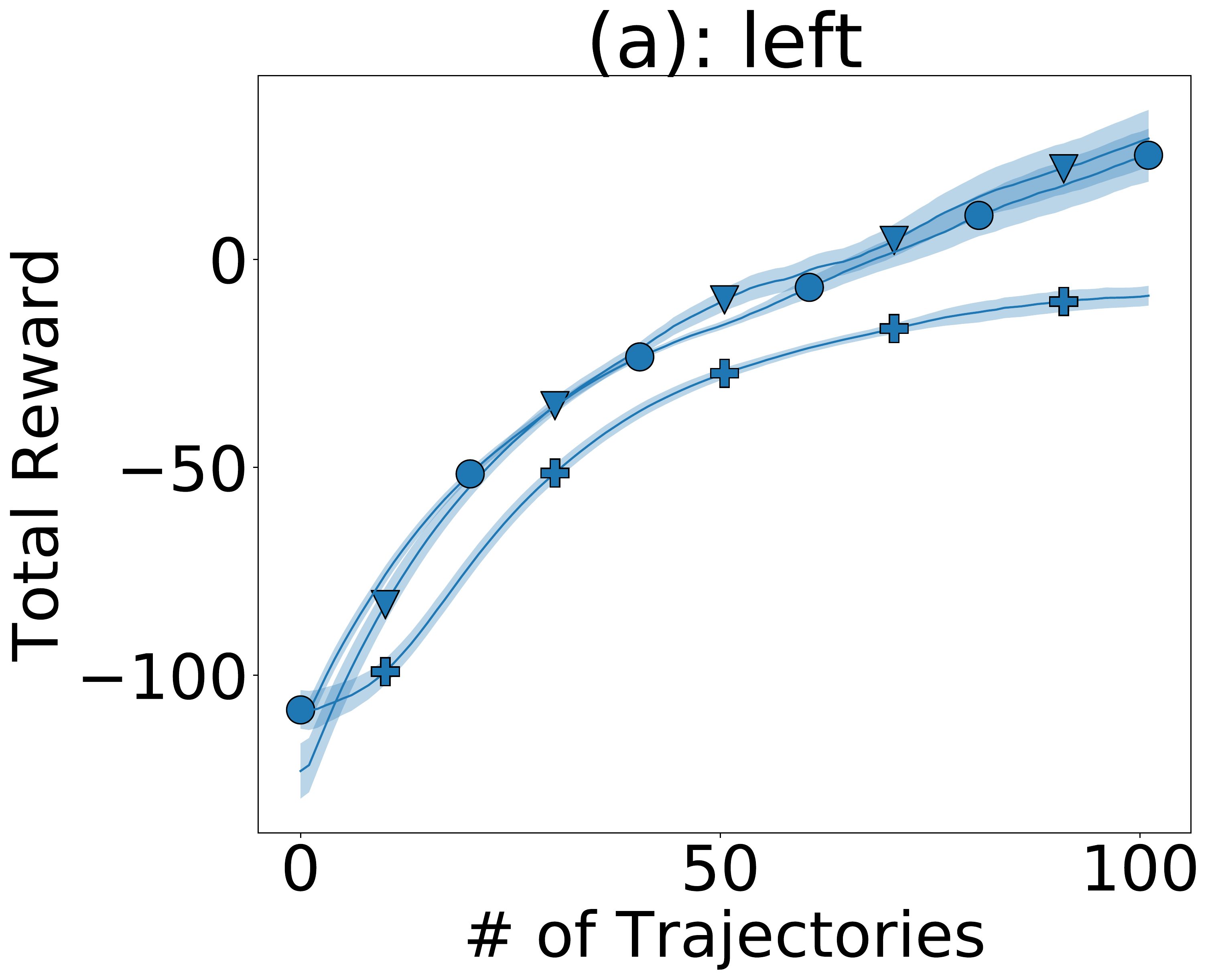}
    \includegraphics[width=\widthC pt]{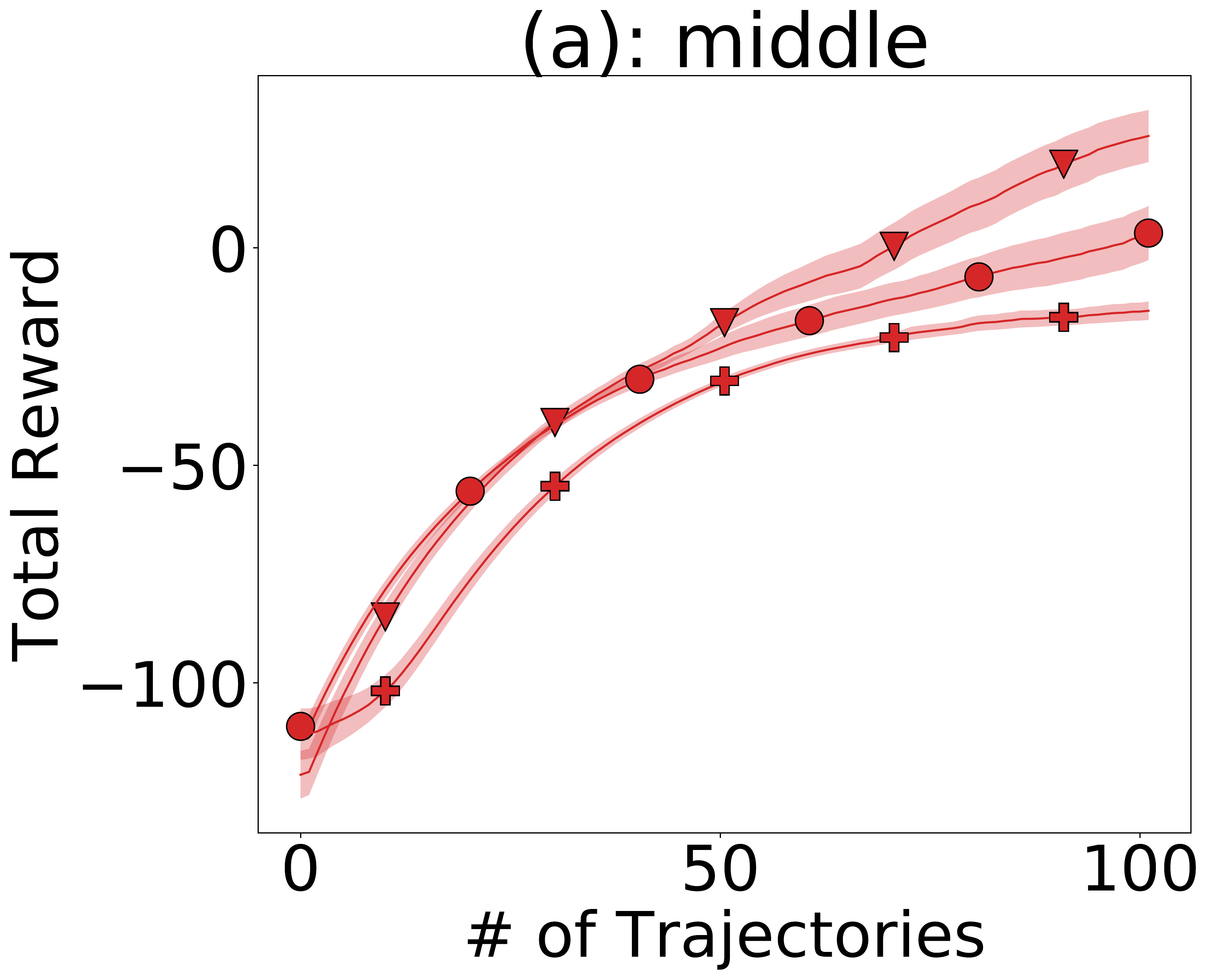}
    \includegraphics[width=\widthC pt]{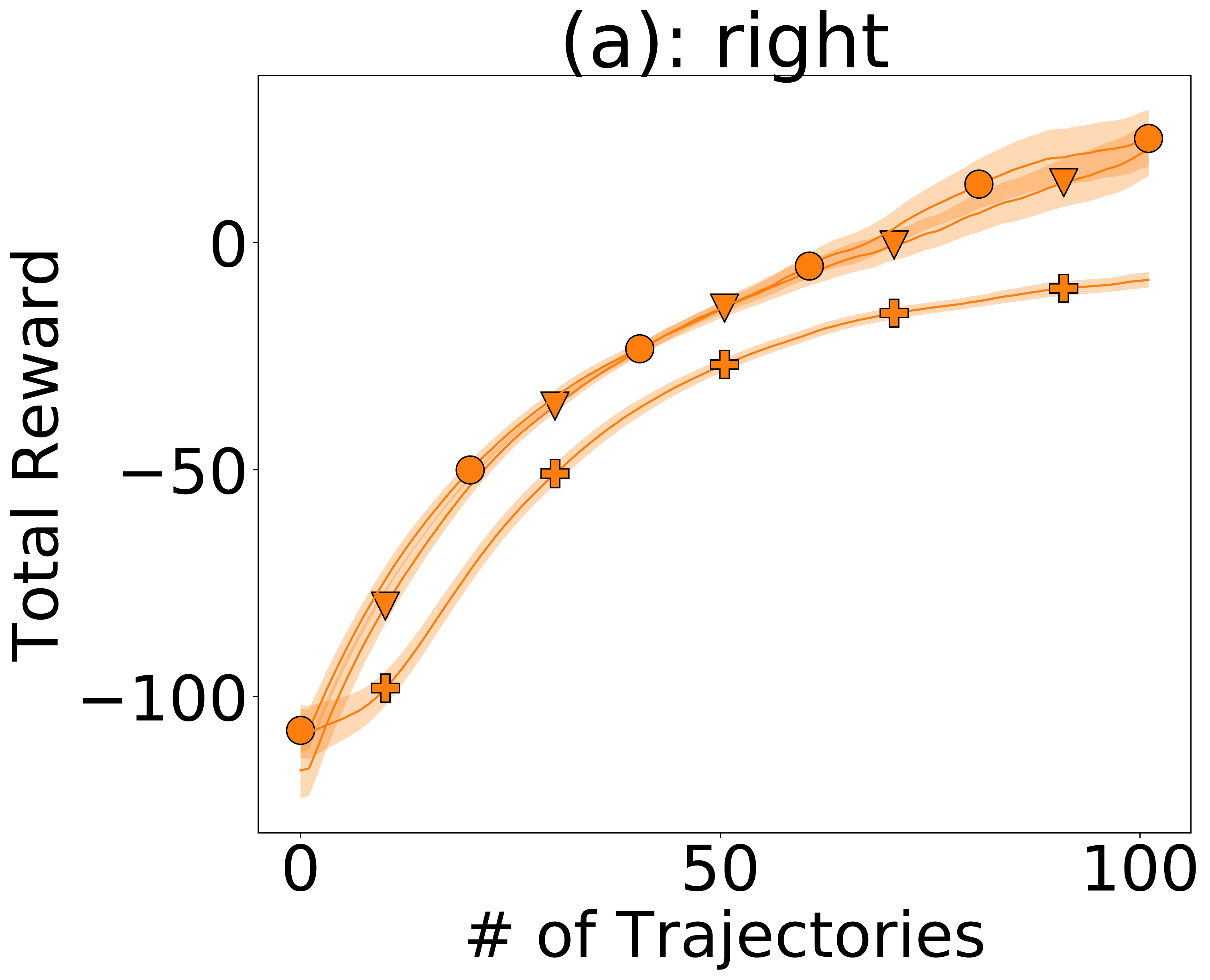}
    \hspace{10pt}
    \includegraphics[width=\widthC pt]{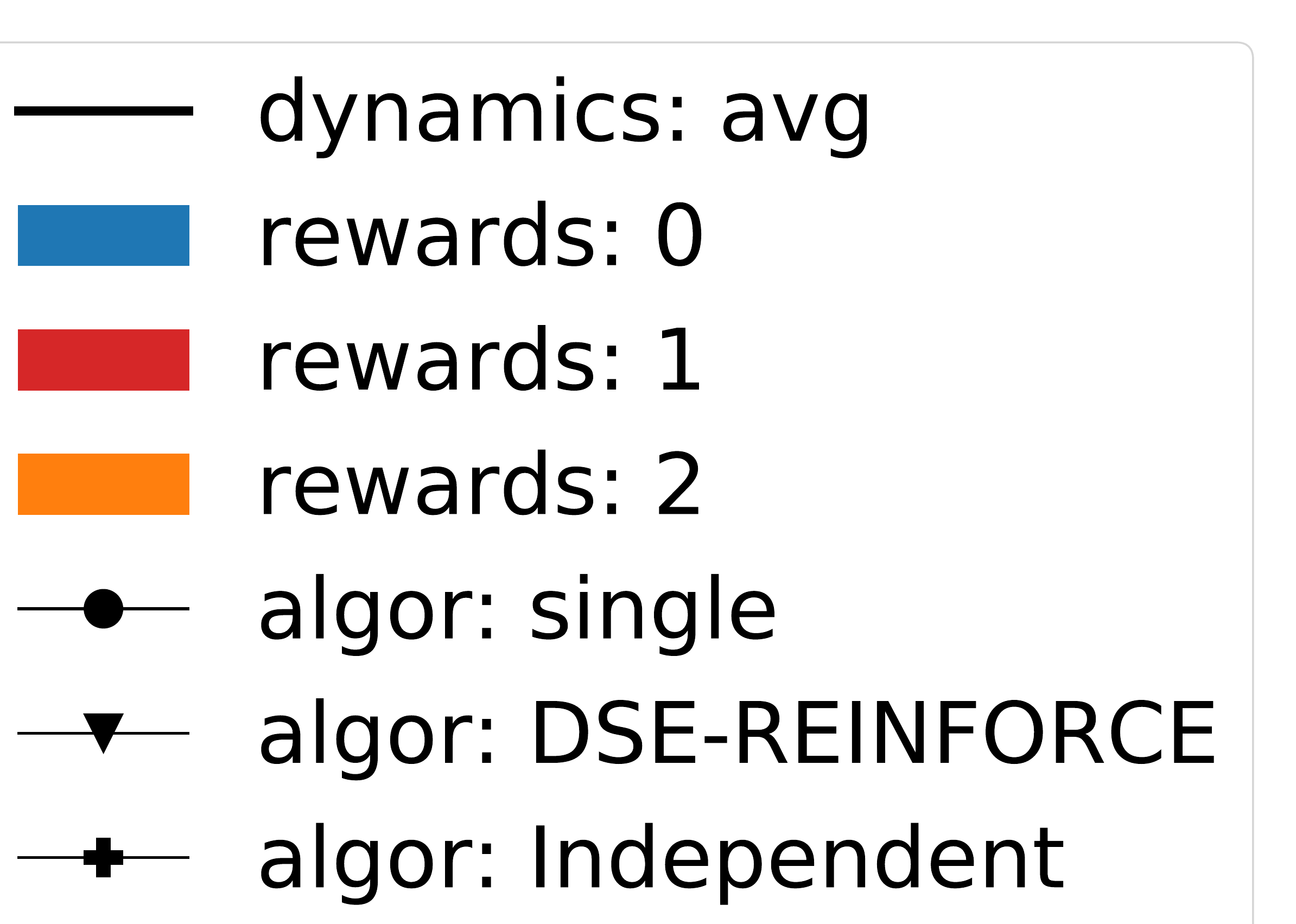}
    \caption{Multi-task Training the Mujoco Reacher-v2 in the full configuration ($3\times3$). (a) shows the episodic rewards obtained by both DSE-SAC, compared against a single embedding algorithm, and training independently. Shaded regions are standard deviations. For each dynamics case, the lengths of the 2 arm components (arm0, arm1) were set as follows: \{(arm0, arm1): $(50\%, 50\%), (33\%, 67\%), (67\%, 33\%)$\}. Note we have averaged over dynamics cases of each algorithm. Appendix C.4 details the original data. }
    \label{fig:reacher_learning_results}
\end{figure}

The original Reacher environment consists of moving the tip of a robotic arm to a random location; its state space included position of the goal. We modified this environment by removing the goal position information and instead, learn an embedding for it. This is considerably a more difficult task. Further, we modified it to vary the dynamics and reward functions by changing the arm lengths and goal position.  We chose $3$ different goal locations and $3$ different arm lengths. 

Figure~\ref{fig:reacher_learning_results} shows the results of our experiments with DSE-SAC on our multi-task Reacher problem. We compared these results with single-task independent learners and single-embedding SAC. As we see, both single-embedding and DSE-SAC have comparable performance and exceed the single-task learner. 
We also carried out experiments comparing DSE-SAC with DSE-REINFORCE (Appendix C.5) where DSE-SAC outperforms DSE-REINFORCE in Reacher by a large margin. Further, we carried out the ``interpolation'' experiments similar to the previous Cartpole experiments (Appendix C.5) showing generalization capabilities in Reacher.

\paragraph{Retraining and generalization of DSE-SAC on multi-task Reacher}
Similar to the Cartpole scenario, we test generalization of DSE-SAC when training with missing tasks on the conditions ($6-3$) and ($4-5$). Testing is performed on unseen test-tasks by initializing the variational distributions by matching index. Learning curves of the multitask policy is shown in Figure \ref{fig:retraining_sac}. In Table 3 of the Appendix, we found that the initial performance in test-tasks is on average better for DSE-SAC compared to the single-embedding SAC algorithm and the performs well in terms of the number of trajectories that it takes for a single-task learner to reach such performance. DSE-SAC obtained $-21.96 \pm 2.20$ episodic reward, while the single embedding obtained $-24.29 \pm 1.50$ episodic reward. It also takes the single-task policy $43.22 \pm 2.16$ number trajectories to reach the performance of DSE-SAC.

\begin{wrapfigure}{R}{0.5\hsize}
\vspace{-5mm}
    \centering
    \includegraphics[width=0.485\hsize]{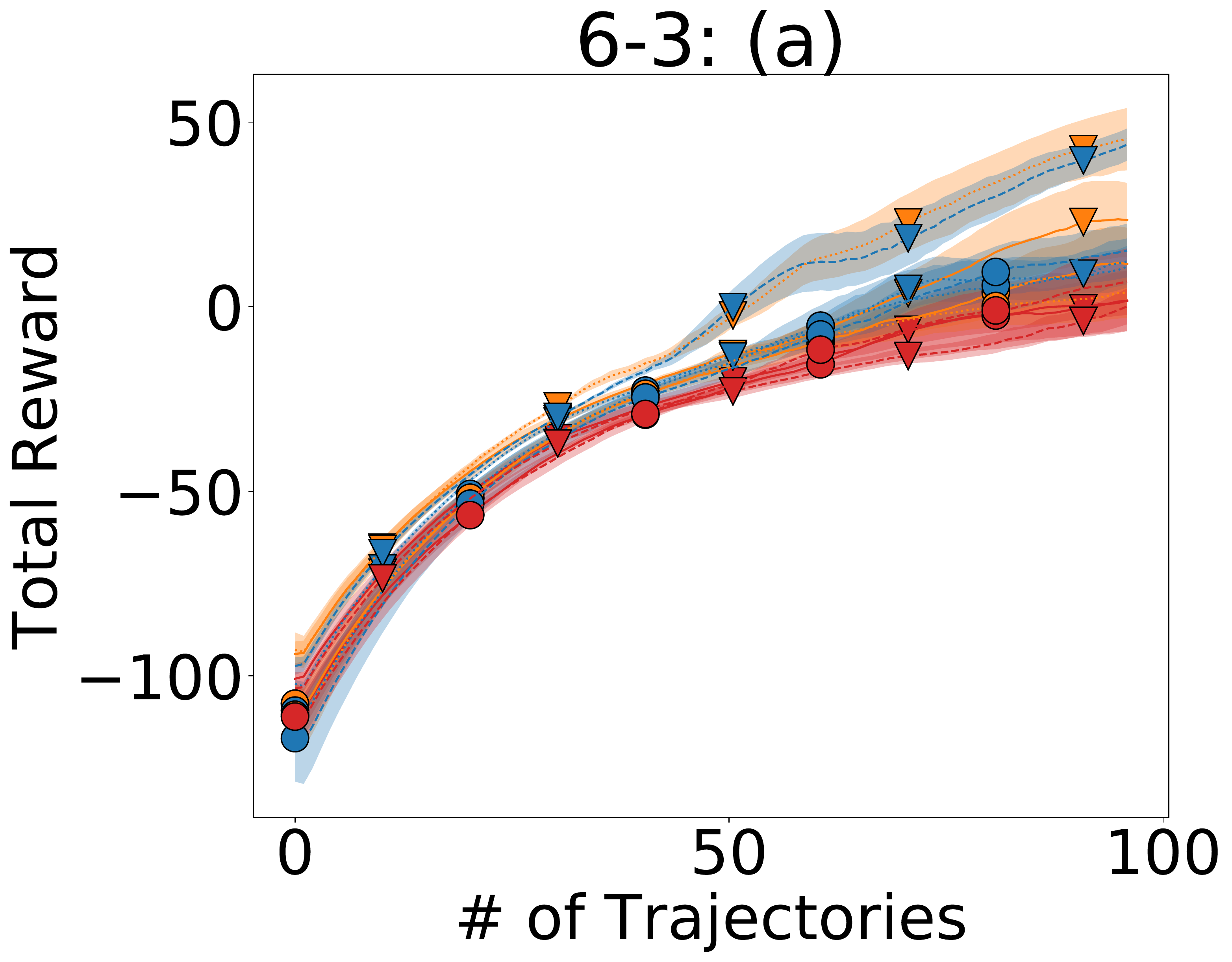}
    \includegraphics[width=0.48\hsize]{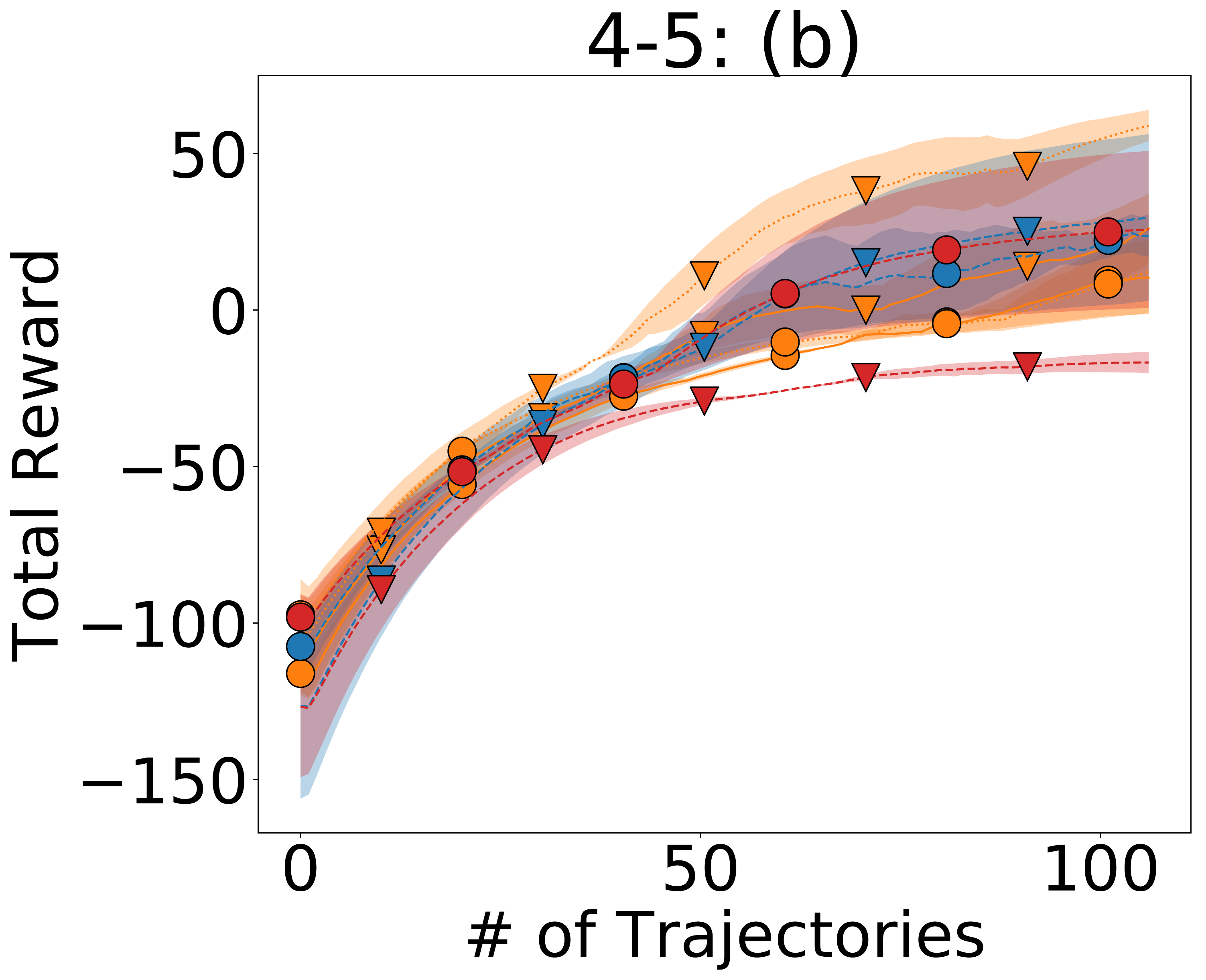}\\
    \includegraphics[width=0.97\hsize]{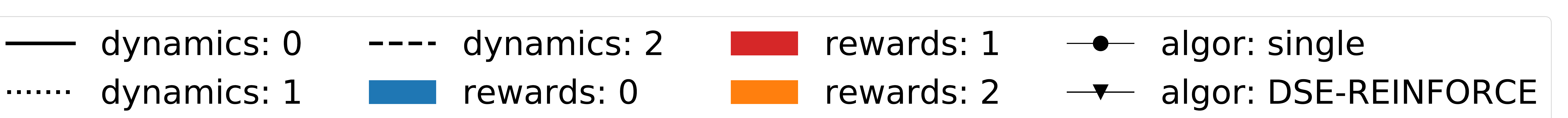}
    \caption{Multi-task training on Reacher-v2 under an incomplete grid of problems ($6-3$ and $4-5$) for DSE-SAC compared against a single-embedding SAC condition with the same hyperparameters.}
    \label{fig:retraining_sac}
    \vspace{-5mm}
\end{wrapfigure}

\paragraph{HRL on Reacher}
We tested the policy trained with DSE-SAC on a HRL scenario. In this case, we continuously moved the goal location in a circle passing by locations that the multi-task policy has never seen. We trained with standard single-task Soft Actor-Critic (H-SAC), a high-level policy that acts on the latent space $g$ and uses the pre-trained multi-task policy $\pi_\theta$ as low-level policy. Such policy is compared against the baseline of SAC trained directly on low-level actions. 
Figure~\ref{fig:hrl_asteroid_reward}(b) shows  the performances of H-SAC and SAC in purple and green respectively. We see that H-SAC can solve the task faster than standard SAC can. 

\section{Conclusions}
We have developed a multi-task framework from a variational inference perspective  that is able to learn latent spaces that generalize to unseen tasks where the dynamics and reward can change independently. In particular, the disentangling allows for better generalization and faster retraining in new tasks. We have shown that the policies learned with our two algorithms DSE-REINFORCE and DSE-SAC, can be used successfully in HRL scenarios.

A promising future direction for DSE-SAC could be to learn Q-functions and Value functions that do not depend on the task-index but directly depend on the latent variables. This would allow for the training of a single Q-function and Value function instead of one per goal and dynamics condition. 

\nocite{langley00}

\bibliography{example_paper}
\bibliographystyle{ieeetr}

\newpage

\appendix
\section{Proofs}
\subsection{Information term weights justification} We can easily weigh each information term with $\frac{1}{\alpha_d}, \frac{1}{\alpha_r}, \frac{1}{\alpha_\pi}$ by assuming 
\begin{equation*}
  \dynqvariat( \dynlat_t |i) := \frac{\bar \dynqvariat (\dynlat_t |i)^{\frac{1}{\alpha_d}}}{\int \bar \dynqvariat(\dynlat_t |i)^{\frac{1}{\alpha_d}} d\dynlat_t} \quad  
  \rewqvariat( \rewlat_t|i) := \frac{\bar \rewqvariat (\rewlat_t|j)^{\frac{1}{\alpha_r}}}{\int \bar \rewqvariat^{\frac{1}{\alpha_r}}(\rewlat_t |j) d\rewlat_t} 
  \quad 
\end{equation*}
and 
\begin{equation*}
\pi(a_t|\dynlat_t, \rewlat_t, s_t):=\frac{ \bar \pi (a_t|\dynlat_t, \rewlat_t, s_t)^{\frac{1}{\alpha_\pi}}}{\int  \bar \pi (a_t|\dynlat_t, \rewlat_t, s_t)^{\frac{1}{\alpha_\pi}} da_t}
\end{equation*}
More concretely, this gives
\begin{align*}
   \max_{\bar \pi, \bar \rewqvariat, \bar \dynqvariat} \mathbb{E}_{q(\tau)} \bigg[ & \sum_{t=0}^{T-1} r_j(s_t, a_t)   - \frac{1}{\alpha_d} \log \frac{\bar \dynqvariat(\dynlat_t |i)}{p(\dynlat_t | i)} \nonumber \\
   & - \frac{1}{\alpha_r} \log \frac{\bar \rewqvariat(\rewlat_t|j)}{p(\rewlat_t | j)} -\frac{1}{ \alpha_\pi} \log \frac{\bar \pi(a_t | \dynlat_t, \rewlat_t, s_t)}{p(a_t | \dynlat_t, \rewlat_t, s_t)} \bigg]
\end{align*}
where we have eliminated the constant terms since they do not affect the solution of the  optimization problem. Therefore, to unclutter the notation, we override the definition of the variational distributions by $\pi := \bar \pi, \dynqvariat := \bar \dynqvariat $ and $\rewqvariat:= \bar \rewqvariat$.

\subsection{Proof Lemma 1}
\begin{lemma}[Index- and state-dependent Value Function Recursion] 
The index-dependent Value function satisfies the following recursive property.  
\begin{align*}\label{eq:value_recursion_ij}
 &V^\pi_{i,j}(s) = \sum_{ \dynlat,\rewlat, a} \dynqvariat(\dynlat|i)  \rewqvariat(\rewlat|j) \pi(a|s,\rewlat,\dynlat) \bigg[ r_j(s,a) \nonumber \\ 
 &- \frac{1}{\alpha_d} \log\frac{\dynqvariat(\dynlat|i)}{p(\dynlat|i)} 
 - \frac{1}{\alpha_r}\log\frac{\rewqvariat(\rewlat|j)}{p(\rewlat|j)}
-\frac{1}{\alpha_\pi}\log\frac{\pi(a|s, \rewlat,\dynlat) }{p(a|s,\rewlat,\dynlat)}  \nonumber \\
& \hspace{2.8cm} + \gamma \sum_{s'} P_i(s'|s,a) V^\pi_{i,j}(s')
 \bigg] 
\end{align*}
\end{lemma}
\begin{proof}
We start by stating again the definition of the value function:
\begin{align*}
 V^\pi_{i,j}(s) := \mathbb{E} \bigg[ &\sum_{t=0}^{\infty} \gamma^t \bigg( r_j(s_t, a_t)   - \frac{1}{\alpha_d} \log \frac{\dynqvariat(\dynlat_t |i)}{p(\dynlat_t | i)}  \nonumber \\
    & - \frac{1}{\alpha_r} \log  \frac{\rewqvariat(\rewlat_t|j)}{p(\rewlat_t | j)} -  \frac{1}{\alpha_\pi} \log \frac{\pi(a_t | \dynlat_t, \rewlat_t, s_t)}{p(a_t | \dynlat_t, \rewlat_t, s_t)}  \bigg) \bigg]  .
\end{align*}
Then we take out the terms with $t=0$ inside the summation and write explicitly the expectation, i.e.,
\begin{align*}
    &  V^\pi_{i,j}(s)  = \sum_{z_0,g_0,a_0} q_\delta (z_0|i) q_\omega (g_0|j) \pi(a_0|s_0, g_0,z_0) \Bigg[ r_j(s_0, a_0) - \frac{1}{\alpha_d}\log \frac{q_\delta(z_0 |i)}{p(z_0)} \\ 
    & -\frac{1}{\alpha_r} \log \frac{q_\omega (g_0 |j)}{p(g_0)} 
    - \frac{1}{\alpha_\pi} \log \frac{\pi(a_0|s_0, z_0, g_0)}{p(a_0|s_0, z_0, g_0)} + \gamma \sum_{s_1}P_i(s_1|s_0 a_0)  \mathbb{E} \bigg[ \sum_{t=0}^{\infty} \gamma^{t} \bigg( r_j(s_{t+1}, a_{t+1}) \\
    &- \frac{1}{\alpha_d} \log \frac{\dynqvariat(\dynlat_{t+1} |i)}{p(\dynlat_{t+1} | i)}  
     - \frac{1}{\alpha_r} \log  \frac{\rewqvariat(\rewlat_{t+1}|j)}{p(\rewlat_{t+1} | j)} -  \frac{1}{\alpha_\pi} \log \frac{\pi(a_{t+1} | \dynlat_{t+1}, \rewlat_{t+1}, s_t)}{p(a_{t+1} | \dynlat_{t+1}, \rewlat_{t+1}, s_{t+1})}  \bigg) \bigg] \Bigg].
\end{align*}
We see now that the inner expectation term is in fact  $V^\pi_{i,j}(s_1)$. Therefore,  changing the sub-indices to $z:=z_0$, $g:= g_0$,  $a:= a_0$, $s:=s_0$ and $s:= s_1$  we proved the lemma. 
\end{proof}

\subsection{Lagrangian for DSE-SAC Optimal Policy}
\begin{definition}[DSE Lagrangian]
Let $p$ be an arbitrary distribution over states. Then the Lagrangian is defined as 
\begin{align*}
 \mathcal L(\pi &, \dynqvariat, \rewqvariat) := \sum_{s,\dynlat, \rewlat} \lambda_0(s,\rewlat,\dynlat) \left( \sum_a \pi(a|s,\dynlat,\rewlat) -1  \right) + \sum_{s} p(s)V^\pi(s) \\
 &+ \sum_j\lambda_1(j)\left( \sum_\rewlat \rewqvariat(\rewlat |j) - 1\right) + \sum_i\lambda_2(i)\left( \sum_\dynlat \dynqvariat(\dynlat |i) - 1\right)
\end{align*}
where $\lambda_0(s,g,z)$, $\lambda_1(j)$, $\lambda_2(i)$ are the Lagrange multipliers ensuring that the policy and variational distributions are properly normalized.
\end{definition}

\subsection{Proof Lemma 2}
\begin{lemma}[Optimal policy with DSE]\label{eq:optimal_policy} Let the variational distributions $\dynqvariat$ and $\rewqvariat$ be fixed. Then, the optimal policy is 
\begin{equation*}
    \pi^\star(a|s,\rewlat,\dynlat) = \frac{p(a|s, \dynlat, \rewlat)\exp\left( \alpha_\pi \bar Q^{\pi^\star}(s,a, \dynlat, \rewlat) \right)}{Z(s,\dynlat,\rewlat)}
\end{equation*}
where $Z(s,\rewlat,\dynlat)$ is the normalizing function and $\bar Q^{\pi^\star}(s,a,\rewlat,\dynlat) := \sum_{i,j} q(i|\dynlat)$ $q(j|\rewlat) Q^{\pi^\star}_{i,j}(s,a)$ with $q(i|\dynlat)=\frac{p(i)\dynqvariat(\dynlat|i) }{\sum_i p(i)\dynqvariat(\dynlat|i)}$ and $q(j|\rewlat)=\frac{p(j)\rewqvariat(\rewlat|j)}{\sum_j p(j)\rewqvariat(\rewlat|j)}$ are the Bayesian posteriors over the task indices.
\end{lemma}

\begin{proof}
We take the functional derivative of the Lagrangian with respect to $\pi(a^\star| s^\star, z^\star, g^\star)$ where the star denotes a particular element resulting in  
\begin{align*}
&  \frac{\partial \mathcal L(\pi , \dynqvariat, \rewqvariat)}{\partial \pi(a^\star| s^\star, z^\star, g^\star)}  = 
  \frac{\partial}{\partial \pi(a^\star| s^\star, z^\star, g^\star)} \sum_{s} p(s)V^\pi(s)\\
 & \hspace{4cm}+ \frac{\partial}{\partial \pi(a^\star| s^\star, z^\star, g^\star)} \sum_{s,\dynlat, \rewlat} \lambda_0(s,\dynlat,\rewlat) \left( \sum_a \pi(a|s,\dynlat,\rewlat) -1  \right) \\
& =  \sum_{i,j} p(i)p(j) q_\delta (z^\star |i) q_\omega(g^\star |j) p(s^\star)\bigg[ r_j(s^\star, a^\star) + \mathbb E_{s'|s^\star, a^\star, i} \left[ V_{i,j}^\pi(s')\right] \\
& -\log \frac{\pi(a^\star| s^\star, \dynlat^\star, \rewlat^\star)}{p(a^\star| s^\star, \dynlat^\star, \rewlat^\star)} + \pi(a^\star| s^\star, \dynlat^\star, \rewlat^\star) (-1) \frac{p(a^\star| s^\star, \dynlat^\star, \rewlat^\star)}{\pi(a^\star| s^\star, \dynlat^\star, \rewlat^\star)} \frac{1}{p(a^\star| s^\star, \dynlat^\star, \rewlat^\star)}  \bigg]\\
& + \lambda_0(s^\star, \dynlat^\star, \rewlat^\star).
\end{align*}
Next, equating the previous equation to zero  and using the following equalities $p(i)q_\delta(z^\star|i) = q(i|z^\star)q(z^\star) $, $q_\omega(g^\star|j) = q(j|g^\star)q(g^\star) $ and $Q_{i,j}(s^\star, a^\star) = r_j(s^\star, a^\star) + \mathbb E_{s'|s^\star, a^\star, i} \left[ V_{i,j}^\pi(s')\right] $ we obtain
\begin{equation}
\sum_{i,j} q(i|z^\star)  q(j|g^\star) Q_{i,j}(s^\star, a^\star) -\log \frac{\pi(a^\star| s^\star, \dynlat^\star, \rewlat^\star)}{p(a^\star| s^\star, \dynlat^\star, \rewlat^\star)} + \frac{-1 \lambda(s^\star, \dynlat^\star,\rewlat^\star) }{p(s^\star) q(\dynlat^\star) q(\rewlat^\star)} = 0
\end{equation}
Re-arranging  the terms we have
\begin{equation}
    \pi(a^\star| s^\star, z^\star, g^\star) =  p(a^\star| s^\star, z^\star, g^\star) e^{\sum_{i,j} q(i|z^\star)  q(j|g^\star) Q_{i,j}(s^\star, a^\star)  + \frac{-1 \lambda(s^\star, z^\star, g^\star) }{p(s^\star) q(z^\star) q(g^\star)} }.
\end{equation}
Finally, using the fact that $\sum_{a}\pi(a| s, z, g) = 1 $ we can obtain the value of the Lagrange multiplier $\lambda_0$. Then, we obtain the desired policy
\begin{equation}
    \pi^\star(a|s,\rewlat,\dynlat) = \frac{p(a|s, \dynlat, \rewlat)\exp\left( \alpha_\pi \bar Q^{\pi^\star}(s,a, \dynlat, \rewlat) \right)}{Z(s,\dynlat,\rewlat)} 
\end{equation}
\end{proof}

\subsection{Proof Lemma 3}
\begin{lemma}[Optimal Embeddings] \label{eq:optimal_variationals}
Assuming a fixed policy $\pi$ the optimal variational distributions  are given by
\begin{equation*}
   \dynqvariat^\star(\dynlat|i) = \frac{p(\dynlat|i)e^{ \alpha_d D^\star_i(\dynlat)  }}{Z(i)} \quad  \rewqvariat^\star(\rewlat|j) = \frac{p(\rewlat|i)e^{ \alpha_r G^\star_j(\rewlat)  }}{Z(j)}
\end{equation*}
where $D$ and $G$ are conceptually similar to Value functions but depend on the latent variables and task indices. More formally, 
\begin{align*}
 &D^\star_i(z)  :=  \mathop{\mathbb E}_{\substack{j\\ g\sim q_{\omega_j}^\star}} \mathop{\mathbb{E}}_{\substack{a\sim \pi \\ s \sim p}} \bigg( Q^{\pi}_{i,j}(s,a) - \frac{1}{\alpha_\pi} \log \frac{\pi(a|s,\dynlat,\rewlat)}{p(a|s,\dynlat,\rewlat)}\bigg) \\ 
 &G^\star_j(g)   :=  \mathop{\mathbb E}_{\substack{i \\ z\sim q_{\delta_i}^\star}} \mathop{\mathbb{E}}_{\substack{a\sim \pi \\ s \sim p}} \bigg( Q^{\pi}_{i,j}(s,a) - \frac{1}{\alpha_\pi} \log \frac{\pi(a|s, \dynlat, \rewlat)}{p(a|s, \dynlat, \rewlat)}\bigg). 
\end{align*}
\end{lemma}
\begin{proof}
Following similar a procedure as before  we obtain
\begin{align*}
  \frac{\partial \mathcal L(\pi , \dynqvariat, \rewqvariat)}{\partial q_\delta(z^\star|i^\star)}  = & \sum_j p(i^\star)p(j) \sum_g q_\omega(g |j) \bigg[ -- \frac{1}{\alpha_d} \log \frac{\dynqvariat(\dynlat^\star |i^\star )}{p(\dynlat^\star | i^\star )}   - \frac{1}{\alpha_r} \log  \frac{\rewqvariat(\rewlat |j )}{p(\rewlat | j )} \\
& + \mathop{\mathbb{E}}_{\substack{a\sim \pi(\cdot | s,g,z^\star) \\ s \sim p}} \bigg( Q^{\pi}_{i^\star,j}(s,a) - \frac{1}{\alpha_\pi} \log \frac{\pi(a|s,g,z^\star)}{p(a|s,g,z^\star)}\bigg)   \bigg] - 1 - \lambda_2(i^\star) = 0
\end{align*}
Again, using the fact that $\sum_\dynlat \dynqvariat(\dynlat |i) = 1$ we can obtain the value of the Lagrange multiplier. Additionally, re-arranging  and simplifying the terms that do not depend on $z$ and $i$ we obtain the optimal variational distribution. The exact same approach can be used to derive the other optimal variational distribution $q_\omega$.
\end{proof}

\section{Details of Algorithms}
\subsection{Gradients}
\paragraph{Gradients of the variational parameters:} 
Given the latter, the gradient of the variational parameters specific for a dynamics context $i$ is given by
\begin{align*}
  \nabla_{\delta_i} \mathcal L (\theta, &\delta, \omega)  =  \frac{1}{MJ}\sum_{m,j=1}^{M,J}\bigg[ \sum_{t=0}^{T-H} \tilde R_t( \tau^{i,j}_m)\nabla_{\delta_i} \log \pi_\theta \left( a_{t,m}^{i,j} | s_{t,m}^{i,j}, z_{\delta_i}( \epsilon_{t,m}^{i,j}), g_{\omega_j}( \varepsilon_{t,m}^{i,j}) \right) \\
  & - \frac{C}{\alpha_d} \nabla_{\delta_i} \textrm{KL}(q_\delta (\cdot | i) || p(\cdot | i) ) \bigg] 
\end{align*}
where $H:= \frac{1}{1-\gamma}$ is used to ensure that only long-enough trajectories are used to compute the gradients, in such a way that they provide a correct estimate of the returns. 
Due to space reasons, we just mention that the gradients for specific goal context $j$ can be easily computed in a similar fashion. 
\paragraph{Gradients of the shared parameters:} 
Since the gradients of the shared parameters cannot flow  through the reward function (which is considered unknown), we use a score function estimator for the policy updates---also known as the log-derivative trick. This gives us the gradients with respect to the shared parameters: 
\begin{align*}
  \nabla_{\theta} \mathcal L (\theta, &\delta, \omega)  =  \frac{1}{MIJ}\sum_{m,i,j=1}^{M,I,J}\bigg[ \sum_{t=0}^{T-H} \tilde R_t( \tau^{i,j}_m)\nabla_{\theta} \log \pi_\theta \left( a_{t,m}^{i,j} | s_{t,m}^{i,j}, z_{\delta_i}( \epsilon_{t,m}^{i,j}), g_{\omega_j}( \varepsilon_{t,m}^{i,j}) \right) \bigg] 
\end{align*}

\subsection{Algorithms}

\begin{algorithm}[H]
\begin{algorithmic}[1]
\STATE  For each $i,j$, initialize parameters $\delta_i$, $\omega_j$ and $\theta$
\FOR{each iteration} 
    \FOR{each $i,j$}
        \STATE Collect trajectories $\{ \tau_m^{i,j}\}_{m=1}^M$
    \ENDFOR
    \FOR{each $i,j$}
        \STATE  $\delta_i \gets \delta_i + \nabla_{\delta_i} \mathcal L (\theta, \delta, \omega)$
        \STATE  $\omega_j \gets \omega_j + \nabla_{\omega_j} \mathcal L (\theta, \delta, \omega)$
    \ENDFOR
    \STATE $\theta \gets \theta + \nabla_{\theta} \mathcal L(\theta, \delta, \omega)$
\ENDFOR
\end{algorithmic}
\caption{DSE-REINFORCE}
\label{alg:DSE-REINFORCE}
\end{algorithm}

\begin{algorithm}[H]
\begin{algorithmic}[1]
\STATE For each $i$, $j$, initialize parameters $\psi_{i,j}$, $\bar \psi_{i,j}$, $\phi_{i,j}$, $\delta_i$, $\omega_j$ and $\theta$.
\FOR{ each iteration } 
    \FOR{ each $i,j$ combination}
        \FOR{ each environment step}
            \STATE $z \sim q_{\delta_i}$, $g \sim q_{\omega_j}$, $a \sim \pi_\theta$, $s_{t+1} \sim P_i$ 
            \STATE $\mathcal D_{i,j} \gets\mathcal D_{i,j}  \cup \{(s_t, a_t, r(s_t, a_t), s_{t+1})\}$
        \ENDFOR
        \FOR{each gradient step}
            \STATE $\phi_{i,j} \gets \phi_{i,j} - \lambda_Q \nabla_{\phi_{i,j}} \mathcal L_{Q_{i,j}} $
            \STATE $\psi_{i,j} \gets \psi_{i,j} - \lambda_V \nabla_{\psi_{i,j}} \mathcal L_{V_{i,j}} $
            \STATE $\delta_i \gets \delta_{i} - \lambda_q \nabla_{\delta_{i}} \mathcal L_{q_{\delta_i}} $
            \STATE $\omega_j \gets \omega_{j} - \lambda_q \nabla_{\omega_{j}} \mathcal L_{q_{\omega_j}} $
        \ENDFOR
        \STATE $ \theta \gets \theta - \lambda_\pi \nabla_{\theta} \mathcal L_{\pi} $
        
    \ENDFOR
    \FOR{ each $i,j$ combination}
        \STATE  $\bar \psi_{i,j} \leftarrow \tau \psi_{i,j} + (1-\tau) \bar\psi_{i,j}$
    \ENDFOR
\ENDFOR
\end{algorithmic}
\caption{DSE-SAC}
\label{alg:DSE-SAC}
\end{algorithm}

\subsection{Hyperparameters}
\subsubsection*{Cartpole}
\begin{table}[h]
\caption{Cartpole hyperparameters}
\label{tab:cartpole_params}
\vskip 0.15in
\begin{center}
\begin{small}
\begin{sc}
\begin{tabular}{c|c|c|}
Parameter & MTRL & HRL \\  
\hline
dim $z$ & 2 & - \\
dim $g$ & 2 & -\\
$\gamma$ & 0.99 & 0.99\\
$H$ & 2 & 100\\
$\alpha_d$ & 50000 & -\\
$\alpha_r$ & 1000 & - \\
$\alpha_\pi$ & $\infty$ & 50\\
$\pi$ learning rate & 0.002 & 0.002\\
$q_\delta$, $q_\omega$ learning rates & 0.002 & - \\
$\beta_{art}$ & 0.02 & 0.02 \\
Max episode length & 300 & 2000\\
Number of tasks & 9 & 1\\
Batch size per task & 4 & 10\\
Extended policy input & Concat and outer product & -
\end{tabular}
\end{sc}
\end{small}
\end{center}
\end{table}

The policies for these problems were composed of neural networks with $H$ hidden units in a single hidden layer. The non-linear component of the hidden layer was the TANH function; the final output passed through a SOFTMAX layer. The input for these networks were the concatenated vector $(g_t, s_t, z_t)$. For the MTRL case, we preprocessed the input by computing the outer product between state vector $s$ and the concatenation of the latent variables $z$ and $g$. We then flattened the outer product and concatenated the original input vector once more.

The hyperparameters for the single-embedding and independent algorithms are the same as the MTRL values from Table~\ref{tab:cartpole_params}. The dimension of the single embedding was equal to the sum of the dimensions of the reward and dynamics latent variables.

\subsubsection*{Reacher}
\begin{table}[h]
\caption{Mujoco Reacher-v2 hyperparameters}
\label{tab:mtrl_params}
\vskip 0.15in
\begin{center}
\begin{small}
\begin{sc}
\begin{tabular}{c|c|c}

Parameter & MTRL & HRL \\
\hline
dim $z$ & 2 & - \\
dim $g$ & 3 & - \\
$\gamma$ & 0.99 & 0.99 \\
\# of hidden units & $50\times2$-100 & 100 \\

$\alpha_\pi$ & 250  & 250\\
$\alpha_{d}$ & 20 & - \\
 $\alpha_{r}$ & 20 & - \\
Proximity Bonus reward  & YES & YES \\
Learning rate policy & 0.003 & 0.003\\
$q_\delta$, $q_\omega$ learning rates & 0.0003 & -\\
Learning rate Q & 0.03 & 0.03 \\
Learning rate V & 0.03 & 0.03 \\
Target smoothing coefficient ($\tau$) & 0.01 &  0.01 \\
Max episode length & 100 & 180 \\
\# of tasks & 8 &  1  \\
\# of dynamics, reward contexts & 2, 4 &  -  \\
Batch size transitions per environment & 128 & 256 \\
Size experience buffer & $3\times10^6$ & $3\times10^5$ \\
Extended policy input & Concatenation & -
\end{tabular}
\end{sc}
\end{small}
\end{center}
\vskip -0.1in
\end{table}

For Mujoco Reacher-v2 environment tasks, continuous actions are sampled from a 2 hidden-layer Gaussian policy network and then squeezed to the bounded interval $[-1, +1]$ using a $\tanh$ function. We concatenate state features $s_t$ and latent variables $z_t$, $g_t$ to form the extended input for the multi-task policy. The state and latents are separately preprocessed by passing through two 50 hidden-unit layers with ReLu activation functions. Both outputs are concatenated to feed the output layer. For the HRL problem we augment the state features by appending the location and distance to the goal to serve as input for a Gaussian network policy over the latent variables with two hidden layers of 100 units each. 

The hyperparameters for the single-embedding and independent algorithms are the same as the MTRL values from Table~\ref{tab:mtrl_params}. The dimension of the single embedding was equal to the sum of the dimensions of the reward and dynamics latent variables.

\section{Additional Experimental Results}
\subsection{Comparison of DSE-REINFORCE against other algorithms}
Here we evaluate the learning of the DSE-REINFORCE policy under the full multi-task ($3 \times 3$) problem space. We compare against a single embedding algorithm, Distral and learning each task independently in Figure \ref{fig:learning_reinforce}.  

\newcommand\widthh{130}
\begin{figure}[t]
    \centering
    \includegraphics[width=\widthh pt]{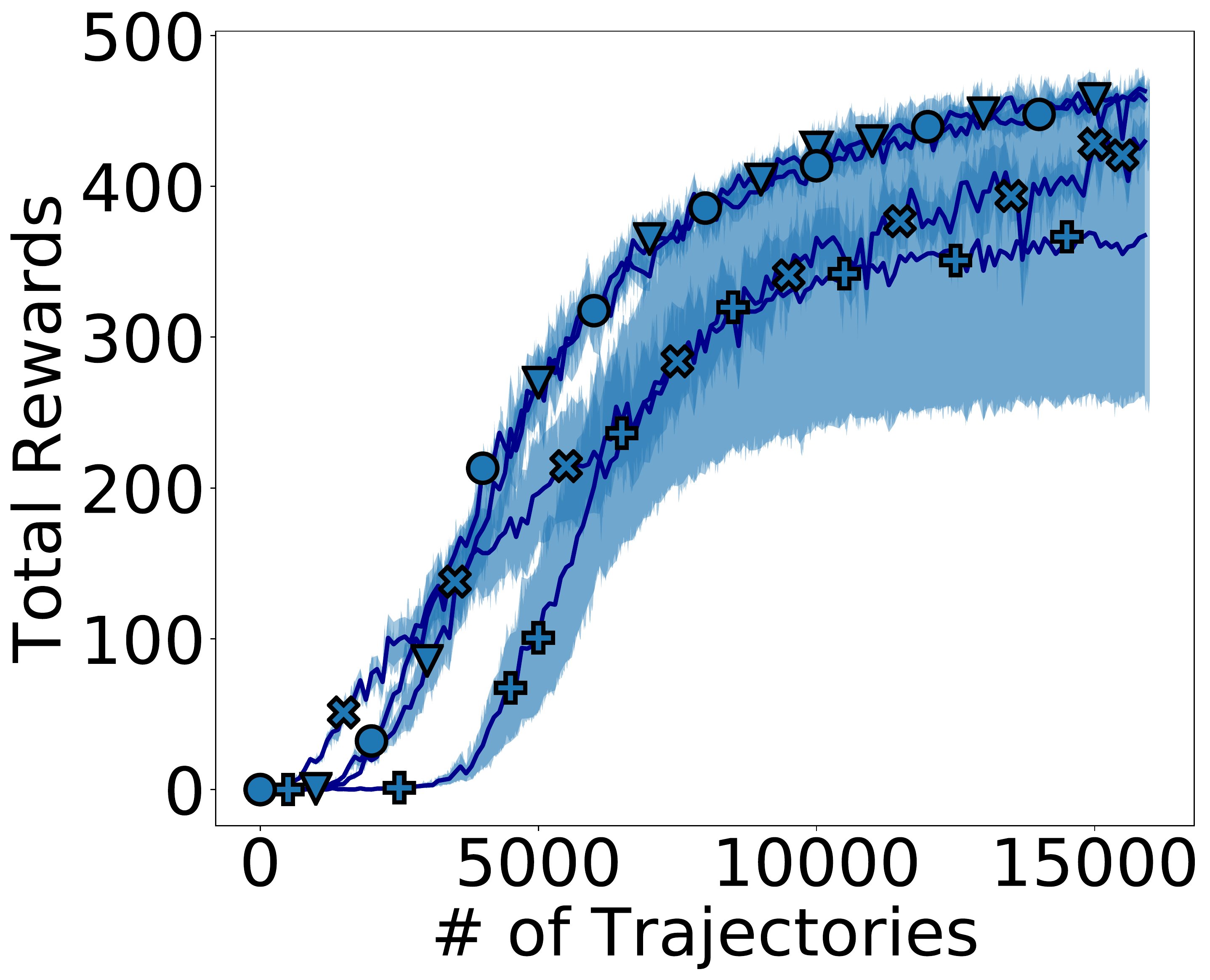}
    \includegraphics[width=\widthh pt]{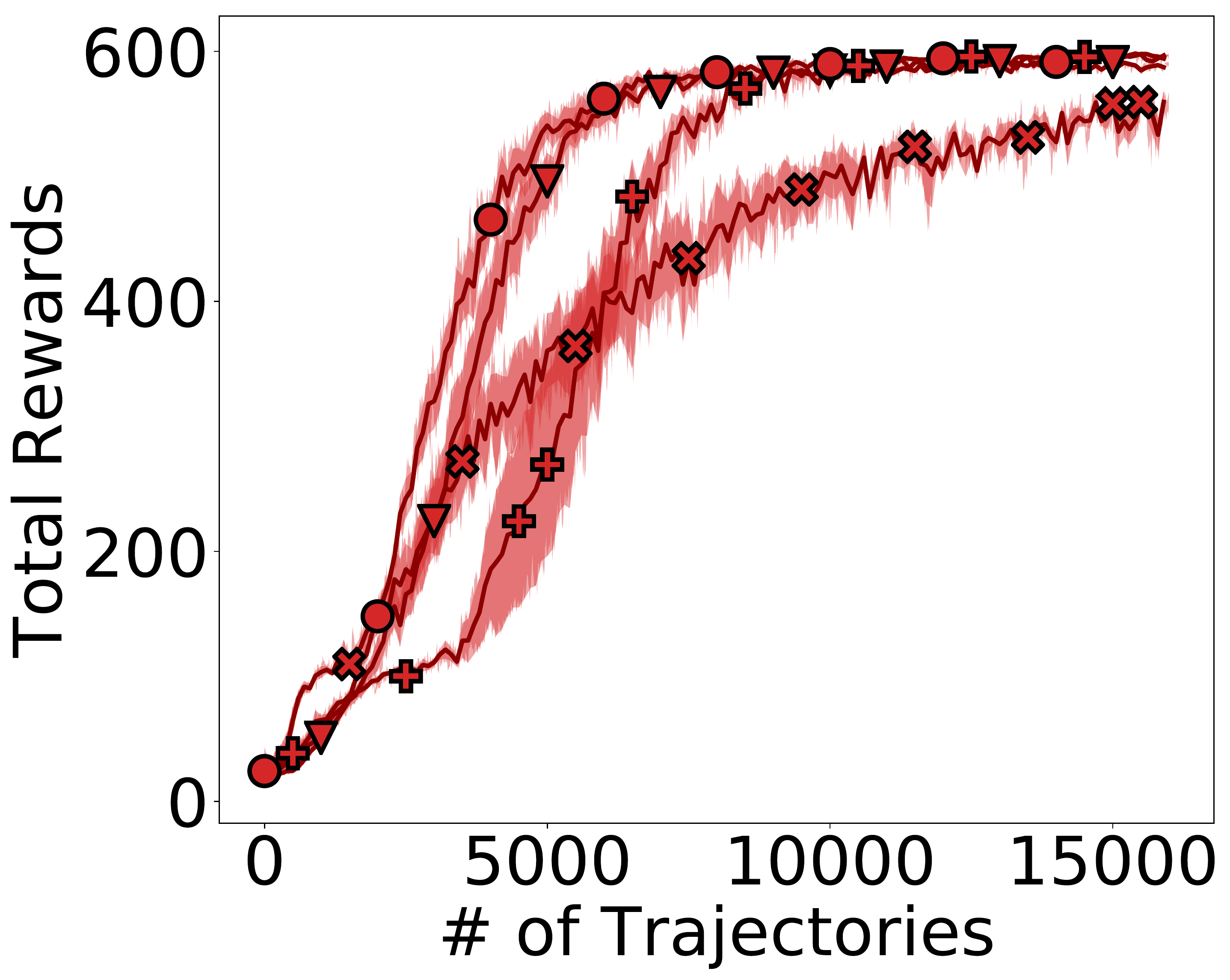}
    \includegraphics[width=\widthh pt]{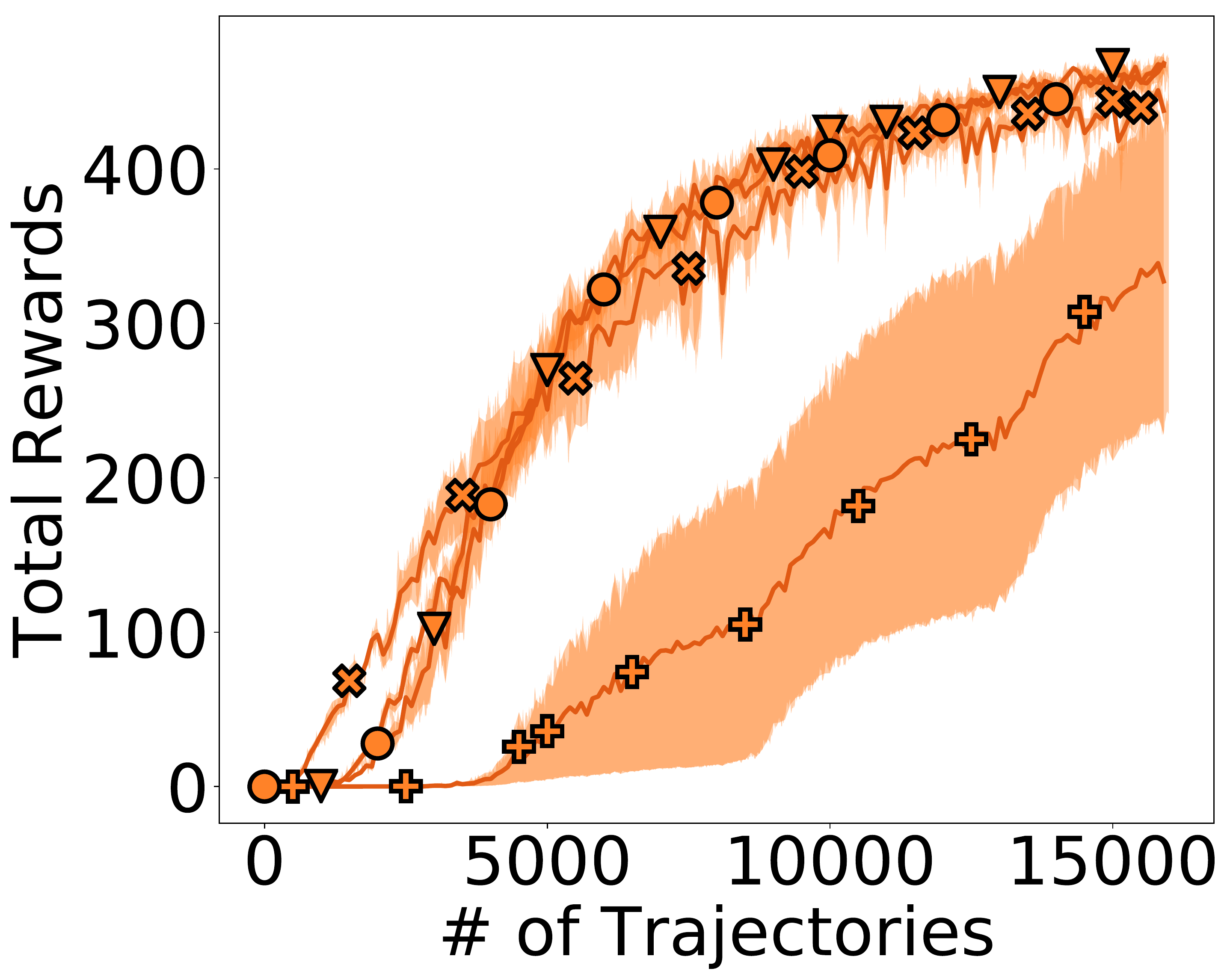} \\
    \includegraphics[width=\widthh pt]{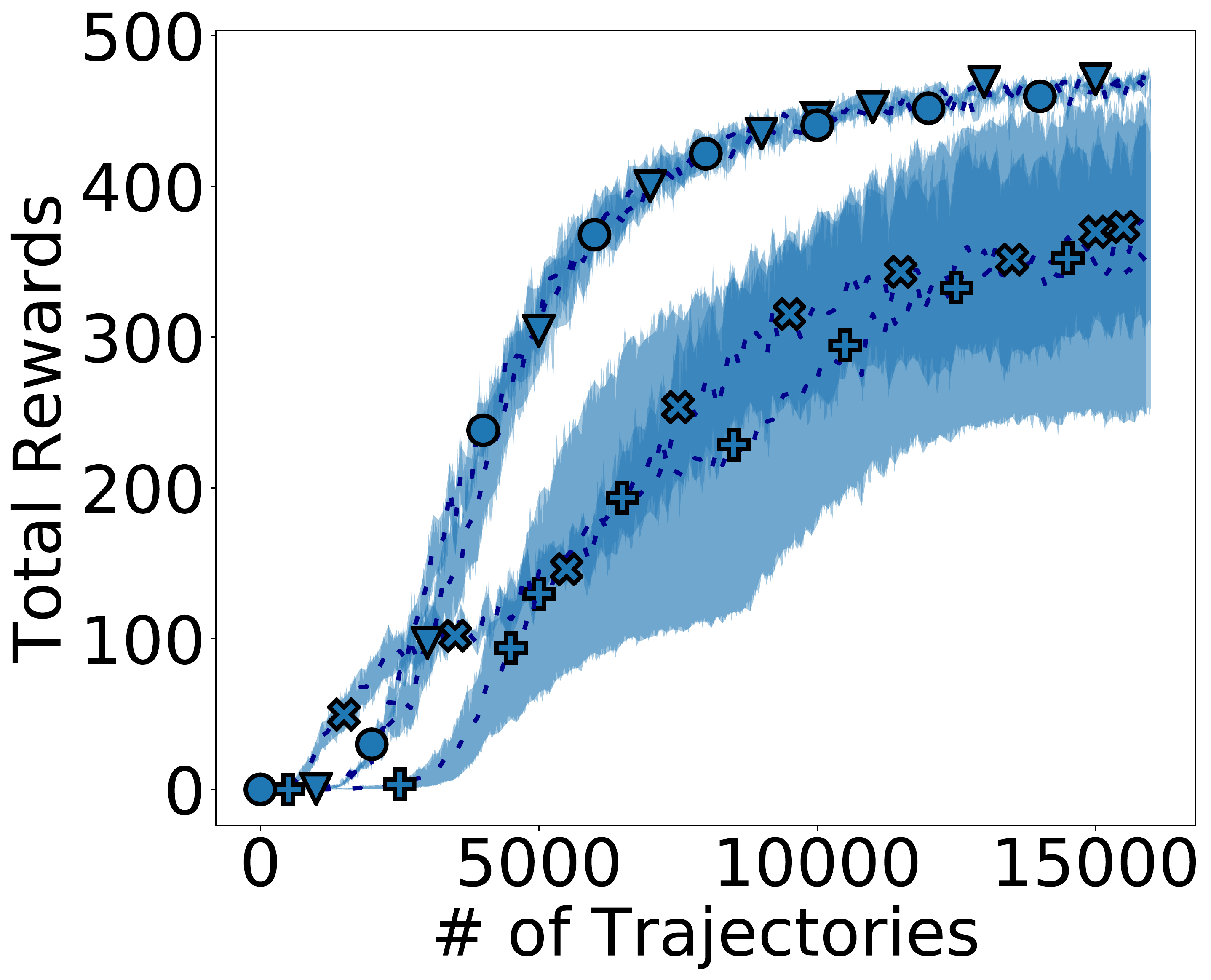}
    \includegraphics[width=\widthh pt]{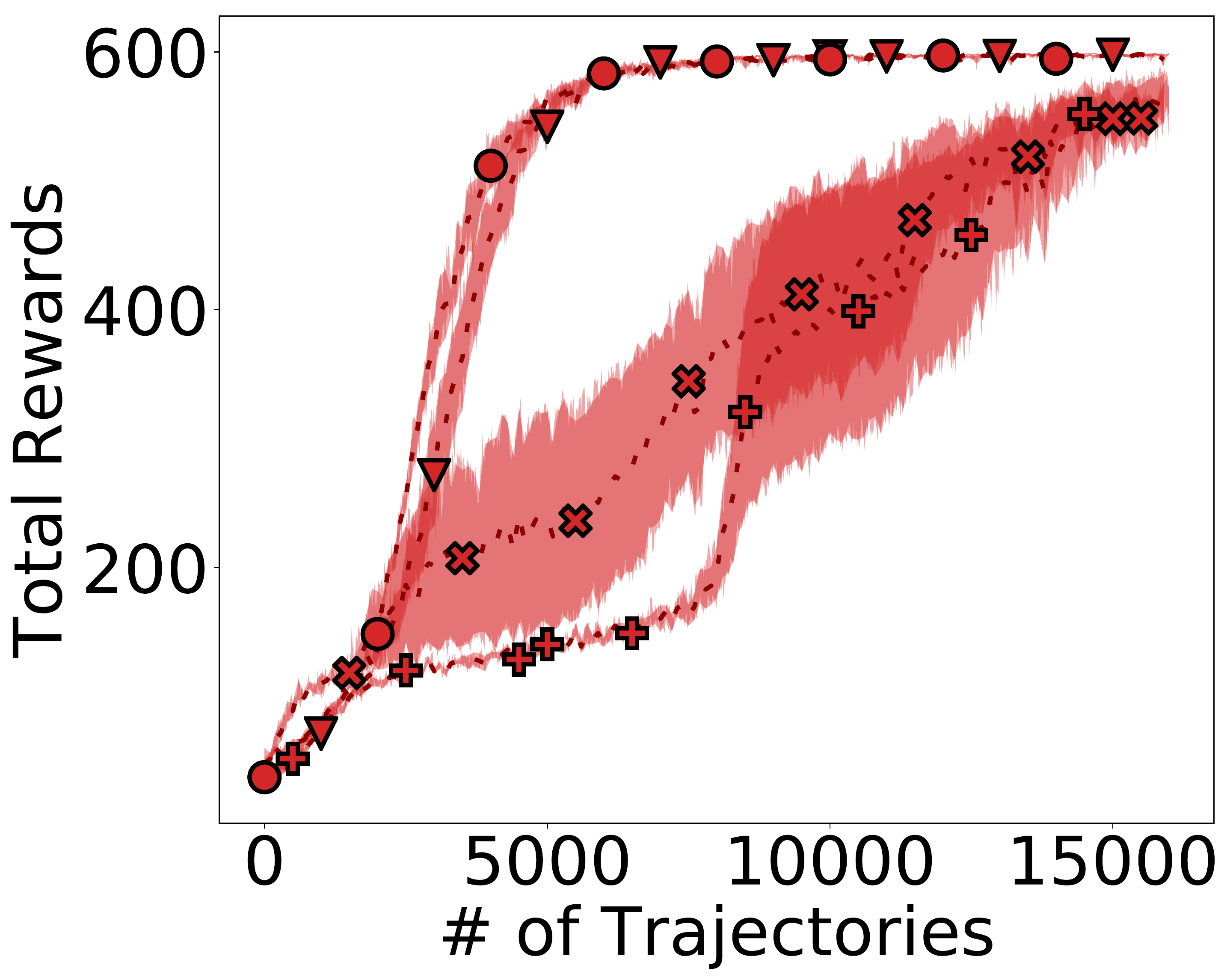}
    \includegraphics[width=\widthh pt]{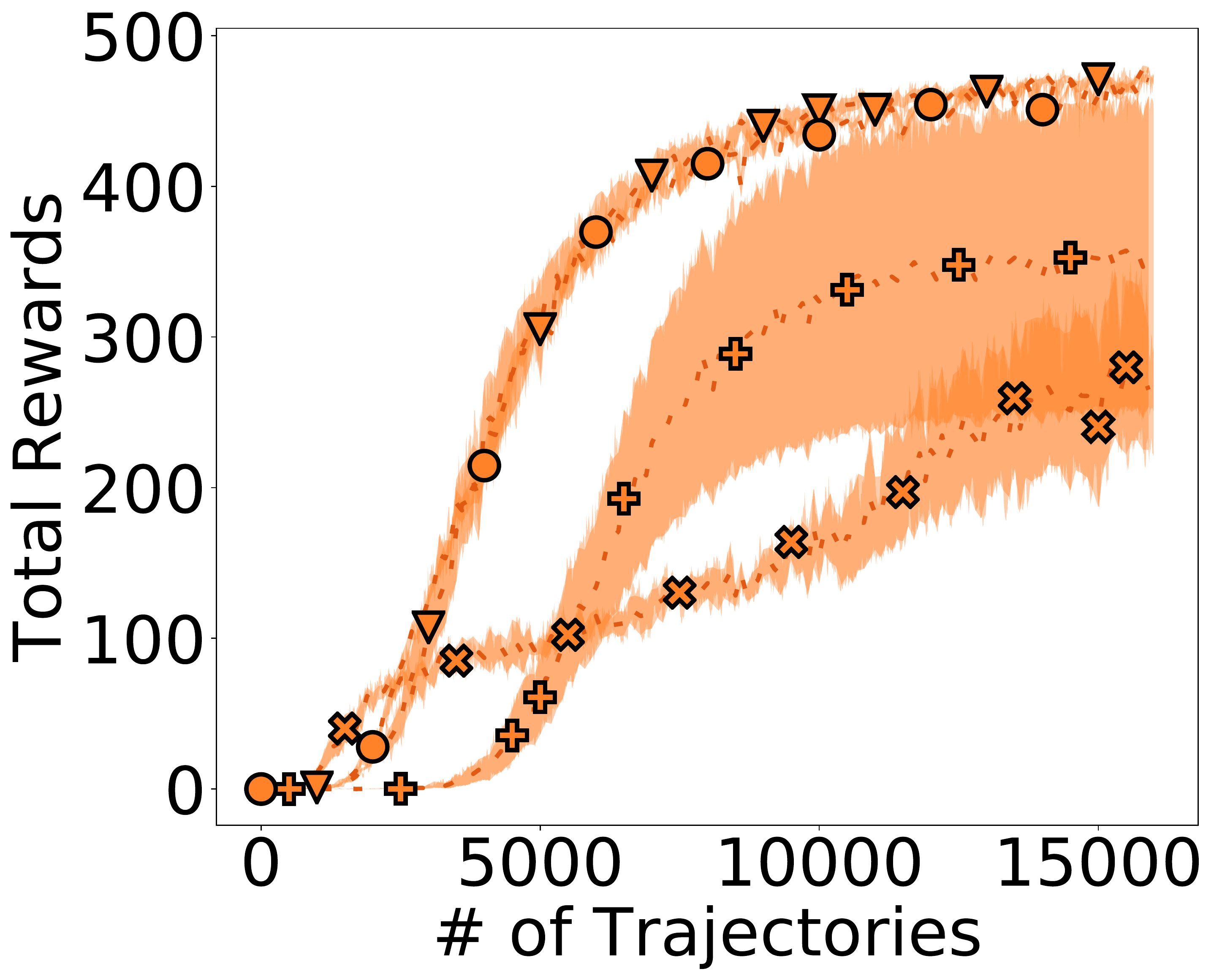} \\
    \includegraphics[width=\widthh pt]{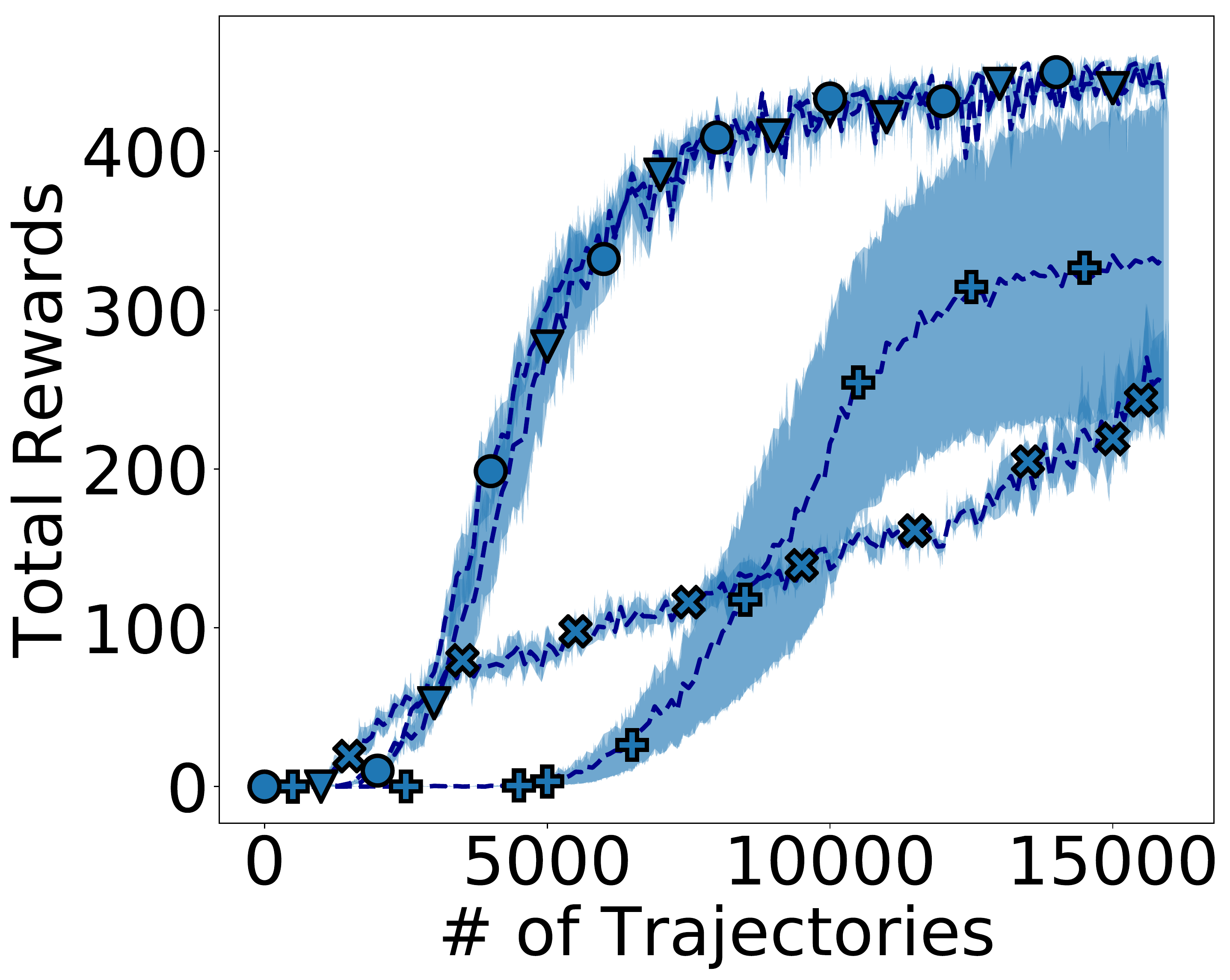}
    \includegraphics[width=\widthh pt]{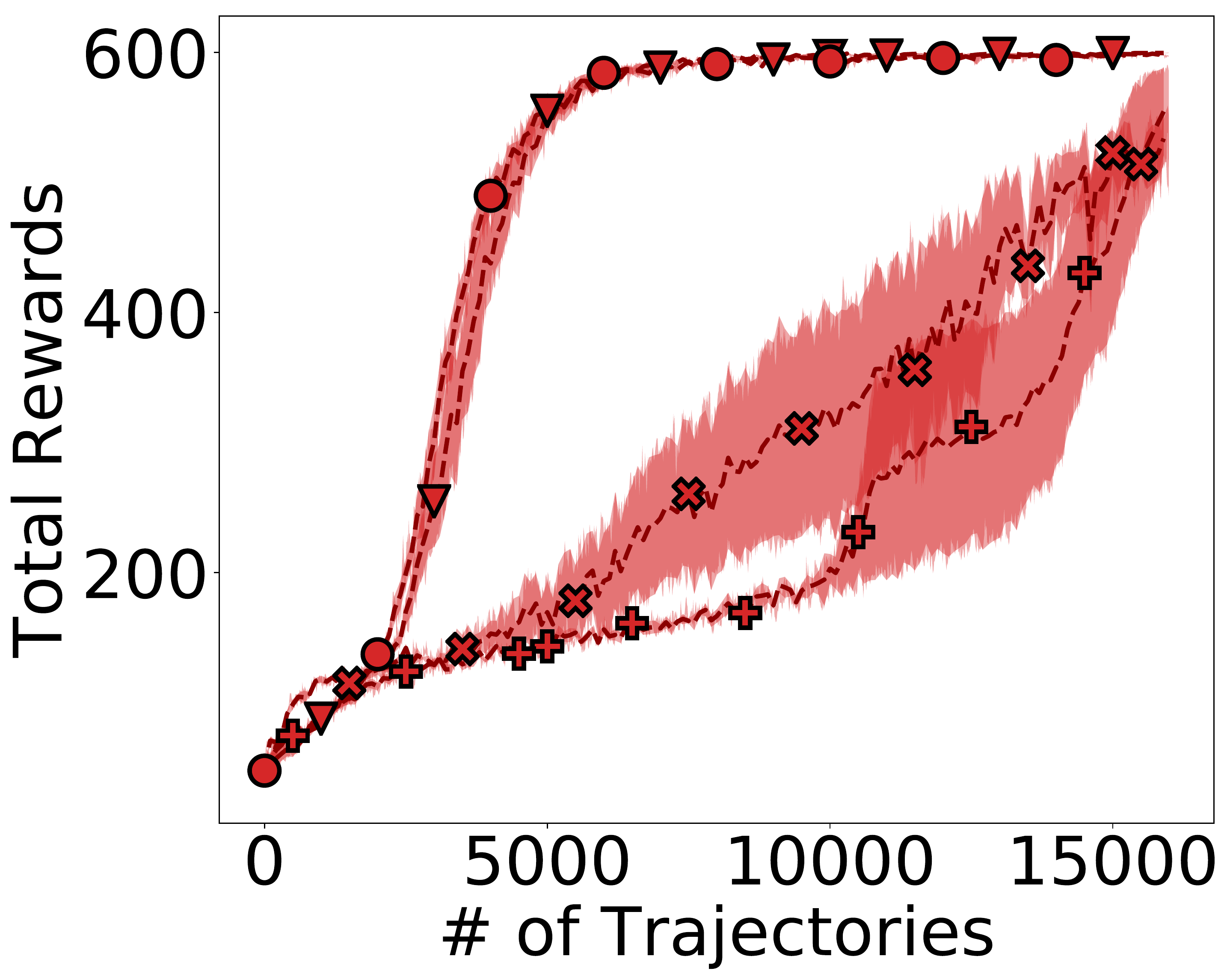}
    \includegraphics[width=\widthh pt]{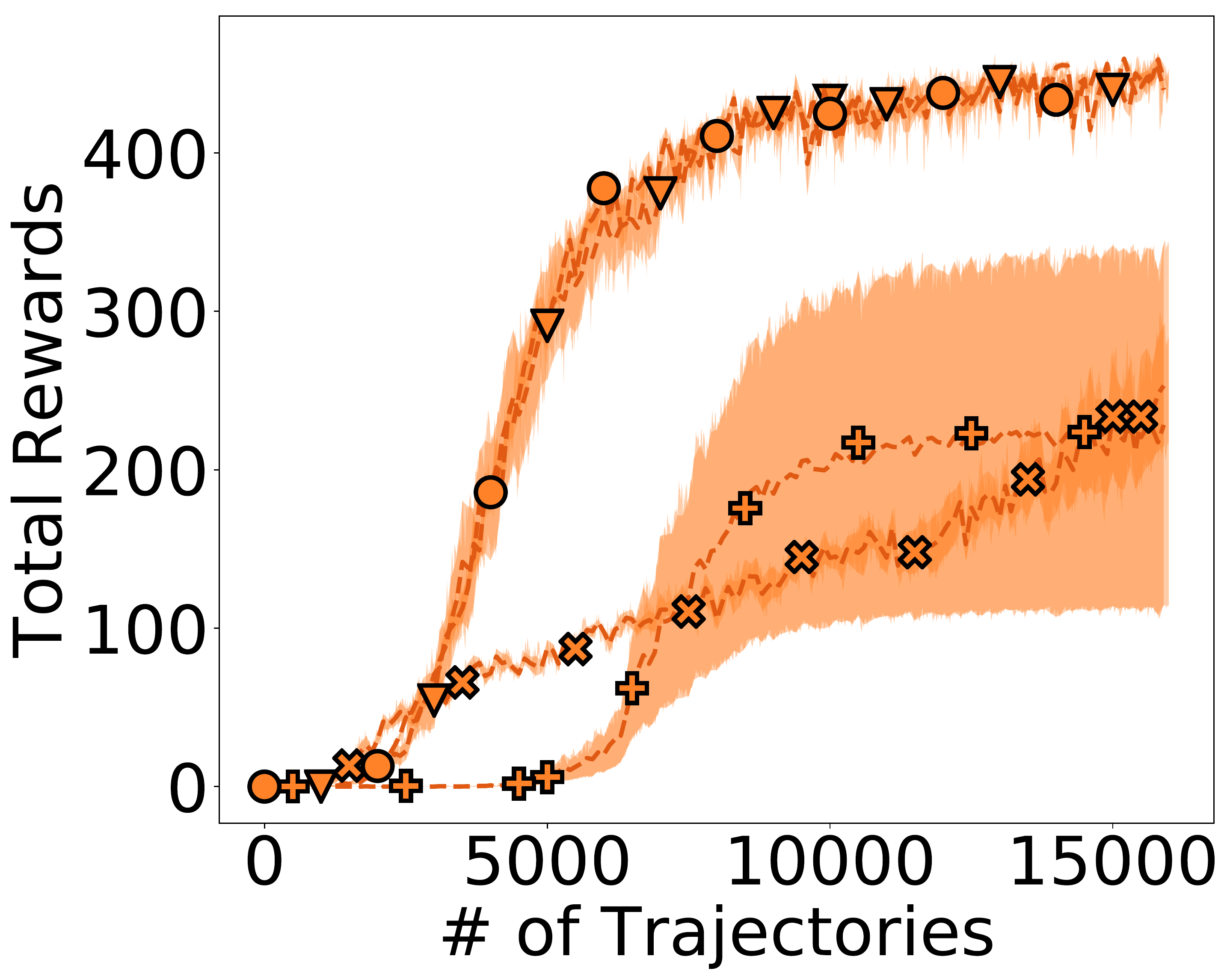} \\
    \includegraphics[width=330 pt]{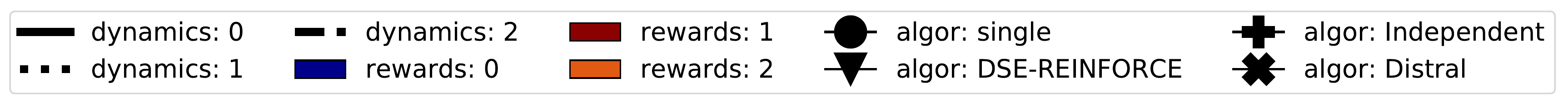}
    \caption{Comparison of DSE-REINFORCE against other algorithms. Here the task configurations are the same as in the manuscript; they are repeated here for completeness. dynamics $ = \{0, 1, 2\}$ are $\mathrm{masscart}=\{0.2, 1.0, 2.0\}$ and $\mathrm{reward}=\{0, 1, 2\}$ are $\mathrm{x_{goal}}=\{-1.0, 0.0, 1.0\}$. The columns represent the reward conditions, and the rows the dynamics conditions, all in ascending order. }
    \label{fig:learning_reinforce}
\end{figure}
From these figures, we see that for all cases, DSE-REINFORCE learns faster, or at least at the same pace as the other algorithms.  The hyperparameters for Distral were chosen as $\beta=10$, $\alpha = 0.5$ and all learning rates were $10^{-4}$, whereas the network architecture consisted of two layers with $50$ hidden neurons and ReLU non-linearity.

\newpage
\subsection{Description of AsteroidCartPole}
In a novel cartpole problem (AsteroidCarpole), the high-level policy  must balance the pole while avoiding collisions---that trigger a terminal event---between the cart and (1 or 2) asteroids that fall from the top. The reward was set to $+1$ for each time step it survives while balancing the pole. Asteroids start at a random $x$ coordinate (Type 1), or directly above the cart (Type 2) and are instantaneously reset as they reach the bottom of the screen. The state space of the original cartpole was augmented to include the $x$ and $y$ coordinates of the asteroids. 

 Therefore, the latent variable $z$ that the low-level policy $\pi(a|s,z,g)$ needs to execute low-level actions was also fixed to the mean of the variational distribution for the dynamics condition with $m=1$. In contrast, the latent variables $g$ in which the high-level policy acts could be chosen from a discrete set of five values. Three of them were the means of the learned variational distributions for the goal-contexts and the other two were $\mu = \{[-0.1, -1.30], [0.35, 0.65]\}$.

\subsection{HRL 2-Asteroid AsteroidCartpole}
We also trained a high-level policy on the 2-Asteroid problem. The evolution of the rewards for 3 examples of this are shown in Figure \ref{fig:2asteroid}, along with those of 3 runs with learning a low level policy using the base REINFORCE algorithm.

\begin{figure}[h!]
    \centering
    \includegraphics[width=150pt]{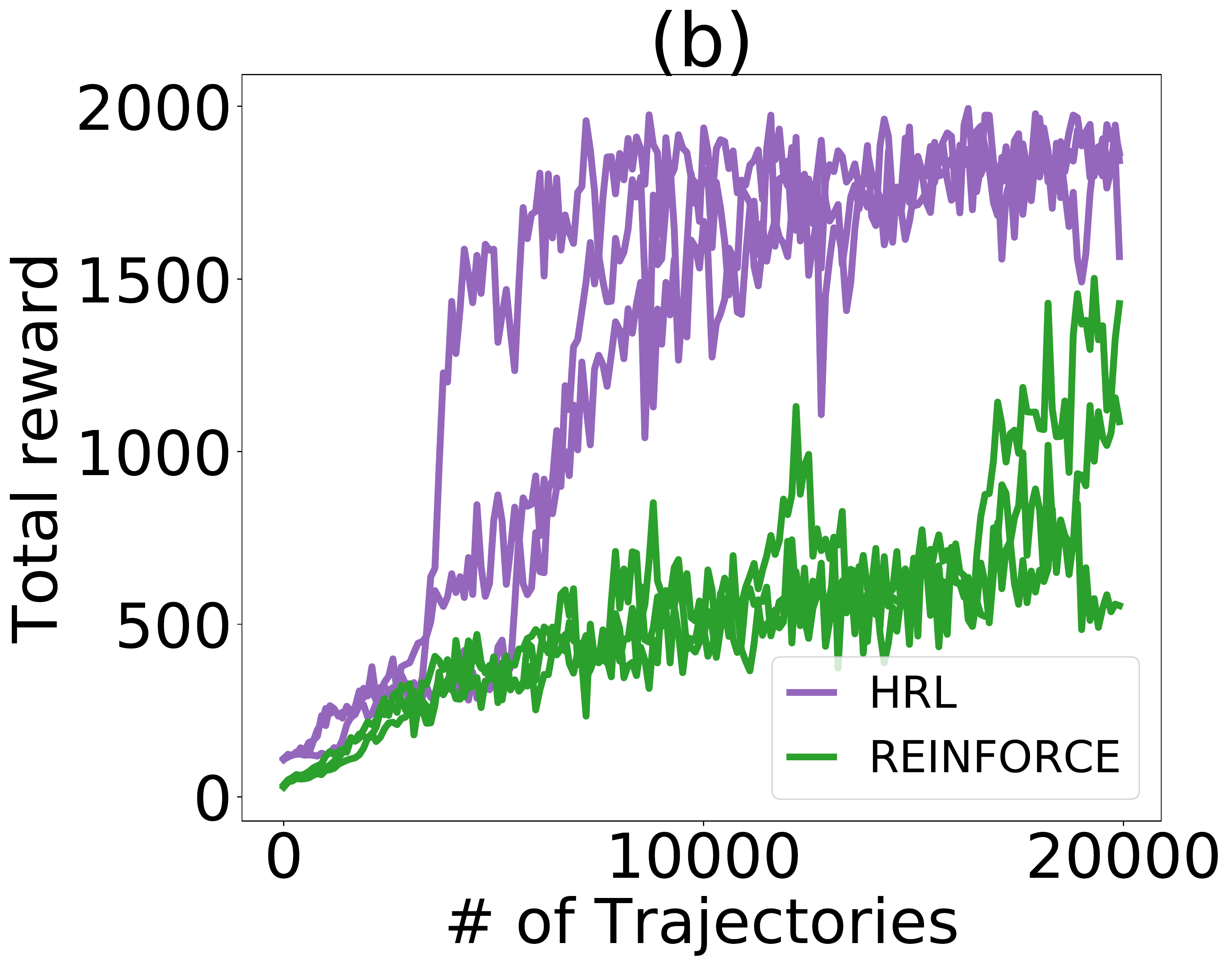}
    \caption{Evolution of the total reward for 2-Asteroid HRL problem. }
    \label{fig:2asteroid}
\end{figure}
\newpage

\subsection{Comparison of DSE-SAC against other algorithms}
As with DSE-REINFORCE, we compared the solution of Reacher-v2 using DSE-SAC against a single-embedding (no disentangling) algorithm, and to training each task independently. The full 3x3 grid is used. The full learning plots for each task is shown in Figure \ref{fig:learning_sac}. 
\begin{figure}[t]
    \centering
    \includegraphics[width=\widthh pt]{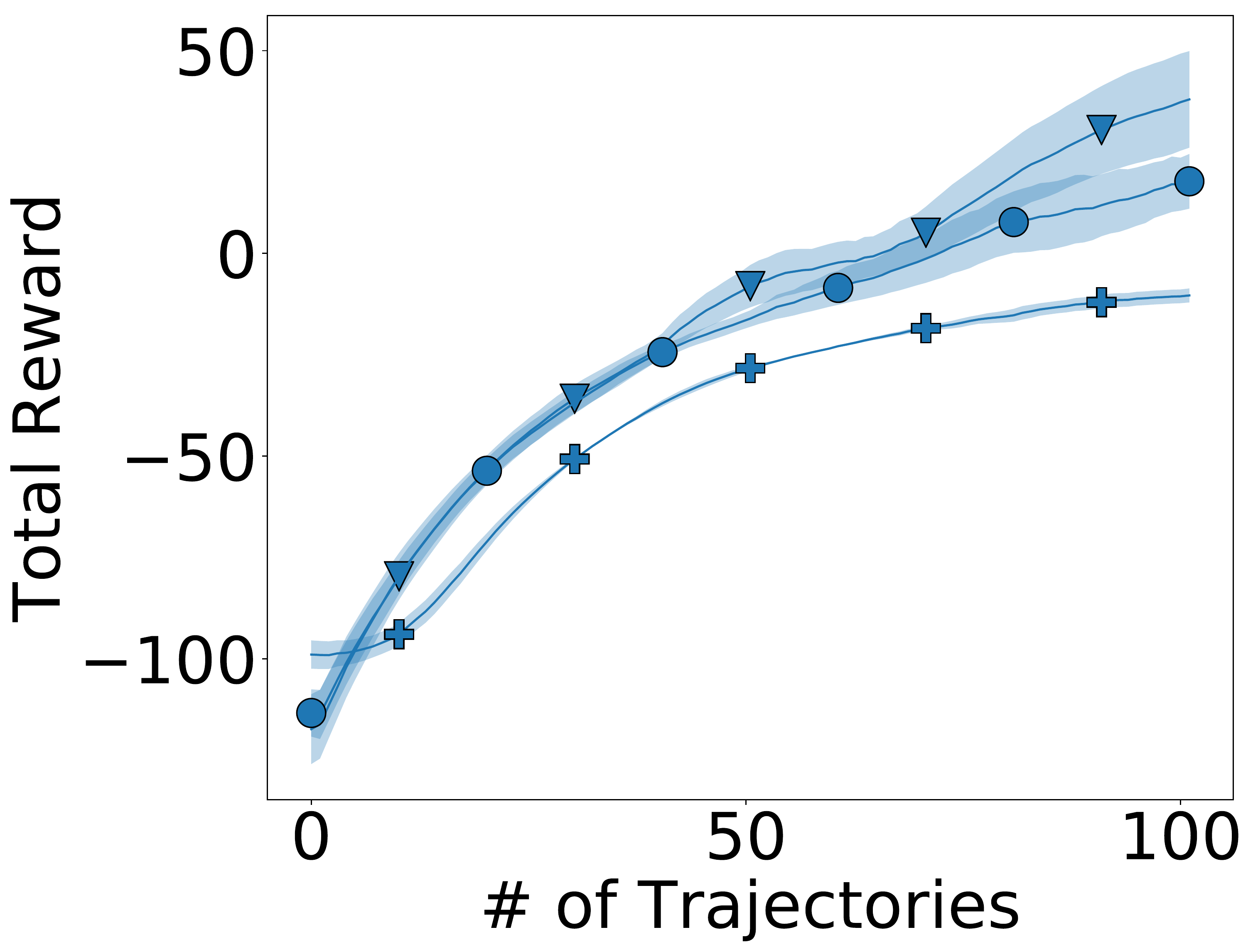}
    \includegraphics[width=\widthh pt]{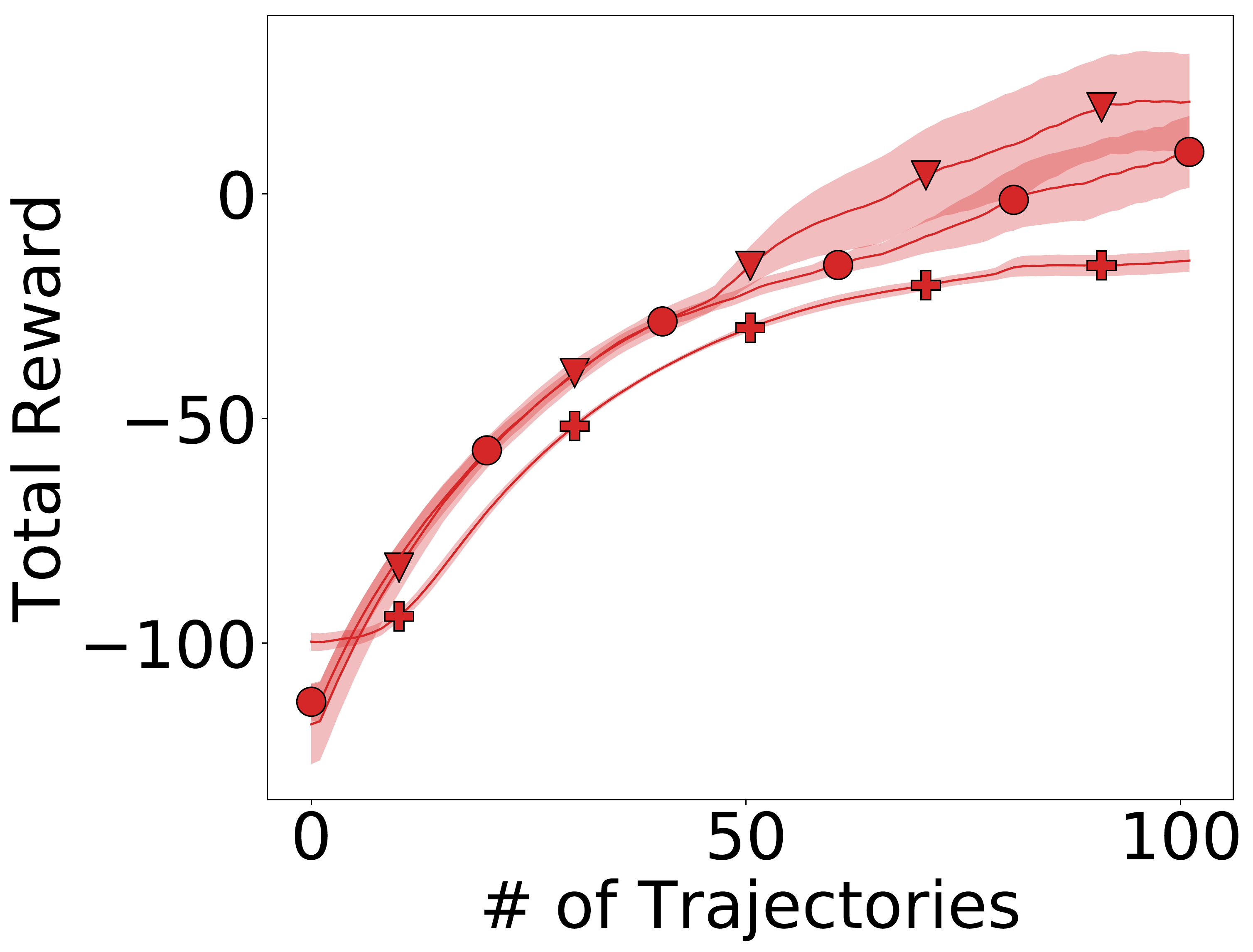}
    \includegraphics[width=\widthh pt]{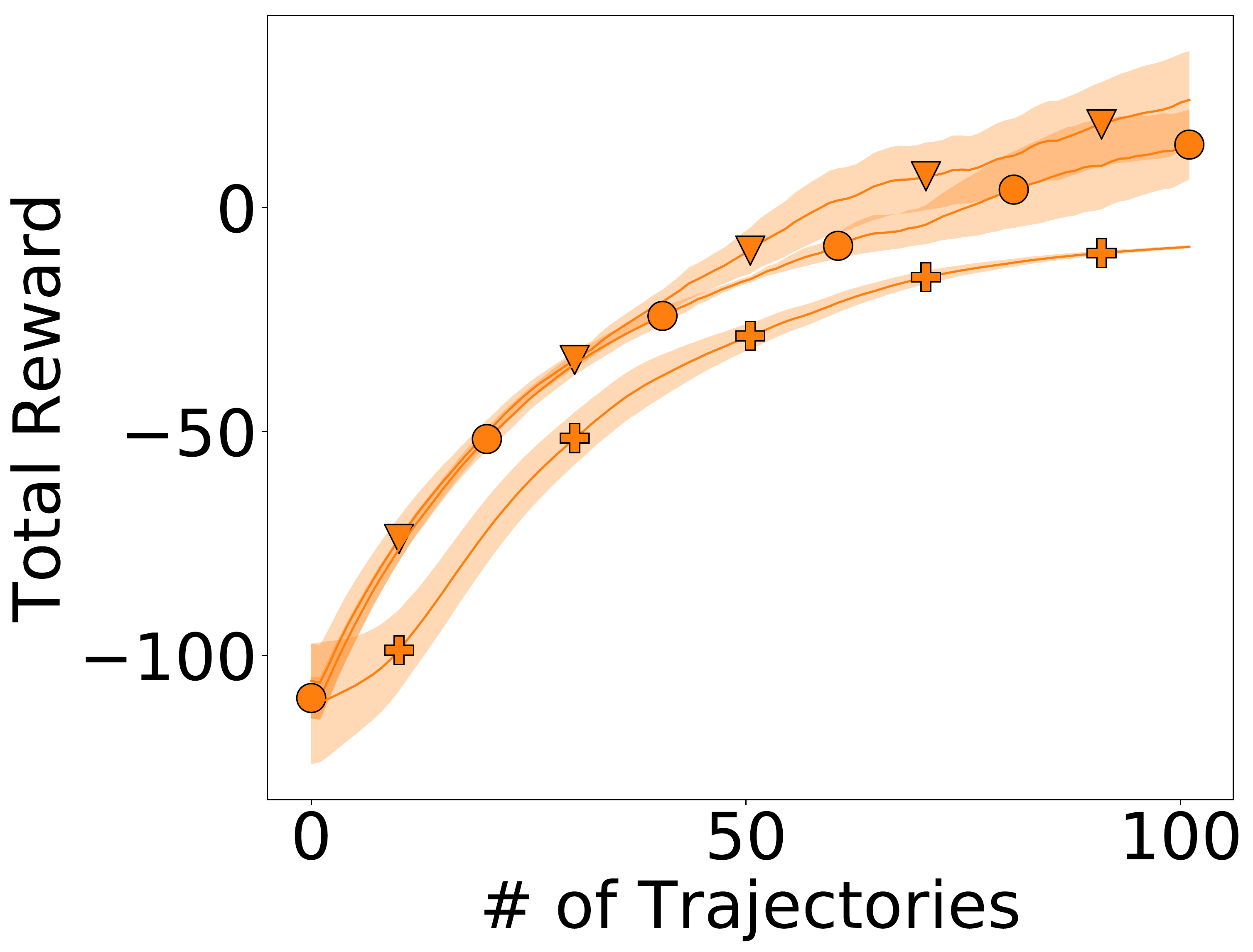} \\
    \includegraphics[width=\widthh pt]{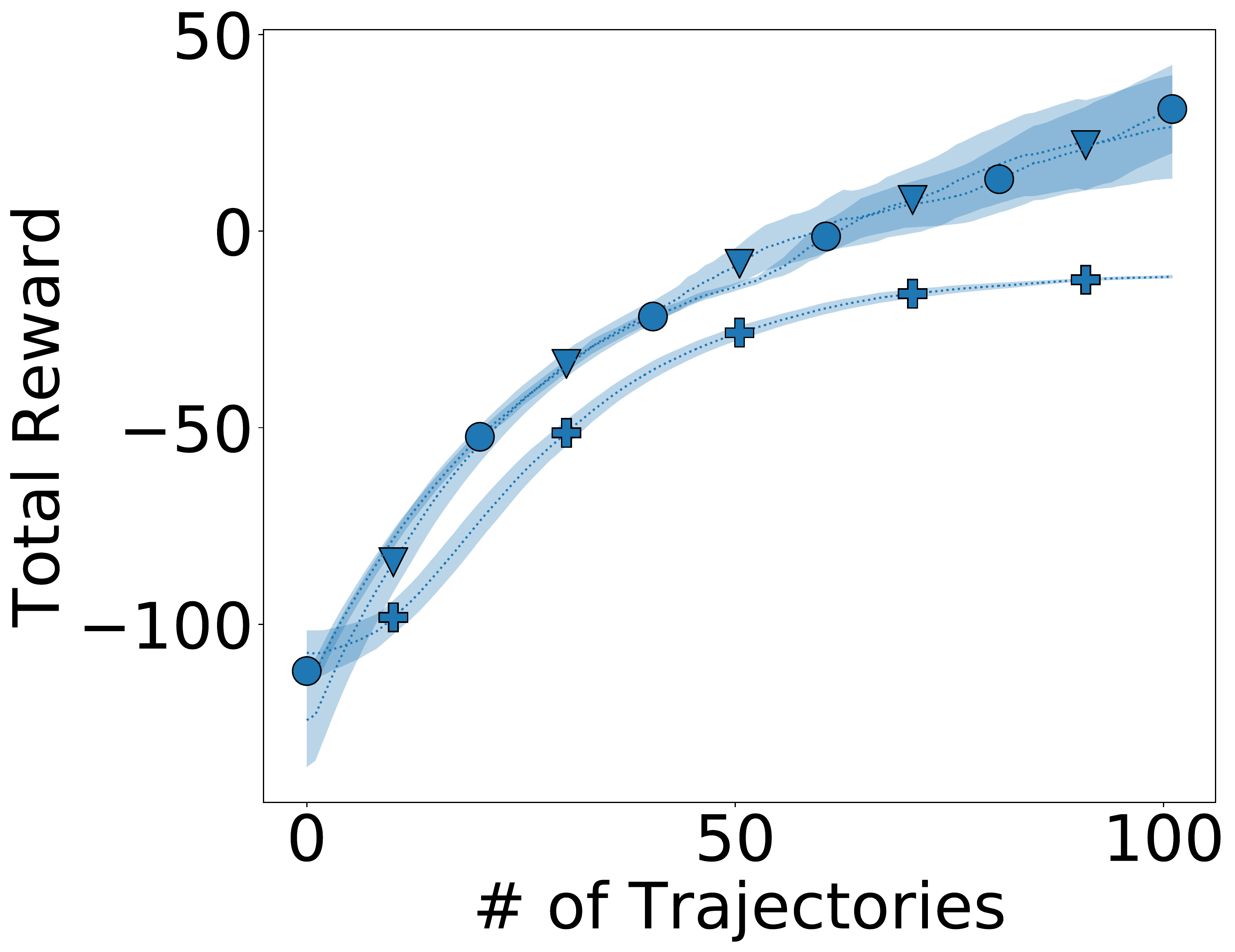}
    \includegraphics[width=\widthh pt]{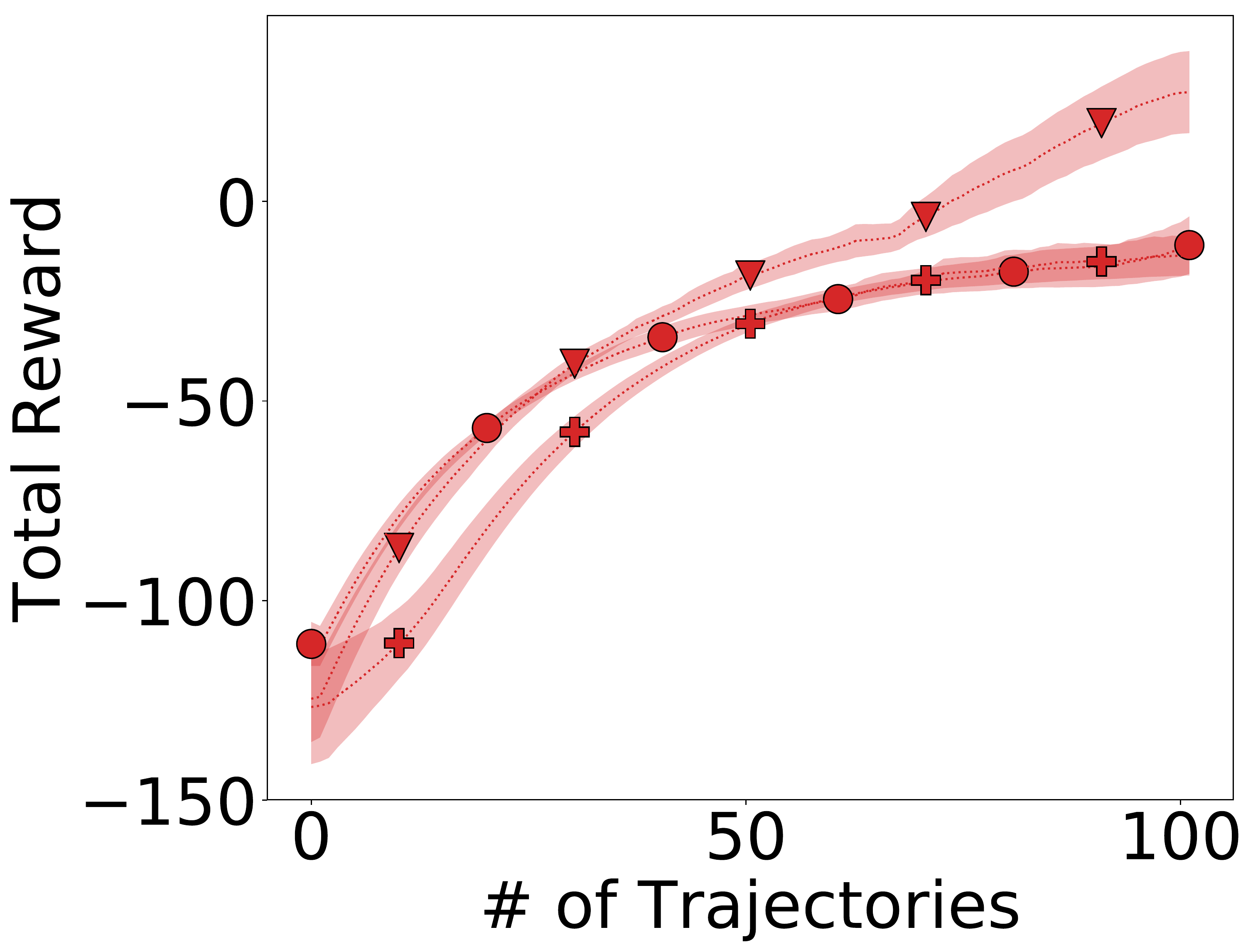}
    \includegraphics[width=\widthh pt]{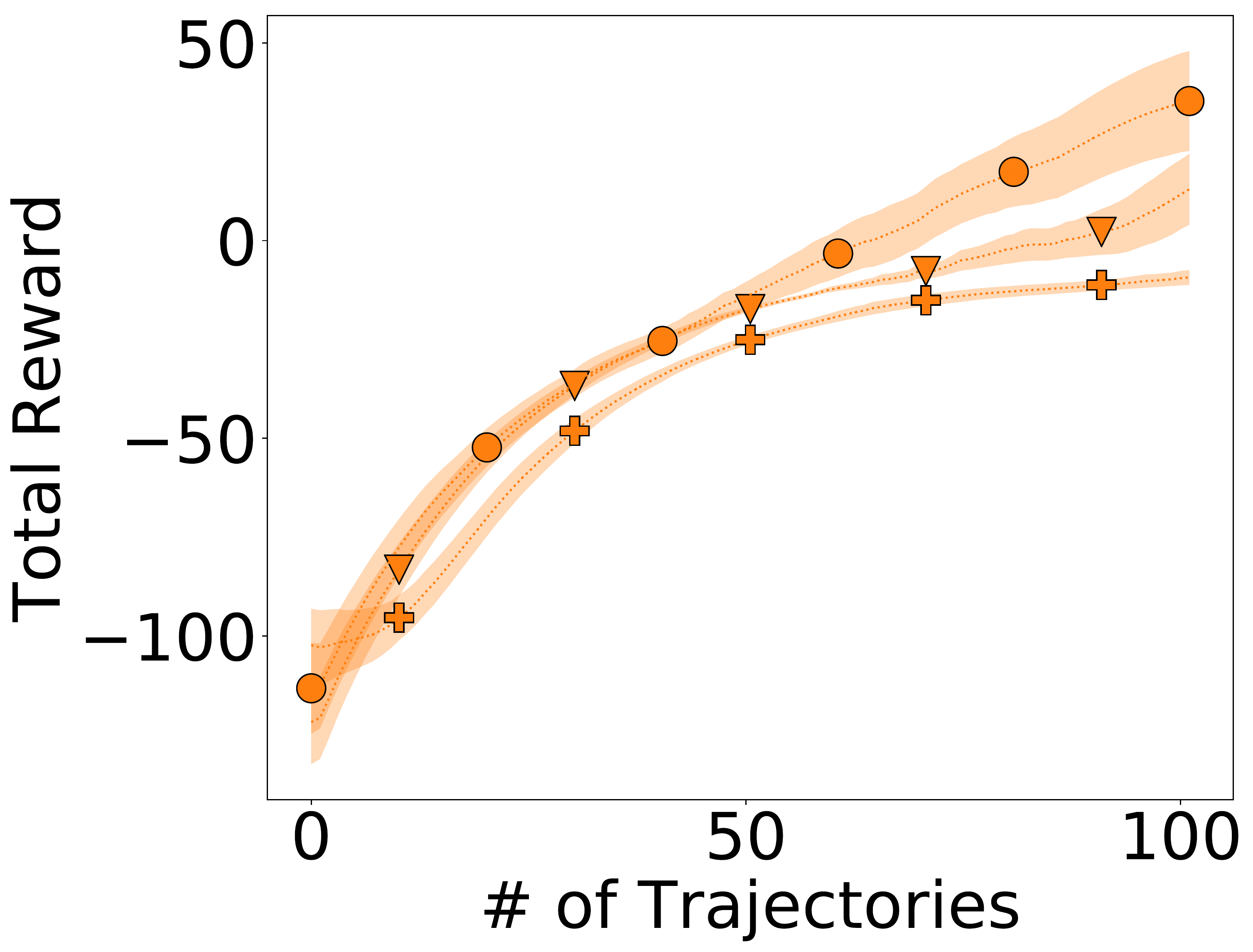} \\
    \includegraphics[width=\widthh pt]{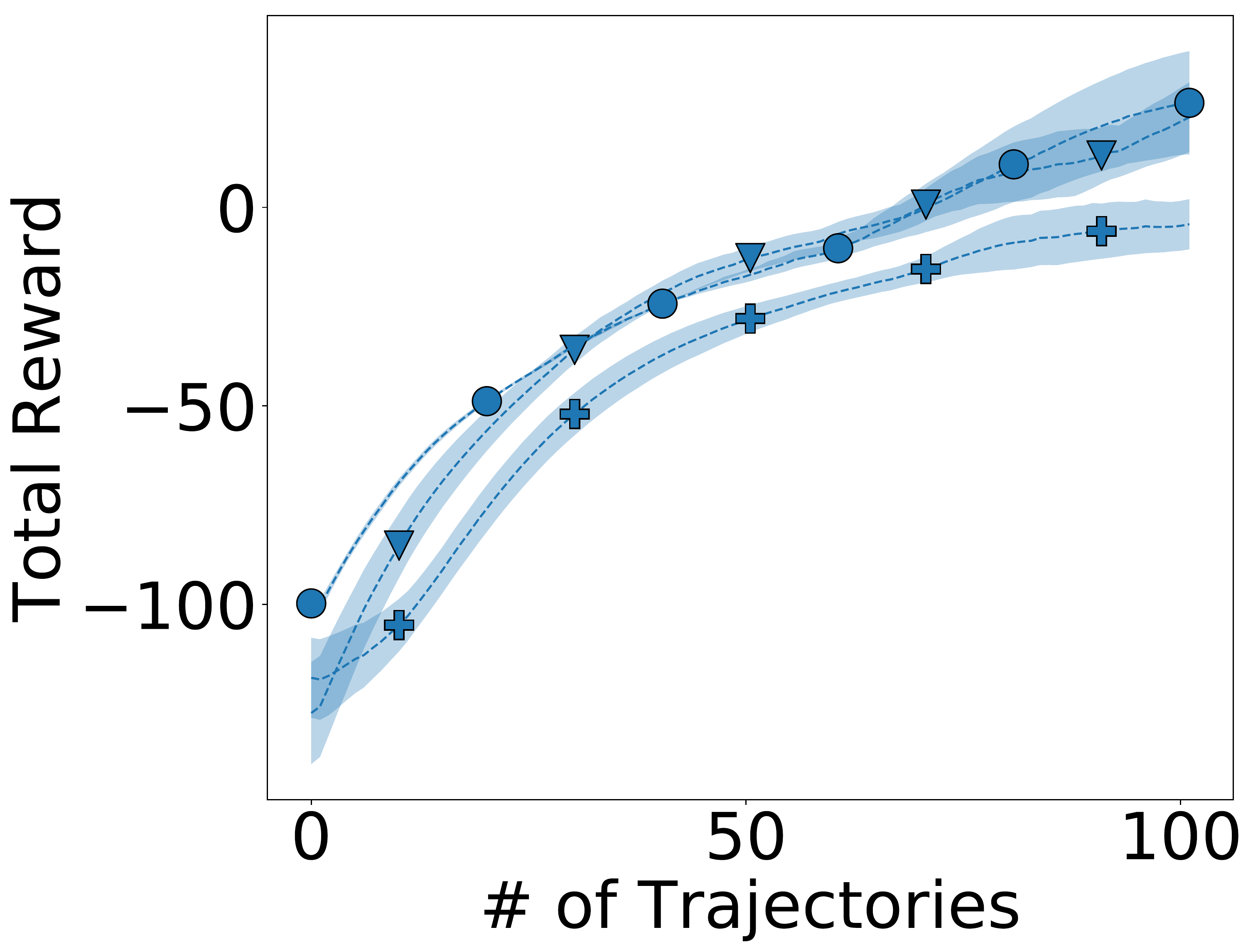}
    \includegraphics[width=\widthh pt]{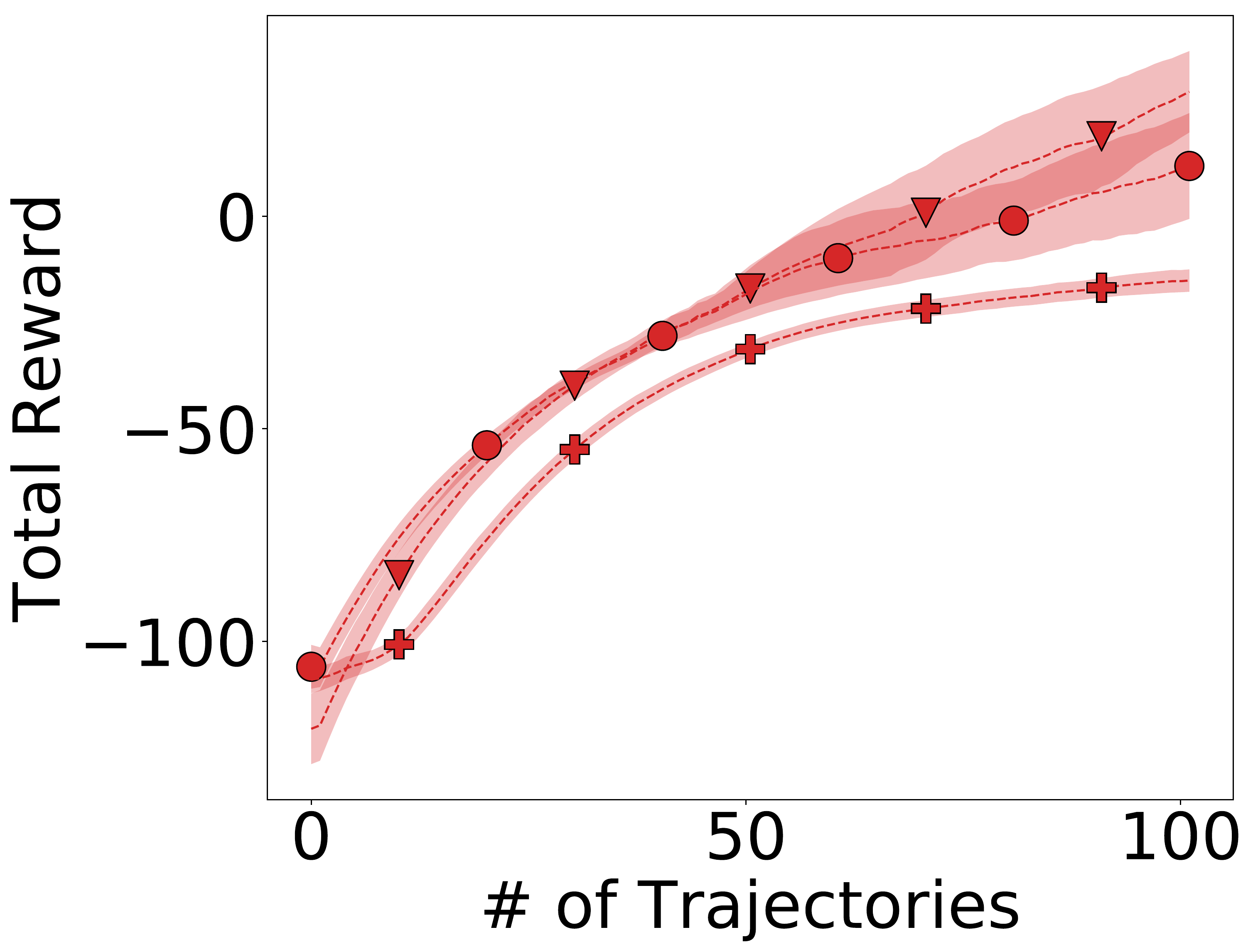}
    \includegraphics[width=\widthh pt]{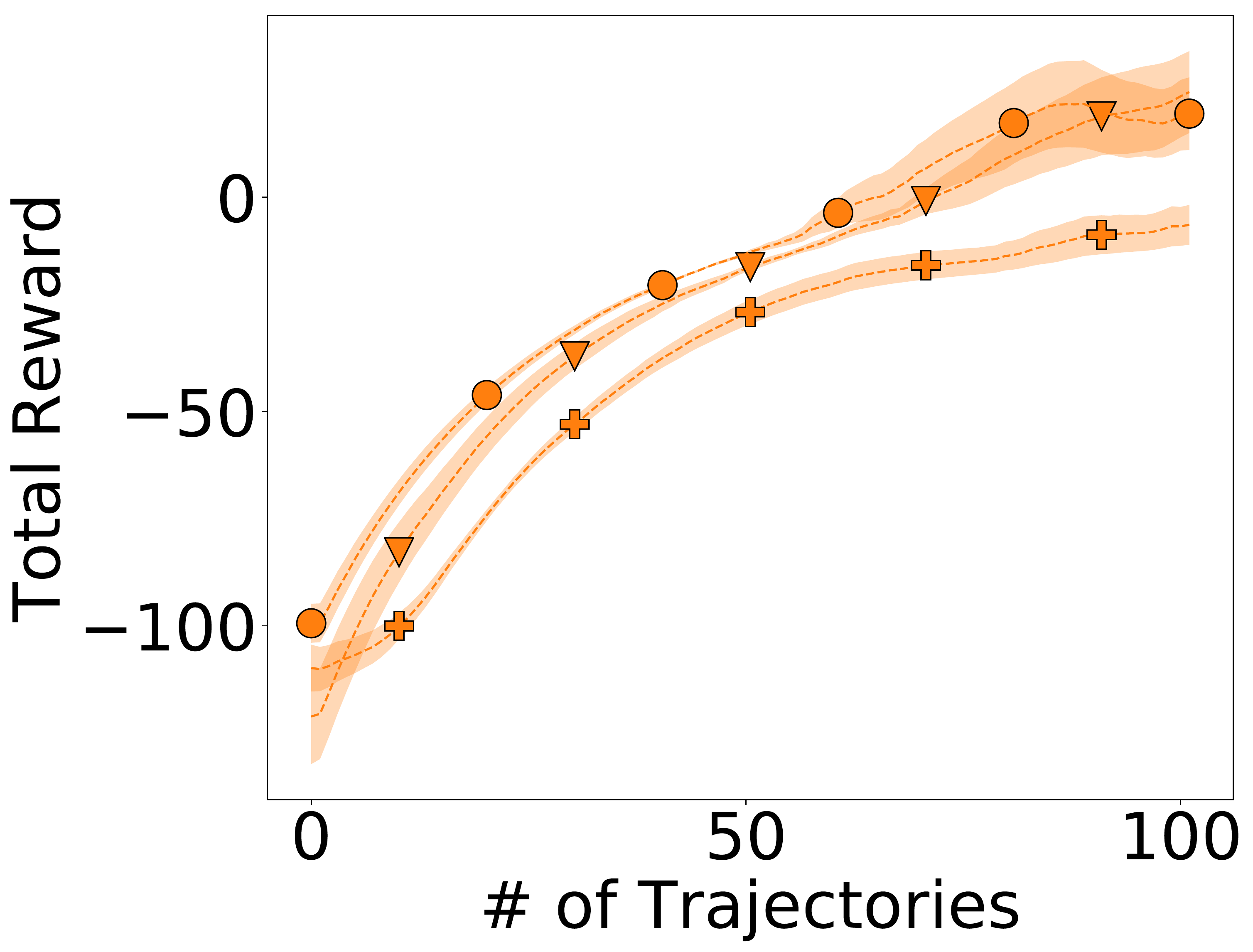} \\
    \includegraphics[width=330 pt]{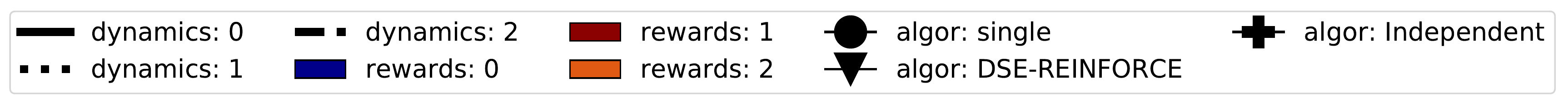}
    \caption{Comparison of DSE-SAC against other algorithms. Here the task configurations are the same as in the manuscript. }
    \label{fig:learning_sac}
\end{figure}

\subsection{Simple generalization experiments with DSE-SAC}
\begin{figure}[t]
 \includegraphics[width=200pt]{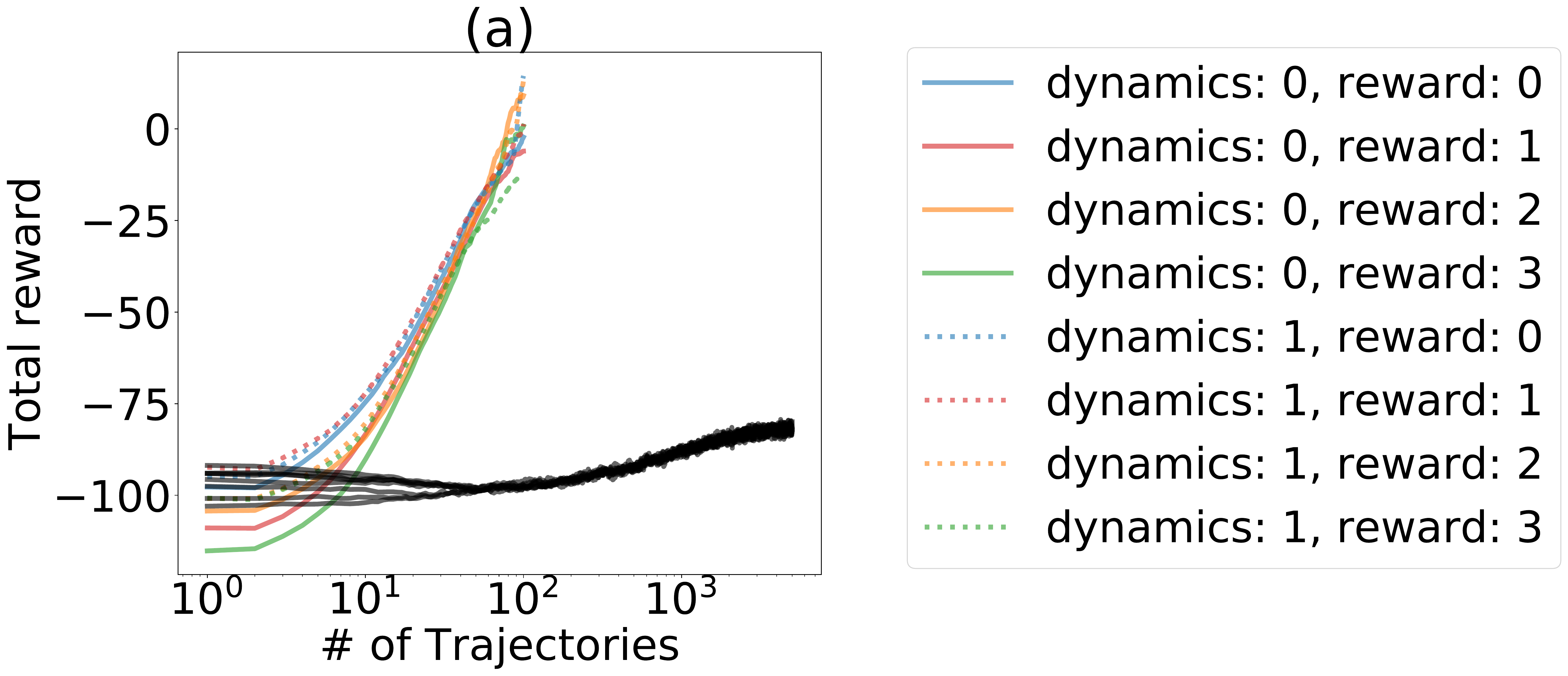}
    \centering
    \includegraphics[width=100pt]{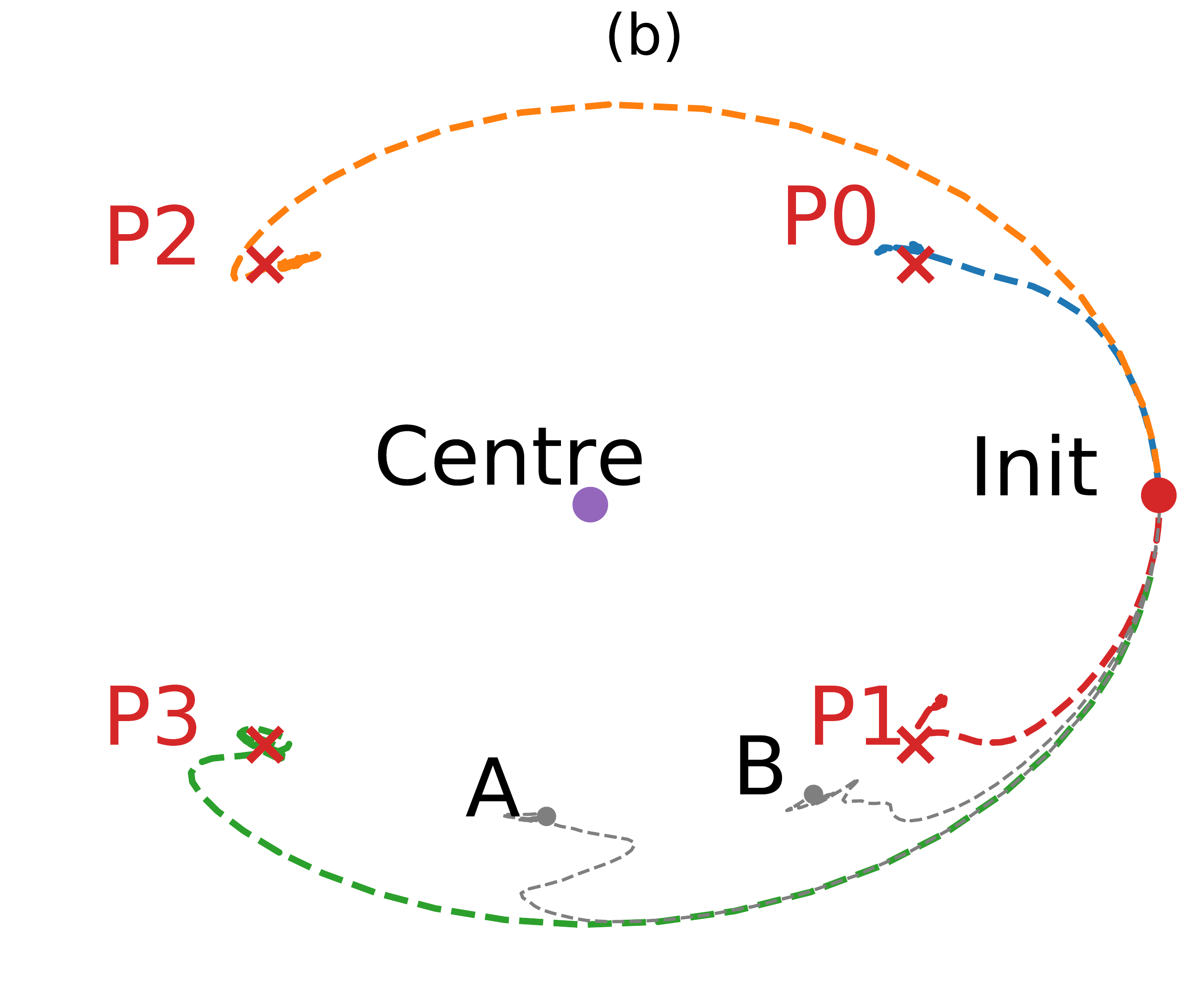}
    
    \caption{Training the Mujoco Reacher-v2 in the full configuration for a problem specification with 2 dynamics cases and 4 rewards cases. The 4 goals here were 4 corners of a square. For each dynamics case, the lengths of the 2 arm components (arm0, arm1) were set as follows: \{(arm0, arm1): $(33\%, 67\%), (67\%, 33\%)$\}. (a) shows the episodic rewards obtained by both DSE-SAC compared to solving the problem using DSE-REINFORCE. (b) shows example trajectories of a trained policy starting with the Reacher's fingertip at "Init" (in red) and reaching the different goals, while pivoting around "Centre" (in purple). P0, P1, P2, and P3 are the goal conditions the Reacher was trained on with the  trajectories in color reaching to them. In grey,  trajectories that reach goal conditions $A$ and $B$ that the Reacher has never encountered in training.}
    \label{fig:generalization_sac}
\end{figure}

In Figure \ref{fig:generalization_sac}(a), we observe how clearly DSE-SAC outperforms DSE-REINFORCE, here acting as a baseline, in this more complex problem. 

The grey trajectories in figure \ref{fig:generalization_sac}(b) reaching $A$ and $B$ are obtained by choosing intermediate values of $g$ between
the means of the variational distributions of each goal condition, whereas the colored trajectories are sampled using the variational means of each goal condition. As it can be seen, those intermediate values lead to trajectories reaching intermediate goal locations highlighting the generalization of the latent space to different goals.

\subsection{Generalization DSE-SAC}

\begin{table}[]
\caption{Generalization experiments on $6-3$ and $4-5$ configurations with DSE-SAC, compared against single embedding and the independent training}
\begin{tabular}{c|c|c|c|}
Experiment & DSE initial reward & Single initial reward & \# Trajectories Independent\\
\hline
6-3: (0,0) & $-18.79 \pm 4.50$     & $-22.26 \pm 0.75$          & $44.75 \pm 2.62$                                 \\
6-3: (1,1) & $-8.63 \pm 1.85$      & $-21.09 \pm 0.77$          & $68.5  \pm 7.80$                                  \\
6-3: (2,2) & $-24.04 \pm 2.76$     & $-28.84 \pm 2.12$          & $40.5 \pm 1.93$                                  \\
4-5: (0,0) & $-27.06 \pm 3.55$     & $-25.05 \pm 1.605$         & $37.25 \pm 2.63$                                \\
4-5: (0,1) & $-32.15 \pm 4.94$     & $-25.05 \pm 1.60$          & $37.00 \pm 2.63$                                           \\
4-5: (1,0) & $-25.78 \pm 0.75$     & $-16.68 \pm 5.31$          & $36.25 \pm 2.83$                                           \\
4-5: (1,1) & $-33.44 \pm 3.06$     & $-27.00 \pm 1.65$          & $36.25 \pm 2.69$                                           \\
4-5: (2,2) & $-20.15 \pm 12.13$    & $-18.62 \pm 12.13$         & $45.25 \pm 1.55$                                          
\end{tabular}
\end{table}
Table 3 shows results from initialising for the unseen tasks of the incomplete problem configurations. The initial reward columns show the initial reward of a trajectory from the corresponding algorithm by matching the correct variational parameters to the indices of the problem. The last columns shows the average number trajectories needed to train an independent policy to reach the reward obtained by DSE-SAC immediately. 
\subsection{Learning trajectories of the latent variables}
In Figure \ref{fig:latents} we plot the evolution of the latent variables for both the MTRL problems involving both the Cartpole and Mujoco Reacher-v2 environments. The legends here match the corresponding plots in the main text. For the cartpole, the final distributions achieved are reflected in Figures 1(b) and 1(c) in the main text.

\begin{figure}[t]
    \centering
    \includegraphics[width=\widthh pt]{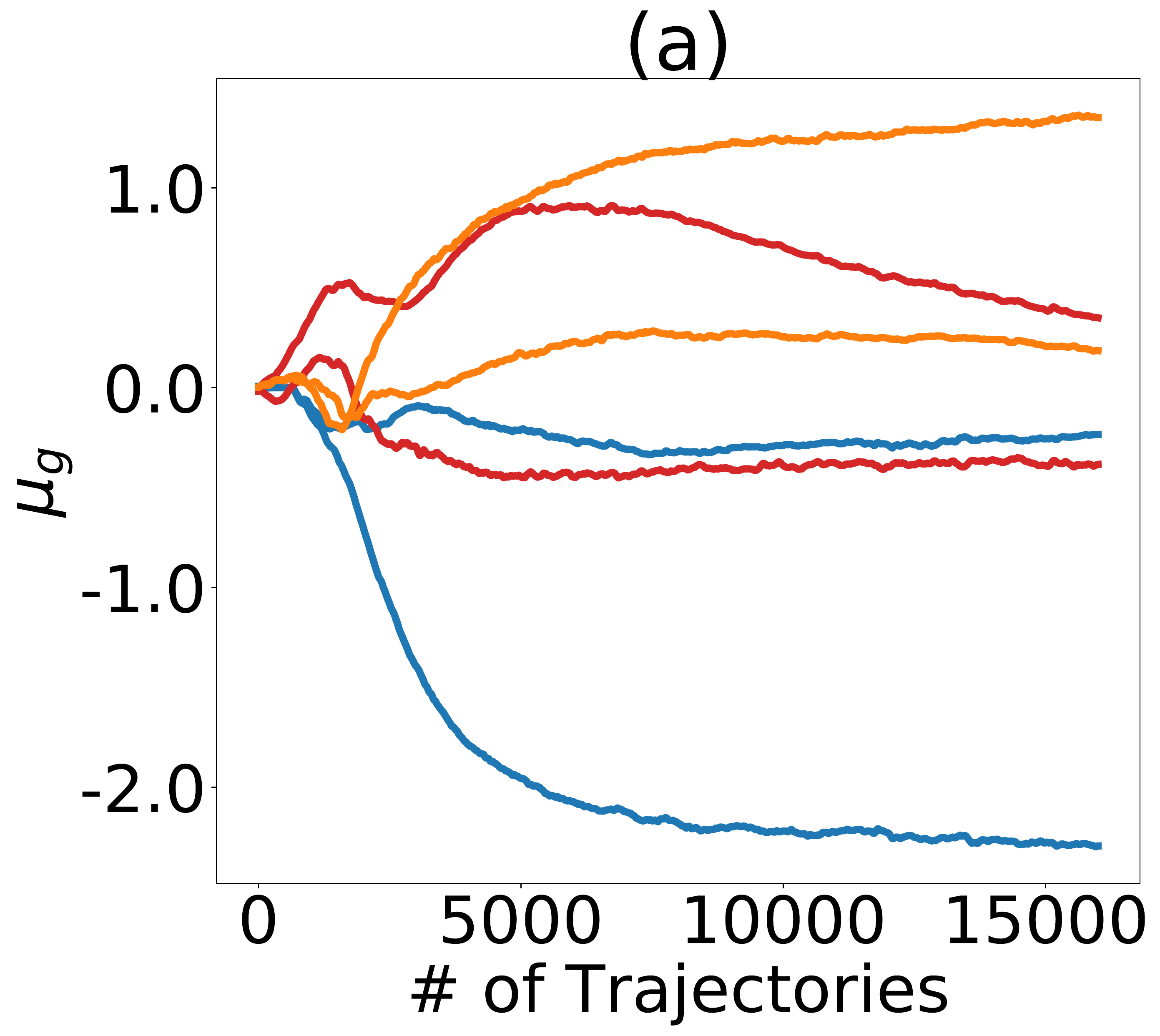}
    \includegraphics[width=\widthh pt]{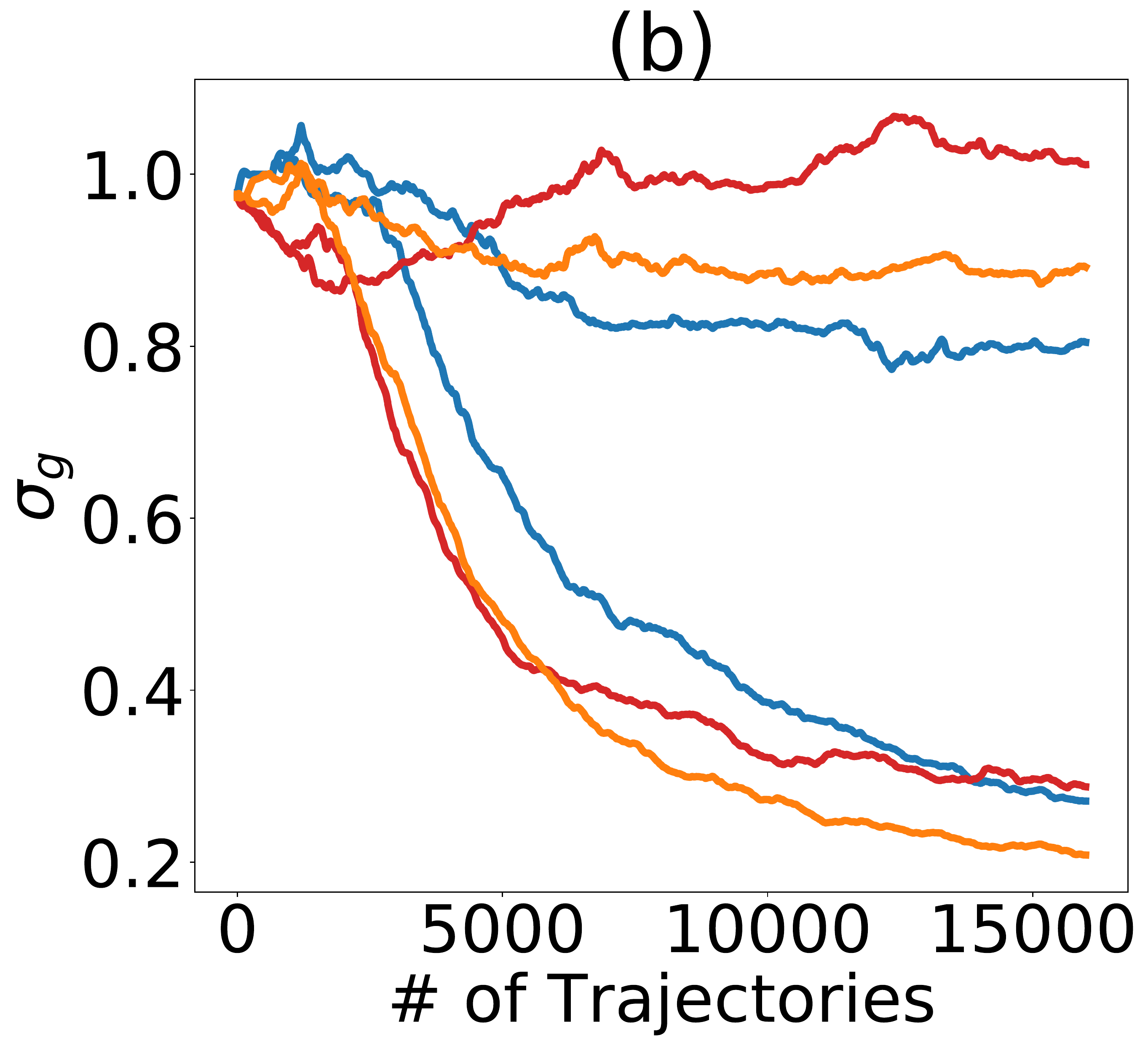} \\
    \includegraphics[width=\widthh pt]{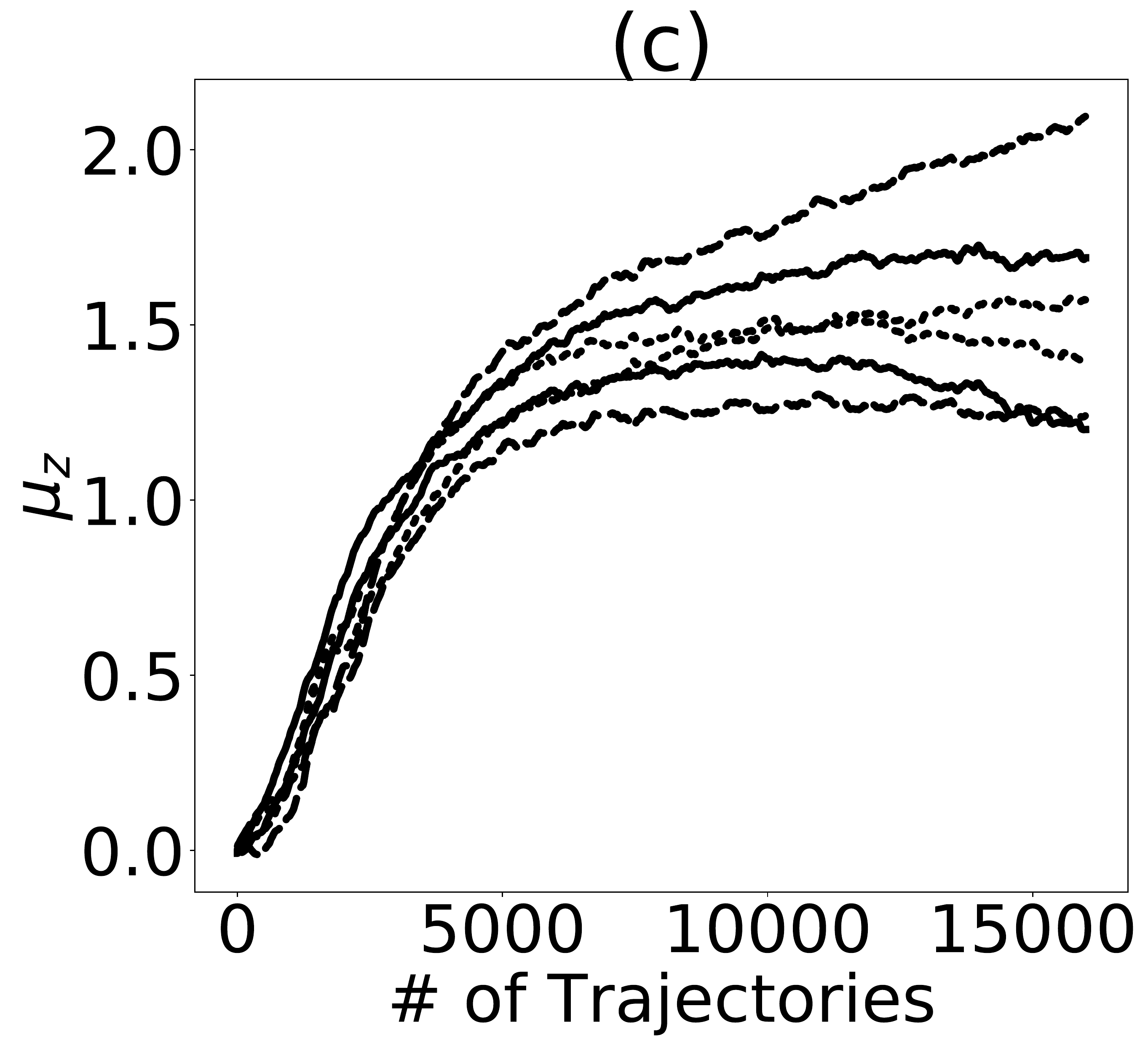}
    \includegraphics[width=\widthh pt]{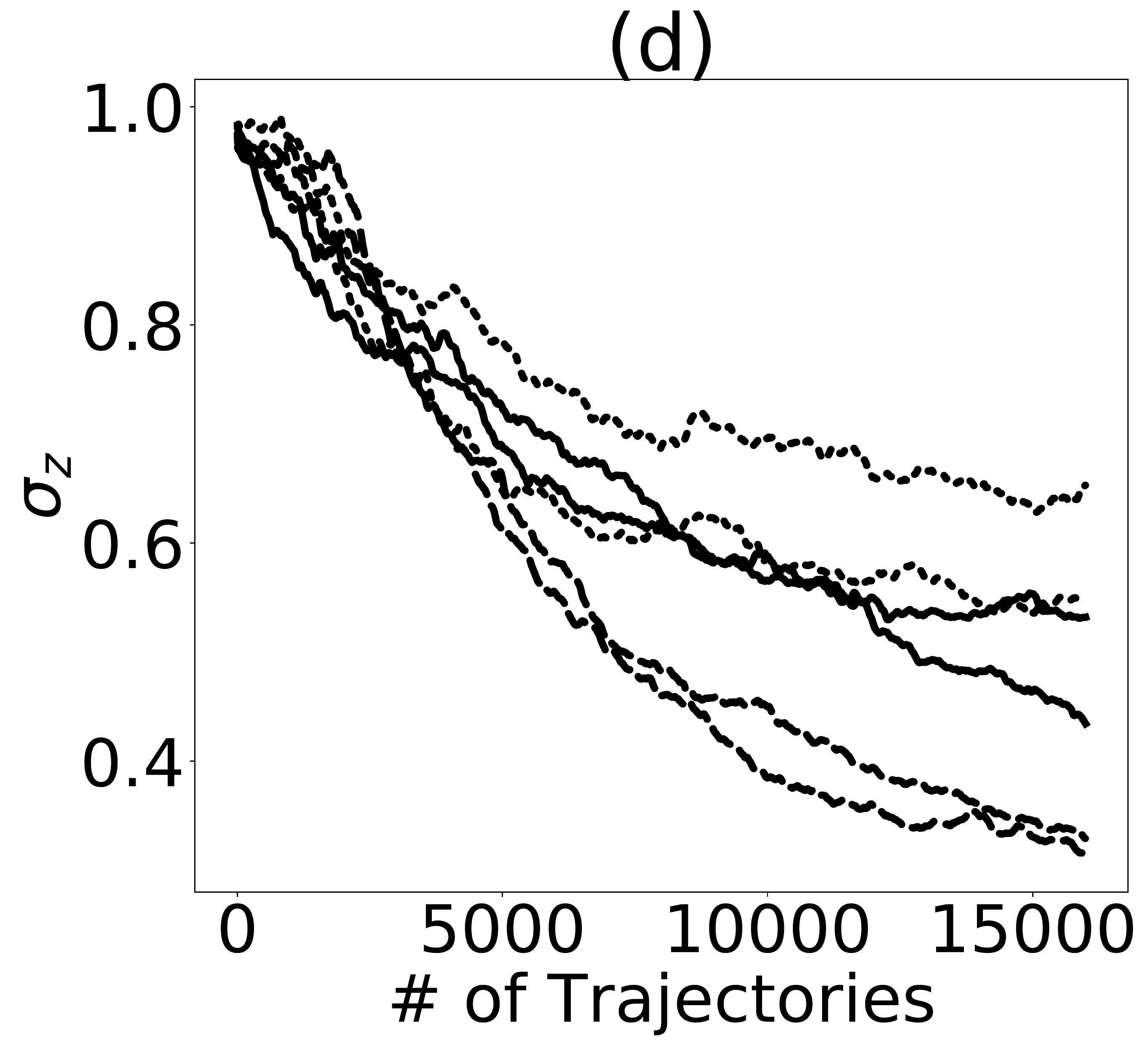} \\
    \includegraphics[width=150 pt]{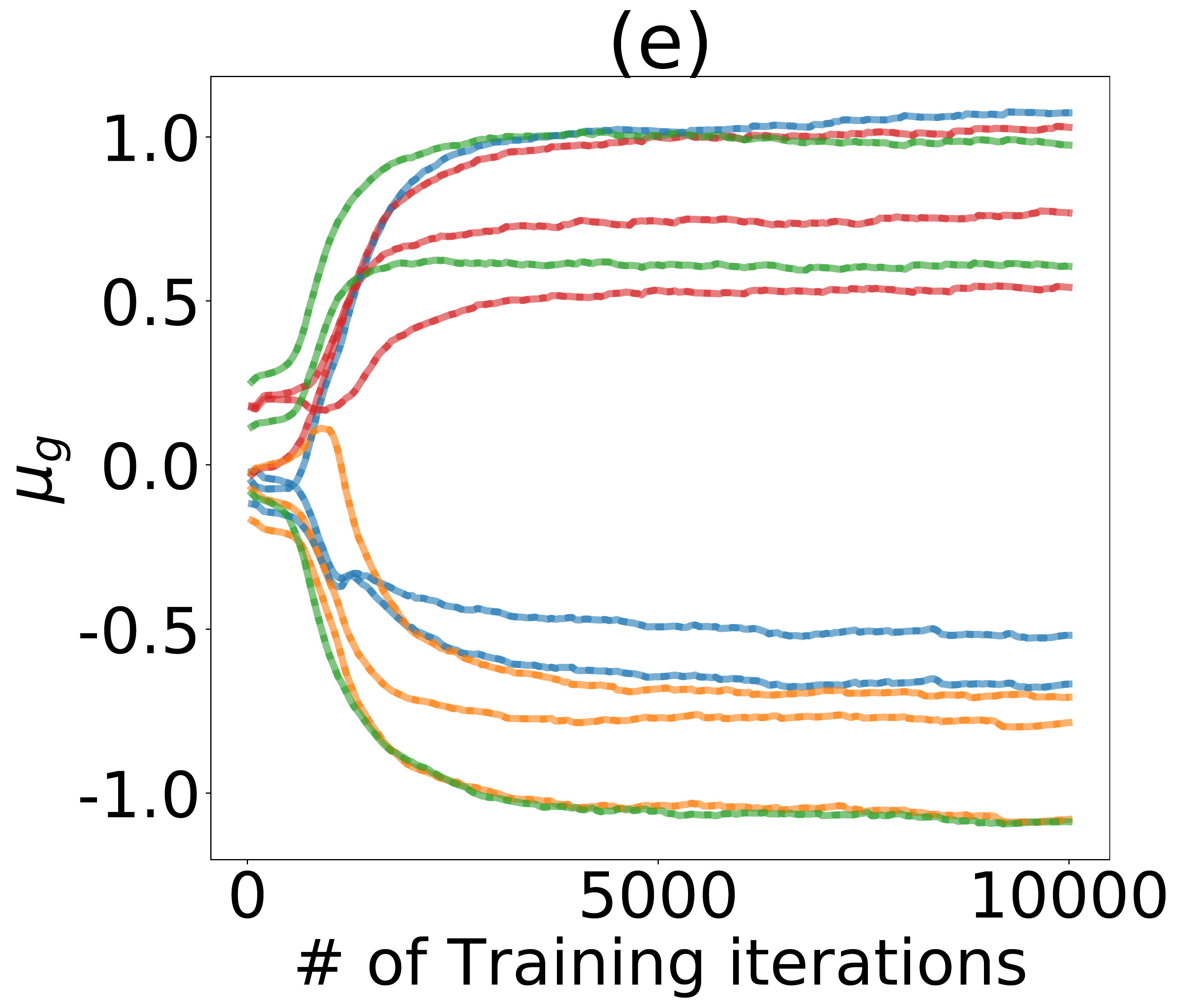}
    \includegraphics[width=150 pt]{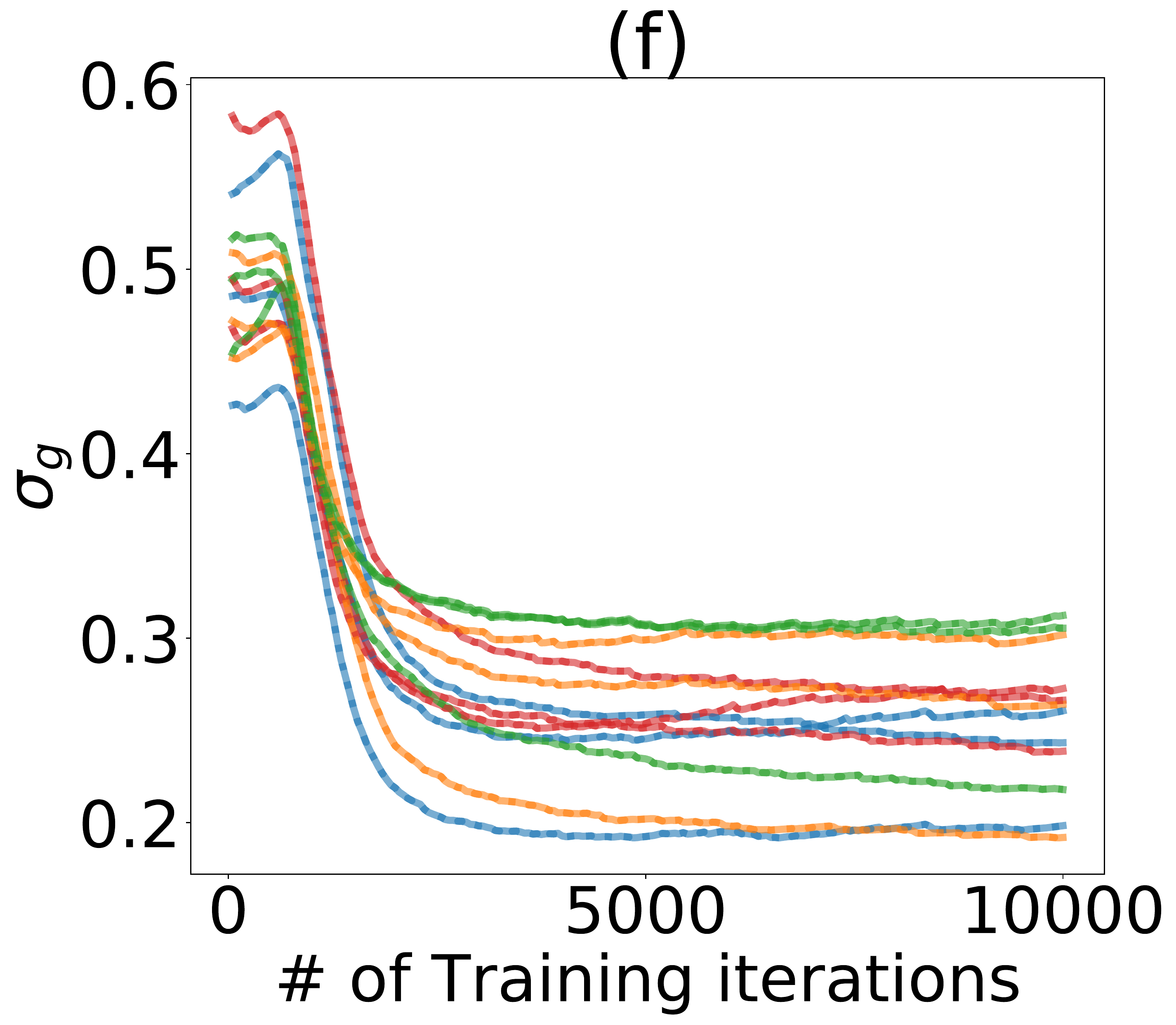} \\
    \includegraphics[width=150 pt]{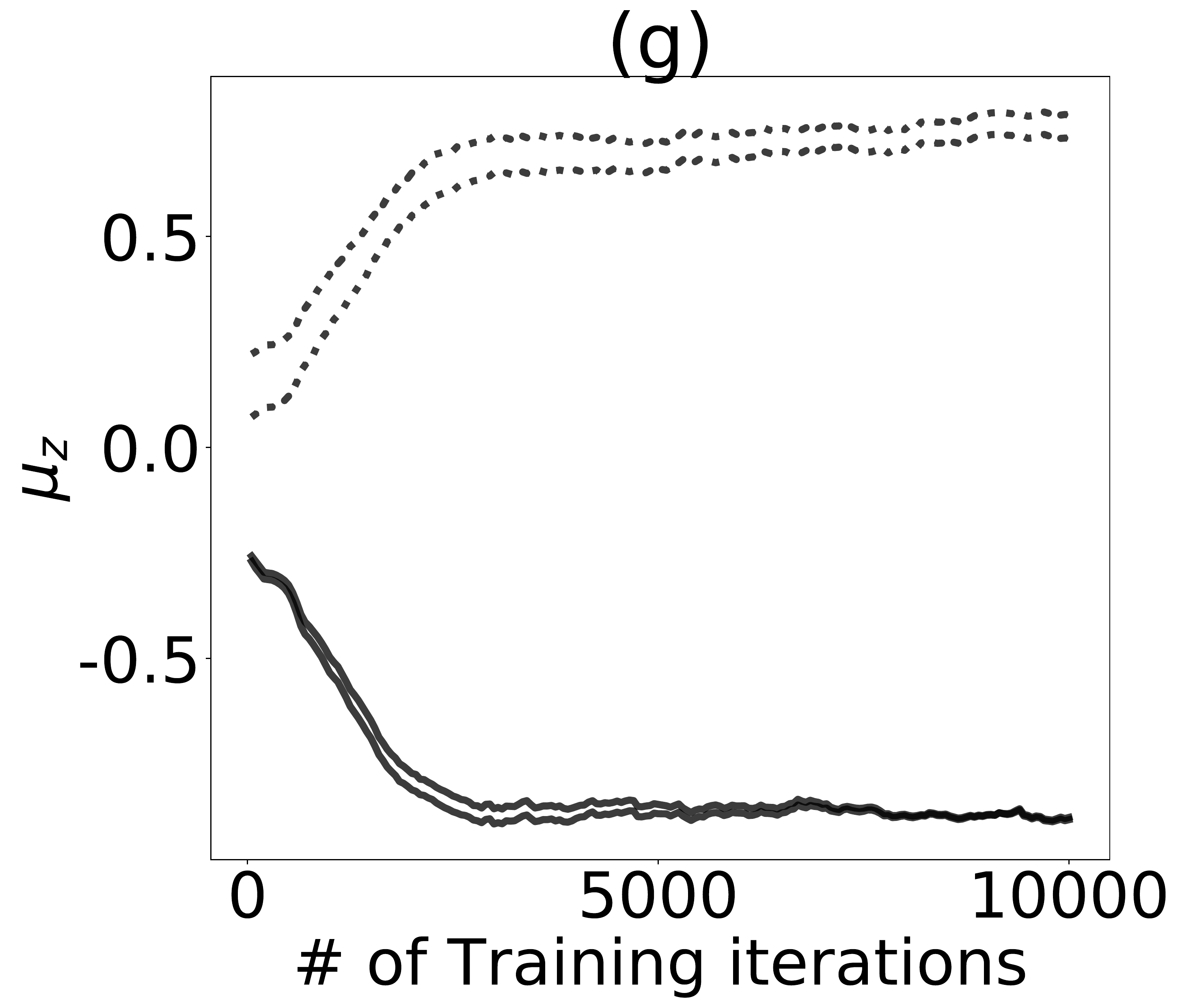}
    \includegraphics[width=150 pt]{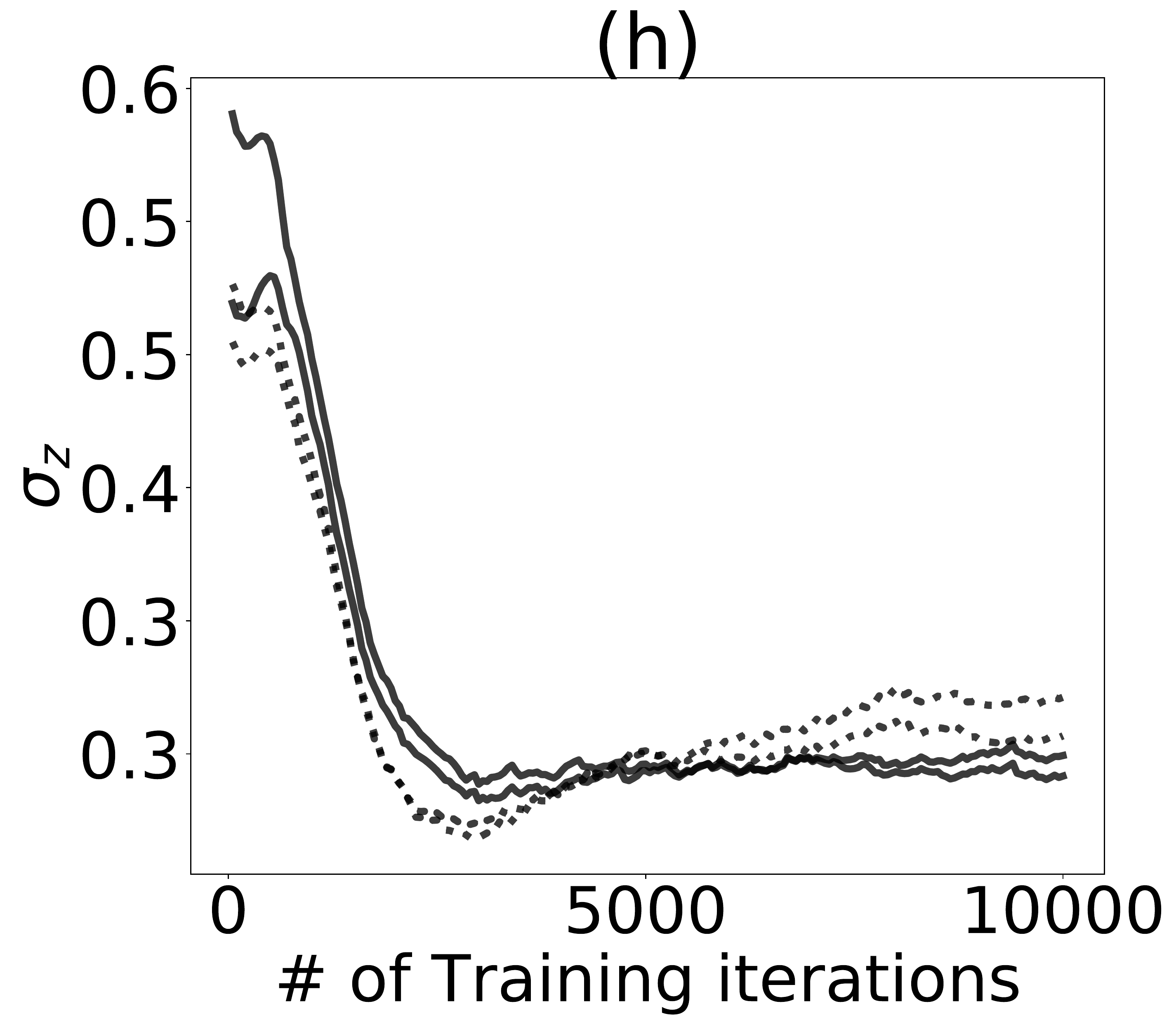}
    \caption{Learning the latent variables. (a), (b), (c) and (d) were for the MTRL Cartpole problem, while (e), (f), (g), (h) were for the MTRL Mujoco Reacher-v2 problem from Figure \ref{fig:generalization_sac}.}
    \label{fig:latents}
\end{figure}
\newpage
\end{document}